%% file: main_aistat.tex
\documentclass[twoside]{article}

\usepackage[accepted]{aistats2021}
\usepackage{alias_aistats}

\newif\ifbulletpoint
\bulletpointtrue
\bulletpointfalse

\newif\ifaistats
\aistatstrue

\newif\ifappendix
\appendixfalse

\usepackage{selectp}

\usepackage{tcolorbox}
\usepackage[round]{natbib}

\begin{document}

\doparttoc
\faketableofcontents

\twocolumn[
\aistatstitle{Instance-Wise Minimax-Optimal Algorithms for Logistic Bandits}
\aistatsauthor{Marc Abeille$^{{\star}}$ \And Louis Faury$^{{\star}}$ \And  Cl\'ement Calauz\`enes}
\aistatsaddress{ Criteo AI Lab \And  Criteo AI Lab\\ LTCI T\'el\'ecomParis \And Criteo AI Lab} ]

\begin{abstract}
    Logistic Bandits have recently attracted substantial attention, by providing an uncluttered yet challenging framework for understanding the impact of non-linearity in parametrized bandits. It was shown by \cite{faury2020improved} that the learning-theoretic difficulties of Logistic Bandits can be embodied by a \emph{large} (sometimes prohibitively) problem-dependent constant $\kappa$, characterizing the magnitude of the reward's non-linearity. In this paper we introduce a novel algorithm for which we provide a refined analysis. This allows for a better characterization of the effect of non-linearity and yields improved problem-dependent guarantees. In most favorable cases this leads to a regret upper-bound scaling as $\tilde{\mathcal{O}}(d\sqrt{T/\kappa})$, which dramatically improves over the $\tilde{\mathcal{O}}(d\sqrt{T}+\kappa)$ state-of-the-art guarantees. We prove that this rate is \emph{minimax-optimal} by deriving a $\Omega(d\sqrt{T/\kappa})$ problem-dependent lower-bound. Our analysis identifies two regimes (permanent and transitory) of the regret, which ultimately re-conciliates \citep{faury2020improved} with the Bayesian approach of \cite{dong2019performance}.
    In contrast to previous works, we find that in the permanent regime non-linearity can dramatically ease the exploration-exploitation trade-off. While it also impacts the length of the transitory phase in a problem-dependent fashion, we show that this impact is mild in most reasonable configurations. 
\end{abstract}

\input{introduction.tex}

\input{setting_rw.tex}

\input{algo.tex}

\input{regret_bound.tex}

\input{hli.tex}

\input{relaxation.tex}

\input{conclusion.tex}

\bibliography{bib.bib}
\bibliographystyle{plainnat}

\clearpage

\onecolumn 


\appendixtrue

\addcontentsline{toc}{section}{Appendix}
\appendix
\part{}

\input{appendix/appendix.tex}

\end{document}

%% file: introduction.tex
\section{INTRODUCTION}

\paragraph{Motivation.} The Logistic Bandit (\textbf{LogB}) model is a sequential decision-making framework that recently received increasing attention in the parametric bandits literature \citep{li2010contextual, dumitrascu2018pg, dong2019performance, faury2020improved}. This interest can reasonably be attributed to the practical advantages of Logistic Bandits over Linear Bandits (\textbf{LB}) \citep{danistochastic2008, abbasi2011improved} and to the distinctive learning-theoretical questions that arise in their analysis. On the practical side, \textbf{LogB} addresses environments with \emph{binary} rewards (ubiquitous in real-word applications) where it was shown to empirically improve over \textbf{LB} approaches \citep{li2012unbiased}. On the theoretical side, \textbf{LogB} offers a rigorous framework to study the effects of non-linearity on the exploration-exploitation trade-off for parametrized bandits. It therefore stands as a stepping-stone in generalizing the well-understood \textbf{LB} framework to more general and complex reward structures. This particular goal has driven a large part of the research on parametrized bandits, through the study of Generalized Linear Bandits \citep{filippi2010parametric,li2017provably} and Kernelized Bandits \citep{valko2013finite, chowdhury2017kernelized}.

\paragraph{Non-Linearity in LogB.}
The importance of the non-linearity is fundamentally \emph{problem-dependent} in the \textbf{LogB} setting. Interestingly enough, the effects of the non-linearity can be compactly summed-up in a problem-dependent constant, which we will for now denote $\kappa$. Intuitively, $\kappa$ can be understood as a \emph{badness of fit} between the true reward signal and a linear approximation. Given the highly non-linear nature of the logistic function it can become prohibitively large, even for reasonable problem instances. The first known regret upper-bounds for \textbf{LogB}  were provided by \cite{filippi2010parametric}, scaling as $\bigo{\kappa d\sqrt{T}}$. This suggests that non-linearity is highly detrimental for the exploration-exploitation trade-off as the more non-linear the reward (\emph{i.e} the bigger $\kappa$) the larger the regret. 

\paragraph{Recent Work.}
This conclusion was nuanced by \cite{faury2020improved} who introduced an algorithm achieving a regret upper-bound scaling as $\bigo{d\sqrt{T}+\kappa}$. Their bound henceforth tells a different story, namely that for large horizons the effect of non-linearity disappears. However, it is not clear if the scaling of the regret's first-order term is optimal (w.r.t $\kappa$) as to the best of our knowledge there exist no instance-dependent lower-bounds for \textbf{LogB}. Furthermore, the presence in the regret bound of a second-order term scaling with $\kappa$ suggests that the non-linearity can still be particularly harmful for small horizons. 
A slightly different message on the learning-theoretic difficulties behind the \textbf{LogB} was brought by the Bayesian analysis of \cite{dong2019performance}. They show that in favorable settings the dependency in $\kappa$ can be removed altogether from the Bayesian regret of Thompson Sampling (whatever the horizon). Yet in worst-case instances (and as $\kappa$ grows arbitrarily large) their analysis suggests that the problem can remain arbitrarily hard. 

\paragraph{Contributions.}
In this paper, we \textbf{(1)} introduce a new algorithm for the Logistic Bandit setting, called \ouralgo. Its analysis distinguishes two regimes of the regret during which the behavior of the algorithm is significantly different: a \emph{long-term} regime and a \emph{transitory} regime. We show that \textbf{(2)} in the long-term regime the situation can be much better than what was previously suggested as for a large set of problems the regret scales as $\sqrt{T/\kappa}$. In other words, non-linearity can dramatically ease the exploration-exploitation trade-off. We prove that \textbf{(3)} this scaling is optimal by exhibiting a matching \emph{problem-dependent} lower-bound. To the best of our knowledge, this is the first problem-dependent lower-bound for \textbf{LogB}. We also \textbf{(4)} link the transitory regime to the second-order term in the regret bound of \cite{faury2020improved} and to the worst-case analysis of \cite{dong2019performance}. We show that \textbf{(5)} the length of this transitory phase can be much smaller than $\kappa$ and that \ouralgo{} can \emph{adapt} to the complexity of the problem to avoid long transitory phases. While the definition of \ouralgo{} allows for a neat analysis, it can be challenging to implement. To this end, we \textbf{(6)} provide a \emph{convex relaxation} of \ouralgo, tractable for finite arm-sets (without sacrificing theoretical guarantees). 

%% file: setting_rw.tex
\section{PRELIMINARIES}

\paragraph{Notations}
Let $f$ and $g$ be two univariate real-valued functions. Throughout the article, we denote $f\lesssim_t g$ or $f=\tilde{\mcal{O}}(g)$ to indicate that $g$ dominates $f$ up to logarithmic factors. In proof sketches and discussions, we informally use $f\lessapprox g$ to denote $f\leq Cg$ where $C$ is an universal constant. The notation $\dot{f}$ (resp. $\ddot{f}$) will denote the first (resp. second) derivative of $f$. For any $x\in\mbb{R}$ we will denote $\ltwo{x}$ its $\ell_2$-norm. The notation $\mcal{B}_d(x,r)$ (resp. $\mcal{S}_d(x,r)$) will denote the $d$-dimensional $\ell_2$-ball (resp. sphere) centered at $x$ and with radius $r$. Finally, for two real-valued symmetric matrices $A$ and $B$, the notation $\mbold{A}\succeq \mbold{B}$ indicates that $\mbold{A}-\mbold{B}$ is positive semi-definite. When $\mbold{A}$ is positive semi-definite, we will note $\lVert x\rVert_\mbold{A}=\sqrt{x^\transp \mbold{A}x}$. For two scalar $a$ and $b$, we denote the maximum (resp. minimum) of $(a,b)$ as $a\vee b$ (resp. $a\wedge b)$. For an event $E\in\Omega$, we write $E^C=\Omega\!\setminus\! E$ and $\mathds{1}\{E\}$ the indicator function of $E$.
\subsection{Setting}

\ifaistats
\input{tikz/tikz_def_kappa}
\else\fi

We consider the Logistic Bandit setting, where an agent selects actions (as vectors in $\mbb{R}^d$) and receives binary, Bernoulli distributed rewards. More precisely at every round $t$ the agent observes an arm-set $\mcal{X}$ (potentially infinite) and plays an action $x_t\in\mcal{X}$. She receives a reward $r_{t+1}$ sampled according to a Bernoulli distribution with mean $\mu(x_t^\transp\theta_\star)$, where $\mu(z)\defeq (1+e^{-z})^{-1}$ is the \emph{logistic} function, and $\theta_\star\in\mbb{R}^d$ is \emph{unknown} to the agent. As a result:
\begin{align*}
    \mbb{E}\left[r_{t+1}\,\middle\vert x_t\right] = \mu\left(x_t^\transp\theta_\star\right)\; .
\end{align*}
The logistic function $\mu$ is strictly increasing. It also satisfies a (generalized) \emph{self-concordance} property thanks to the inequality $\vert\ddot{\mu}\vert\leq \dot\mu$.
We will work under the two following standard assumptions.
\begin{ass}[Bounded Arm-Set] 
    For any $x\in\mcal{X}$ the following holds:\footnote{This assumption is made for ease of exposition, and can easily be relaxed. It can be imposed by re-scaling all actions - which will impact $\ltwo{\theta_\star}$ accordingly.} $\ltwo{x}\leq 1$.
\end{ass}

\begin{ass}[Bounded Bandit Parameter]
There exists a \emph{known} constant such that $\ltwo{\theta_\star}\leq S$.
\end{ass}
We will denote $\Theta \defeq \mcal{B}_d(0,S)$. For any $\theta\in\Theta$, we will use the notation $x_\star(\theta) \defeq \argmax_{x\in\mcal{X}} x^\transp\theta$. At  each round $t$, the agent takes a decision following a policy $\pi:\mcal{F}_t\to\mcal{X}$, mapping $\mcal{F}_t\defeq \sigma(\{x_s,r_{s+1}\}_{s=1}^{t-1})$ (the filtration encoding the information acquired so far) to the arms. The goal of the agent is to minimize her cumulative pseudo-regret up to time $T$:
\begin{align*}
    \regret_{\theta_\star}^\pi(T) \defeq \sum_{t=1}^T \mu\left(x_\star(\theta_\star)^\transp\theta_\star\right)-\mu\left(x_t^\transp\theta_\star\right)\; .
\end{align*}
We will drop the dependency in $\pi$ when there is no ambiguity about which policy is considered. 

\ifaistats
\else
\input{tikz/tikz_def_kappa}
\fi

The \emph{conditioning} of $\mu$ lies at the center of the analysis of Logistic Bandits. In previous work this conditioning was evaluated through the whole decision-set $\Theta\times\mcal{X}$ through the problem-dependent quantity $\kappa \defeq \max_{\mcal{X},\Theta} 1/\dot{\mu}(x^\transp\theta)$. In a few words, $\kappa$ quantifies the level of non-linearity of plausible reward signals and in this sense can be understood as a measure of discrepancy with the linear model. As such, it can be significantly \emph{large} even for reasonable \textbf{LogB} problems. We refer the reader to Section 2 of \cite{faury2020improved} for a detailed discussion on the importance of this quantity. In this work, we refine the problem-dependant analysis through the use of the following quantities:\footnote{Again, we will drop the dependency in $\theta_\star$ when there is no ambiguity.}
\begin{align*}
\ifaistats
	\kappa_\star(\theta_\star) &\defeq 1/\dot{\mu}\left(x_\star(\theta_\star)^\transp\theta_\star\right) \;, \\\kappa_\mcal{X}(\theta_\star) &\defeq \max_{x\in\mcal{X}} 1/\dot{\mu}\left(x^\transp\theta_\star\right)\; .
\else
    \kappa_\star(\theta_\star) \defeq 1/\dot{\mu}\left(x_\star(\theta_\star)^\transp\theta_\star\right) \qquad \text{and} \qquad \kappa_\mcal{X}(\theta_\star) \defeq \max_{x\in\mcal{X}} 1/\dot{\mu}\left(x^\transp\theta_\star\right)\; .
\fi
\end{align*}
In other words $\kappa_\star$ and $\kappa_\mcal{X}$ measure the \emph{effective} non-linearity around the best action $x_\star(\theta_\star)$ and in the whole parameter-set. Their definitions are illustrated in \cref{fig:defkappa}. We have the following ordering: $\kappa_\star \leq \kappa_\mcal{X} \leq \kappa$,
with equality between $\kappa_\star$ and $\kappa_\mcal{X}$ for \emph{symmetric} arm-sets (\emph{e.g} $\mcal{X}=\mcal{B}_d(0,1)$). Note that the scalings of $\kappa_\mcal{X}$ and $\kappa$ are fundamentally the same; both grow as $\exp(\ltwo{\theta_\star})$ and can therefore be \emph{very} large, even in reasonable settings.

\subsection{Related Work}

\paragraph{Generalized Linear Bandits.} Non-linear parametric bandits were first studied by \cite{filippi2010parametric}, who introduced an optimistic algorithm for Generalized Linear Bandits. Their approach was generalized to randomized algorithms \citep{russo2013eluder, russo2014learning, abeille2017linear} and further refined for the finite-armed setting by \cite{li2017provably}. Some efforts have also been made to adapt the previous approaches to be fully-online and efficient \citep{zhang2016online, jun2017scalable}. All the aforementioned contributions provide regret bounds scaling proportionally to $\kappa$, which was recently proven to be sub-optimal for the logistic bandit. 

\paragraph{Logistic Bandits.} \cite{faury2020improved} introduced an algorithm which regret bound scales as $\bigo{d\sqrt{T} + \kappa d^2}$. This nuances the folk intuition that non-linearity can be only detrimental to the exploration-exploitation trade-off. Indeed, when $T$ is sufficiently large ($T\gtrapprox\kappa^2$) the regret bound is seemingly \emph{independent} of $\kappa$ and one recovers the regret bound of the \textbf{LB} (\emph{e.g} $\bigo{d\sqrt{T}}$). In other words, the non-linearity no longer plays a part in the exploration-exploitation trade-off. 
The presence of a second order term (scaling with $\kappa d^2$) in the regret bound also suggests that under short horizons ($T\lessapprox \kappa^2$) the problem remains \emph{hard} - as the regret bound scales linearly with $T$. 
Finally, note that the algorithm of \cite{faury2020improved} is impractical: it involves non-convex optimization steps, as well as maintaining a set of constraints (the admissible log-odds) which size grows linearly with time. 

\paragraph{A Bayesian Perspective.} 
The nature of the second order term of \cite{faury2020improved} and whether it could be improved is still an open question. It is however coherent, to some extent, with the Bayesian analysis of \cite{dong2019performance}: by letting $\kappa$ be arbitrarily large (compared to $T$) they construct arm-sets where no policy can enjoy sub-linear regret. Their construction is particularly worst-case, yet emphasizes that some \textbf{LogB} instances are notably hard. On the other hand they also provide a positive result; they exhibit scenarios where the Bayesian regret is upper-bounded by $\sqrt{T}$, \emph{independently} of $\kappa$. This stresses that second order dependencies in $\kappa$ are fundamentally related to the arm-set structure and suggests there is room for improvement. 

\subsection{Outline and Contributions}
In \cref{sec:alg} we formally introduce \ouralgo, an algorithm for \textbf{LogB} based on the \textbf{O}ptimism in \textbf{F}ace of \textbf{U}ncertainty (\textbf{OFU}) principle. 

We collect our main results in \cref{sec:results}: 
\vspace{-5pt}
\begin{itemize}[leftmargin=0cm,itemindent=.5cm,labelwidth=\itemindent,labelsep=0cm,align=left, itemsep=-5pt]
	\item[-] \hspace{-4mm} \cref{thm:generalregret} provides a regret upper-bound for \ouralgo. It decomposes in two terms $\regretplus_{\theta_\star}$ and $\regretmoins_{\theta_\star}$,  each associated with a different regime of the regret: \emph{permanent} and \emph{transitory}. $\regretmoins_{\theta_\star}$ refines the second-order term of \cite{faury2020improved} by introducing the notion of \emph{\bad} arms, essentially played in a transitory phase. $\regretplus_{\theta_\star}$ dominates when $T$ is large and scales as $\bigo{d\sqrt{T/\kappa_\star}}$. 
	\item[-]  \hspace{-4mm} \cref{thm:lowerboundlocal}  provides a matching problem-dependent lower-bound proving that \ouralgo{} is minimax-optimal. The main implication is that non-linearity in \textbf{LogB} can ease the exploration-exploitation trade-off in the long-term regime, postponing the challenge of non-linearity to the transitory phase.

\item[-]\hspace{-4mm} \cref{prop:lengthtransitory} shows that the transitory phase is \emph{short} for reasonable arm-set structures. This confirms that \ouralgo{}'s second order term ($\regretmoins_{\theta_\star}$) can be bounded independently of $\kappa$. In most unfavorable cases, we retrieve the second order term in \cite{faury2020improved}. 

\item[-]\hspace{-4mm} \cref{thm:regretball} synthesizes the aforementioned improvements. For the commonly studied $\mcal{X}=\mcal{B}_d(0,1)$ we prove that \ouralgo{} enjoys a $\bigo{d\sqrt{T/\kappa_\mcal{X}}}$ regret. 

\end{itemize}

We provide some intuition behind the proofs of \cref{thm:generalregret} and \cref{thm:lowerboundlocal} in \cref{sec:sketch}. 

We address tractability issues in \cref{sec:tractable}. In line with previous works \ouralgo{} requires solving non-convex optimization programs. We circumvent this issue in \ouralgorelaxed{} through a \emph{convex} relaxation, at the cost of marginally degrading the regret guarantees.

%% file: tikz/tikz_def_kappa.tex
\pgfdeclarepatternformonly{north east lines wide}%
        {\pgfqpoint{-1pt}{-1pt}}%
        {\pgfqpoint{10pt}{10pt}}%
        {\pgfqpoint{8pt}{8pt}}%
        {
            \pgfsetlinewidth{0.4pt}
            \pgfpathmoveto{\pgfqpoint{0pt}{0pt}}
            \pgfpathlineto{\pgfqpoint{8.0pt}{8.0pt}}
            \pgfusepath{stroke}
        }
        
\begin{figure*}
\subcaptionbox{Assymetric arm-set.\label{fig:defkappa1}}{
 \begin{tikzpicture}[scale=0.8, transform shape]
     \begin{axis}[
        axis x line=center,
        axis y line=center,
        xtick=\empty,
        ytick=\empty,
        scaled ticks=false,
        clip=false,
        xmin=-4,
        xmax=1.5,
        ymin=0,
        ymax=1.1,
        xlabel=$x^\transp\theta$,
        x label style={at={(axis description cs:1.3,0)},anchor=north},
        y label style={at={(axis cs:0.2,1.1)},anchor=north},
        ylabel=$\mu$,
        title={\Large $\boxed{4 = \pmb{\color{red}\kappa_\star} \ll \pmb{\exp(\ltwo{\theta_\star})} \leq \pmb{\color{blue}\kappa_\mathcal{X}}}$},
        title style={at={(0.75,-0.2)},anchor=north,yshift=-0.1}
   ]
     \node[circle,fill,scale=0.3, label=271:{$\color{black}\pmb{x_\star(\theta_\star)^\transp\theta_\star}$}] at (axis cs:0,0.0) {};
     \node[circle,fill,scale=0.3, label=270:{$\color{black}\pmb{\min_{x\in{\mathcal{X}}}x^\transp\theta_\star}$}] at (axis cs:-3,0.0) {};
     
     \draw[dotted, line width= 0.2mm] (axis cs:-3,0.06) -- (axis cs:-3,0.0);
     
     \draw[dashed, line width= 0.7mm, color=blue] (axis cs:-3.7,0.05) -- (axis cs:-1.8,0.07);
     \draw[dashed, line width= 0.7mm, color=red] (axis cs:-0.9,0.2) -- (axis cs:0.3,0.65);
     
     \draw[ultra thick, ->, blue] plot [smooth, tension=1]  coordinates {(axis cs: -3.2,0.25)  (axis cs: -3.4,0.15) (axis cs: -3.2,0.1)  };
      \draw[ultra thick, ->, red] plot [smooth, tension=1]  coordinates {(axis cs: -0.1,0.45)  (axis cs: 0.2,0.4) (axis cs: 0.4,0.45)  };
      
      \node[] at (axis cs: -3.2,0.3) {\color{blue}\bfseries slope $\pmb{1/\kappa_{\mathcal{X}}}$};
     \node[] at (axis cs: 0.75,0.5) {\color{red}\bfseries slope $\pmb{1/\kappa}_{\star}$};
     \path[name path=axis] (axis cs:-3,0) -- (axis cs:0,0);
    \addplot[domain=-3:0, black, line width=2pt, smooth, name path=f] {0.05+1/(e^(-1.5*x)+1)};
    \addplot[gray!25, fill opacity=0.5] fill between[of=f and axis];

    \end{axis}
    
      \draw[gray, ultra thick, pattern=north east lines wide] (0.4,4.95)-- (0.4,6.2) arc(90:270:1.3) -- (0.4,4.95);
      \node[] at (0.1, 6.5) {\color{gray}$\pmb{\mathcal{X}}$};
      \draw[ultra thick, black, ->] (0.4,4.95) -- (1.4,4.95);
      \node[] at (1.7, 4.95) {$\pmb{\theta_\star}$};
    \end{tikzpicture}
    }\hfill
  \subcaptionbox{Symmetric arm-set (unit-ball).\label{fig:defkappa2}}{
  
   \begin{tikzpicture}[scale=0.8,transform shape]
     \begin{axis}[
        axis x line=center,
        axis y line=center,
        xtick=\empty,
        ytick=\empty,
        scaled ticks=false,
        clip=false,
        xmin=-4.5,
        xmax=4.5,
        ymin=0,
        ymax=1.1,
        xlabel=$x^\transp\theta$,
        ylabel=$\mu$,
        xlabel=$x^\transp\theta$,
        x label style={at={(axis description cs:1.3,0)},anchor=north},
       y label style={at={(axis cs:0.3,1.1)},anchor=north},
        title={\Large $\boxed{\pmb{\exp(\ltwo{\theta_\star})}\leq \pmb{\color{red}\kappa_\star} = \pmb{\color{blue}\kappa_\mathcal{X}}}$},
        title style={at={(0.5,-0.2)},anchor=north,yshift=-0.1}
   ]
     \node[circle,fill,scale=0.3, label=270:{$\color{black}\pmb{x_\star(\theta_\star)^\transp\theta_\star}$}] at (axis cs:3,0.0) {};
     \node[circle,fill,scale=0.3, label=270:{$\color{black}\pmb{-x_\star(\theta_\star)^\transp\theta_\star}$}] at (axis cs:-3,0.0) {};
     
     \draw[dotted, line width= 0.2mm] (axis cs:-3,0.01) -- (axis cs:-3,0.0);
     \draw[dotted, line width= 0.2mm] (axis cs:3,1) -- (axis cs:3,0.0);
     
     \draw[dashed, line width= 0.7mm, color=blue] (axis cs:-4,0.01) -- (axis cs:-1.5,0.03);
     \draw[dashed, line width= 0.7mm, color=red] (axis cs:1.7,1) -- (axis cs:4.2,1.02);
     
     \path[name path=axis] (axis cs:-3,0) -- (axis cs:3,0);
    \addplot[domain=-3:3, black, line width=2pt, smooth, name path=f] {0.01+1/(e^(-1.5*x)+1)};
    \addplot[gray!25, fill opacity=0.5] fill between[of=f and axis];
    
    \draw[ultra thick, gray, pattern=north east lines wide] (axis cs:-5,1) circle (1.2cm);
    \node[] at (axis cs:-5,1.3) {$\color{gray}\pmb{\mathcal{X}}$};
    \draw[->, ultra thick, black] (axis cs:-5,1) -- (axis cs:-3,1);
    \node[] at (axis cs: -2.5,1) {$\pmb{\theta_\star}$};
    \end{axis}
    
    \end{tikzpicture}
}
\caption{Graphical illustration of $\kappa_\star$ and $\kappa_\mathcal{X}$ for different decision-sets (top-left). (a) The decision-set spans the left-hand side of the logistic function, $\kappa_\mcal{X}$ and $\kappa_\star$ have (very) different magnitude. (b) The decision-set spans (symmetrically) the whole spectrum of the logistic function, $\kappa_\mcal{X}$ and $\kappa_\star$ have similar magnitudes.}
\label{fig:defkappa}
\end{figure*}
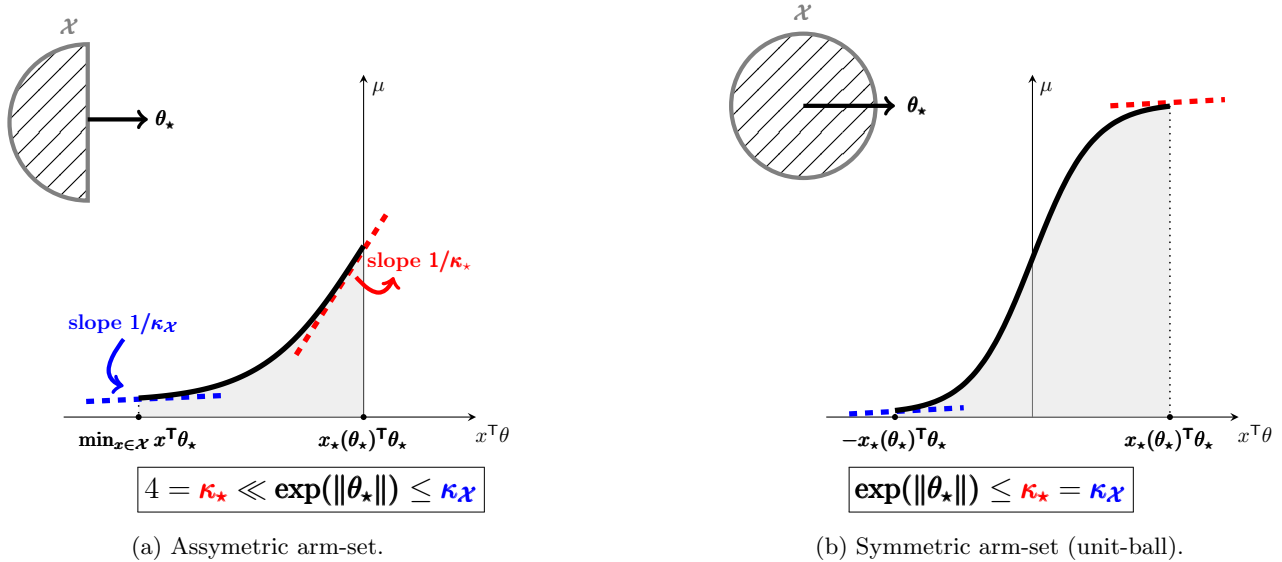

%% file: algo.tex
\section{ALGORITHM}
\label{sec:alg}
\subsection{Confidence Set} 

At the heart of the design of optimistic algorithm is the use of a tight confidence set for $\theta_\star$. We build on \cite{faury2020improved} and recall the main ingredients behind its construction. For a \emph{predictable} time-dependent regularizer $\lambda_t\!>\!0$ we define the log-loss as:
\begin{align*}
	\mcal{L}_t(\theta) \defeq -\sum_{s=1}^{t-1} \ell\left(\mu(x_s^\transp\theta), r_{s+1}\right) + \lambda_t\ltwo{\theta}^2\; .
\end{align*}
where $\ell(x,y)=y\log(x) + (1-y)\log(1-x)$. 
The log-loss is a strongly convex coercive function and its minimum $\hat\theta_t$ is unique and well-defined.  We will denote $\mbold{H_t(\theta)}\defeq\nabla^2 \mcal{L}_t(\theta)\succ 0$ the Hessian of $\mcal{L}_t$ and:
\begin{align*}
	g_t(\theta) \defeq \sum_{s=1}^{t-1}\mu(x_s^\transp\theta)x_s + \lambda_t\theta\; .
\end{align*} 
Finally, for $\delta\in(0,1]$ we define:
\begin{align*}
	\mcal{C}_t(\delta) \defeq \left\{\theta\in\Theta\, \middle\vert \,\left\lVert g_t(\theta) - g_t(\hat\theta_t) \right\rVert_{\mbold{H_t^{-1}(\theta)}}  \leq \gamma_t(\delta) \right\}\; ,
\end{align*} 
where $ \gamma_t(\delta) \defeq \sqrt{\lambda_t}(S+\frac{1}{2}) + \frac{d}{\sqrt{\lambda_t}}\log\left(\frac{4}{\delta}\left(1+\frac{t}{16d\lambda_t}\right)\right)$. The following proposition ensures that $\mcal{C}_t(\delta)$ is a confidence set for $\theta_\star$.

\begin{restatable}[Lemma 1 in \citep{faury2020improved}]{prop}{propconfidenceset} 
\label{prop:confset}
\ifappendix
Let $\delta\in(0,1]$ and define:
\begin{align*}
	E_\delta \defeq \left\{\forall t\geq 1, \,\left\lVert g_t(\theta_\star) - g_t(\hat\theta_t)\right\rVert_{\mbold{H_t^{-1}}(\theta_\star)} \leq \gamma_t(\delta) \right\}\;.
\end{align*}
Then $
\mbb{P}\Big(\forall t\geq 1, \theta_\star\in\mcal{C}_t(\delta)\Big) = \mbb{P}\left(E_\delta\right) \geq 1-\delta $.
\else
\ifaistats
\begin{align*}
	\mbb{P}\Big(\forall t\geq 1,\, \theta_\star\in\mcal{C}_t(\delta)\Big) \geq 1-\delta\; .
\end{align*}
\else
The following holds:
\begin{align*}
	\mbb{P}\Big(\forall t\geq 1,\, \theta_\star\in\mcal{C}_t(\delta)\Big) \geq 1-\delta\, .
\end{align*}
\fi
\fi
\end{restatable}

 The proof is provided in \cref{app:confidenceset} and relies on the tail-inequality of \cite[Theorem 1]{faury2020improved}, adapted to allow time-varying regularizations.\footnote{Time-varying regularization allows to run \ouralgo{} without \emph{a-priori} knowledge of the horizon $T$.}

\subsection{Algorithm}
\ouralgo{} is the counterpart of the \textbf{LB} algorithm \textbf{OFUL} of \cite{abbasi2011improved}. At each round it computes $\hat\theta_t$ and the set $\mcal{C}_t(\delta)$. It then finds an optimistic parameter $\theta_t\in\mcal{C}_t(\delta)$ and plays $x_t$ the greedy action w.r.t $\theta_t$. Formally:
\begin{align}
	(x_t,\theta_t) \in  \argmax_{x\in\mcal{X},\,\theta\in\mcal{C}_t(\delta)} \mu\left(x^\transp\theta\right)\, .
	\label{eq:ouralgo}
\end{align}
The pseudo-code for \ouralgo{} is summarized in \cref{algo:ouralgo}. Notice that we construct $\mcal{C}_t(\delta)$ with $\lambda_t = d\log(t)$, yielding $\gamma_t(\delta)\lessapprox \sqrt{d\log(t)}$.

\input{pseudoalgoofulog.tex}

\paragraph{Parameter-based versus Bonus-based.} \ouralgo{} and the \textnormal{\textbf{LogUCB2}} algorithm of \cite{faury2020improved} both rely on optimism w.r.t the same confidence set. The main difference resides in how they enforce optimism: optimistic parameter search (\ouralgo) versus exploration bonuses (\textnormal{\textbf{LogUCB2}}). In contrast with \textbf{LB}, the two approaches are not equivalent in a non-linear setting. The parameter-based approach has several key advantages. It (1) allows for a much neater analysis and (2) removes some unnecessary algorithmic complexity. A compelling illustration is that \ouralgo{} does not require the demanding projection on the set of admissible log-odds of \textnormal{\textbf{LogUCB2}}. Finally, it (3) yields algorithms that better adapt to the effective complexity of the problem (see \cref{sec:regretball}).

%% file: pseudoalgoofulog.tex
\begin{algorithm}[tb]
   \caption{\ouralgo}
   \label{algo:ouralgo}
\begin{algorithmic}
   \FOR{$t\geq 1$}
   \STATE Set $\lambda_t \leftarrow d\log(t)$.
   \STATE (\emph{Learning}) Solve $\hat\theta_t = \argmin_{\theta} \mcal{L}_t(\theta)$.
   \STATE (\emph{Planning}) Solve $(x_t,\theta_t) \!\in\!\argmax_{\mcal{X},\mcal{C}_t(\delta)} \mu\left(x^\transp\theta\right)$.
   \ifaistats
   \vspace{-12pt}
   \else\fi
    \STATE Play $x_t$ and observe reward $r_{t+1}$. 
   \ENDFOR
\end{algorithmic}
\end{algorithm}

%% file: regret_bound.tex
\ifaistats
\input{tikz/tikz_def_xmoins}
\else
\vspace{-1cm}
\fi

\section{MAIN RESULTS}
\label{sec:results}

\ifaistats
\paragraph{General Regret Upper-Bound.}
\else
\subsection{General Regret Upper-Bound}
\fi

We first define the set of \emph{\bad} arms $\mcal{X}_-$.

\begin{def*}[\Bad{} arms]
\begin{equation*}
\ifaistats
    \mcal{X}_- \defeq \left\vert 
    \begin{aligned} &\left\{x\in\mcal{X}\, \middle\vert \, x^\transp\theta_\star \leq -1 \right\} \text{ if } x_\star(\theta_\star)^\transp\theta_\star>0\; ,\\
    &\left\{x\in\mcal{X}\, \middle\vert \, \dot\mu(x^\transp\theta_\star) \leq (2\kappa_\star(\theta_\star))^{-1} \right\}\text{ otherwise.}
    \end{aligned}\right.
\else
    \mcal{X}_- \defeq \left\vert \begin{aligned}
     &\left\{x\in\mcal{X}\, \middle\vert \, x^\transp\theta_\star \leq -1 \right\} &\text{ if } x_\star(\theta_\star)^\transp\theta_\star>0\; ,\\
    &\left\{x\in\mcal{X}\, \middle\vert \, \dot\mu(x^\transp\theta_\star) \leq (2\kappa_\star(\theta_\star))^{-1} \right\}&\text{ otherwise.}
    \end{aligned}\right.
    \fi
\end{equation*}
\end{def*}
Intuitively, \bad{} arms have a large \emph{gap} and carry little \emph{information}. In details, $\mcal{X}_-$ contains arms $x$ such that $\mu(x^\transp\theta_\star) \ll \mu(x_\star(\theta_\star)^\transp\theta_\star))$ (large gap) and $\dot{\mu}(x^\transp\theta_\star)\approx 0$ (small conditional variance). They lay in the far left-tail of the logistic function: their associated reward realization are almost always $0$. We provide an illustration of $\mcal{X}_-$ in \cref{fig:defxmoins}.

\ifaistats
\else
\input{tikz/tikz_def_xmoins}
\fi

\begin{restatable}[General Regret Upper-Bound]{thm}{thmgeneralregret}
\label{thm:generalregret}
\ifappendix
The regret of \ouralgo{} satisfies:
\begin{align*}
    \regret_{\theta_\star}(T) \leq \regretplus_{\theta_\star}(T) +  \regretmoins_{\theta_\star}(T) \; ,
    \end{align*}
    where with probability at least $1-\delta$:
    \begin{align*}
    \ifaistats
    \regretplus_{\theta_\star}(T)  \lesssim_T d\sqrt{\frac{T}{\kappa_\star}}\qquad \text{ and }\qquad 
    \regretmoins_{\theta_\star}(T)  \lesssim_T   \kappa_{\mcal{X}} d^2 \wedge \left(d^2 + \mu(x_\star(\theta_\star)^\transp\theta_\star)\sum_{t=1}^T  \mathds{1}\left(x_t\in\mcal{X}_-\right)\right)\; .
    \else
    \regretplus_{\theta_\star}(T)  \lesssim_T d\sqrt{\frac{T}{\kappa_\star}}\qquad \text{ and } \qquad 
    \regretmoins_{\theta_\star}(T)  \lesssim_T  \kappa_{\mcal{X}} d^2\wedge \left(d^2 + \mu\left(x_\star(\theta_\star\right)^\transp\theta_\star)\sum_{t=1}^T \mathds{1}\left\{x_t\in\mcal{X}_-\right\}\right)
    \fi
    \end{align*}
\else
The regret of \ouralgo{} satisfies:
\begin{align*}
    \regret_{\theta_\star}(T) \leq \regretplus_{\theta_\star}(T) +  \regretmoins_{\theta_\star}(T) \; ,
    \end{align*}
    where with high-probability:
    \begin{align*}
    \ifaistats
    \regretplus_{\theta_\star}(T)  &\lesssim_T d\sqrt{\frac{T}{\kappa_\star}}\quad \text{ and }\\
    \regretmoins_{\theta_\star}(T)  &\lesssim_T   \kappa_{\mcal{X}} d^2 \wedge \left(d^2 + \sum_{t=1}^T  \mathds{1}\left(x_t\in\mcal{X}_-\right)\right)\; .
    \else
    \regretplus_{\theta_\star}(T)  \lesssim_T d\sqrt{\frac{T}{\kappa_\star}}\qquad \text{ and } \qquad 
    \regretmoins_{\theta_\star}(T)  \lesssim_T  \kappa_{\mcal{X}} d^2\wedge \left(d^2 + \sum_{t=1}^T \mathds{1}\left\{x_t\in\mcal{X}_-\right\}\right)
    \fi
    \end{align*}\fi
\end{restatable}
The proof is deferred to \cref{subsec:thmgeneralregret}. 
\begin{rem*}[On the definition of $\mcal{X}_-$]
We use two alternative definitions for $\mcal{X}_-$ depending on the sign of $x_\star(\theta_\star)^\transp\theta_\star$. This is linked to the two regimes of the logistic function: convex on $\mbb{R}^-$ and concave on $\mbb{R}^+$. \Bad{} arms suffer from the same negative properties irrespectively of the considered case. 
\end{rem*}

\ifaistats
\paragraph{Problem-Dependent Long-Term Regret.}
\else
\subsection{Problem-Dependent Long-Term Regret}
\fi

 A striking consequence of \cref{thm:generalregret} arise for large values of the horizon $T$, when the dominating term is $\regretplus_{\theta_\star}(T)$ scaling as $d\sqrt{T/\kappa_\star}$. This is in sharp contrast with previous results as it highlights that non-linearity impacts the first-order regret's term in a positive sense. Indeed the bigger $\kappa_\star$ (cf. \cref{fig:defkappa2}) the smaller the (asymptotic) regret.
 This bound on the long-term regret is actually quite intuitive; in the asymptotic regime the algorithm mostly plays actions around $x_\star(\theta_\star)$. If the reward signal is \emph{flat} in this region, the regret should scale accordingly. It is therefore natural that the regret scales proportionally with the \emph{local} slope $\dot{\mu}(x_\star(\theta_\star)^\transp\theta_\star) =1/\kappa_\star$.

\ifaistats
\paragraph{The Long-Term Regret is Minimax.}
\else
\subsection{The Long-Term Regret is Minimax}
\fi
The scaling for the long-term regret is \emph{optimal}: we present in \cref{thm:lowerboundlocal} a matching lower-bound. In contrast to existing the lower-bounds for \textbf{LB} our lower-bound is \emph{local}: for \emph{any} nominal instance $\theta_\star$, no policy can ensure a small regret for both $\theta_\star$ and its hardest nearby alternative.\footnote{This lower-bound has a similar flavor to the lower-bound of \cite{simchowitz2020naive} in a reinforcement learning setting.} Formally, for a small constant $\epsilon>0$ let us define the \emph{local} minimax regret:
\begin{align*}
    \minmaxregret_{\theta_\star,T}(\epsilon) \!:=\! \min_\pi \max_{\ltwo{\theta-\theta_\star}\leq \epsilon} \mbb{E}\left[\regret_\theta^\pi(T)\right]\; .
\end{align*}

\begin{restatable}[Local Lower-Bound]{thm}{thmlowerboundlocal}
\label{thm:lowerboundlocal}
\ifappendix
Let $\mcal{X}=\mcal{S}_d(0,1)$. For any problem instance $\theta_\star$ and for $T\geq d^2\kappa_\star(\theta_\star)$, there exist $\epsilon_T$ small enough such that:
\begin{align*}
	 &\text{\textbf{1.}} \quad \minmaxregret_{\theta_\star,T}(\epsilon_T) = \Omega\left(d\sqrt{\frac{T}{\kappa_\star(\theta_\star)}}\right)\\
	 &\text{\textbf{2.}} \quad  \frac{5}{6}\kappa_\star(\theta_\star) \leq \kappa_\star(\theta)\leq \frac{6}{5}\kappa_\star(\theta_\star) \qquad \forall\theta\in\big\{\ltwo{\theta-\theta_\star}\leq \epsilon_T\big\}
\end{align*}
\else
Let $\mcal{X}=\mcal{S}_d(0,1)$. For any problem instance $\theta_\star$ and for $T\geq d^2\kappa_\star(\theta_\star)$, there exists $\epsilon_T$ small enough such that:
\begin{align*}
    \minmaxregret_{\theta_\star,T}(\epsilon_T) = \Omega\left(d\sqrt{\frac{T}{\kappa_\star(\theta_\star)}}\right)\; .
\end{align*}\fi
\end{restatable}

The proof is deferred to \cref{app:lowerbound}. The \emph{locality} of our lower-bound is necessary to take into account problem-dependent quantities associated with the reference point $\theta_\star$ (\emph{e.g} $\kappa_\star$). Naturally, this local lower bound implies a bound on the global minimax complexity. 


\ifaistats
\paragraph{Transitory Regret and \Bad{} Arms.}
\else
\subsection{Transitory Regret and \Bad{} Arms.}
\fi
\label{sec:regretball}
We now discuss \cref{thm:generalregret} for smaller values of the horizon $T$ and turn our attention to $\regretmoins_{\theta_\star}(T)$. In the worst-case, we retrieve the second order term of \cite{faury2020improved} - \emph{i.e} $\regretmoins_{\theta_\star}(T)\leq d^2\kappa$.
However \cref{thm:generalregret}  leaves room for improvement, stressing that $\regretmoins_{\theta_\star}(T)$ is significantly smaller when \bad{} arms $\mcal{X}_-$ are discarded fast enough. Coherently with the Bayesian analysis of \cite{dong2019performance} this is achieved by \ouralgo{} for some arm-set structures. 

\begin{restatable}{prop}{proplengthtransitory}
\label{prop:lengthtransitory}
\ifaistats
The following holds w.h.p:
\begin{align}
    \regretmoins_{\theta_\star}(T) &\lesssim_T d^2 + dK  &\text{ if } &\left\vert\mcal{X}_-\right\vert\leq K\; , \label{eq:rmoinsK}\\
    \regretmoins_{\theta_\star}(T) &\lesssim_T d^3  &\text{ if } &\mcal{X} = \mcal{B}_d(0,1)\; .\label{eq:rmoinsball}
\end{align}
\else
The following holds with high probability:
\begin{align}
    \regretmoins_{\theta_\star}(T) &\lesssim_T d^2 + dK  &\text{ if } &\left\vert\mcal{X}_-\right\vert\leq K \; ,\label{eq:rmoinsK}\\
    \regretmoins_{\theta_\star}(T) &\lesssim_T d^3  &\text{ if } &\mcal{X} = \mcal{B}_d(0,1)\; .\label{eq:rmoinsball}
\end{align}
\fi
\end{restatable}

This result formalizes that \ouralgo{} quickly discards \bad{} arms when \eqref{eq:rmoinsK} there are only a few or \eqref{eq:rmoinsball} the problem's structure is symmetric. The proof is deferred to \cref{sub:proplengthtransitory}.

\begin{rem*}[Adaptivity]
\ouralgo{} effectively \emph{adapts} to the complexity of the problem at hand: its transitory regime varies from $d^2$ to $\kappa_\mcal{X} d^2$ depending on the arm-set's geometry. To obtain similar behavior, bonus-based approaches (e.g \textnormal{\textbf{LogUCB2}}) must hard-code this complexity in the bonus, requiring one design per setting.
\end{rem*}



\paragraph{Unit Ball Case.} The following result embodies the improvement brought by our analysis; both the regret's first-order and second terms are dramatically smaller than in previous approaches (by an order of $\exp(-\ltwo{\theta_\star})$). 
\begin{restatable}[Unit-Ball Regret Upper-Bound]{thm}{thmregretball}
\label{thm:regretball}
If $\mcal{X}=\mcal{B}_d(0,1)$ the regret of \ouralgo{} satisfies:
\begin{align*}
    \regret_{\theta_\star}(T) \lesssim_T d\sqrt{\frac{T}{\kappa_\mcal{X}}} + d^2 \quad \text{w.h.p}\; .
\end{align*}
\end{restatable}


%% file: tikz/tikz_def_xmoins.tex
\makeatletter
    \pgfdeclarepatternformonly[\hatchdistance,\hatchthickness]{flexible hatch}
    {\pgfqpoint{0pt}{0pt}}
    {\pgfqpoint{\hatchdistance}{\hatchdistance}}
    {\pgfpoint{\hatchdistance-1pt}{\hatchdistance-1pt}}%
    {
        \pgfsetcolor{\tikz@pattern@color}
        \pgfsetlinewidth{\hatchthickness}
        \pgfpathmoveto{\pgfqpoint{0pt}{0pt}}
        \pgfpathlineto{\pgfqpoint{\hatchdistance}{\hatchdistance}}
        \pgfusepath{stroke}
    }
    \makeatother
\tikzset{
        hatch distance/.store in=\hatchdistance,
        hatch distance=10pt,
        hatch thickness/.store in=\hatchthickness,
        hatch thickness=2pt
    }

\begin{figure*}[t]
\subcaptionbox{Assymetrical arm-set.}{
 \begin{tikzpicture}[scale=0.8, transform shape]
     \begin{axis}[
        axis x line=center,
        axis y line=center,
        xtick=\empty,
        ytick=\empty,
        scaled ticks=false,
        clip=false,
        xmin=-4,
        xmax=1.5,
        ymin=0,
        ymax=1.1,
        xlabel=$x^\transp\theta$,
        x label style={at={(axis description cs:1.3,0)},anchor=north},
        y label style={at={(axis cs:0.2,1.1)},anchor=north},
        ylabel=$\mu$,
        title style={at={(0.75,-0.2)},anchor=north,yshift=-0.1},
       ]
        
     \node[circle,fill,scale=0.3, label=271:{$\color{black}\pmb{x_\star(\theta_\star)^\transp\theta_\star}$}] at (axis cs:0,0.0) {};
     \node[circle,fill,scale=0.3, label=270:{$\color{black}\pmb{\min_{x\in{\mathcal{X}}}x^\transp\theta_\star}$}] at (axis cs:-3,0.0) {};
     
     \draw[dotted, line width= 0.2mm] (axis cs:-3,0.06) -- (axis cs:-3,0.0);
      
     \path[name path=axis] (axis cs:-3,0) -- (axis cs:0,0);
    \addplot[domain=-3:0, black, line width=2pt, smooth, name path=f,forget plot] {0.05+1/(e^(-1.5*x)+1)};
    \addplot[gray!25, fill opacity=0.5, forget plot] fill between[of=f and axis];
    
    \path[name path=axis2] (axis cs:-3,0) -- (axis cs:-1.5,0);
    \addplot[domain=-3:-1.5, black, line width=2pt, smooth, name path=f2, forget plot] {0.05+1/(e^(-1.5*x)+1)};
    \addplot+[draw, red, pattern=flexible hatch,pattern color=red, area legend] fill between[of=f2 and axis2];
    \end{axis}
    
      \draw[gray, ultra thick, pattern=north east lines wide] (0.4,4.95)-- (0.4,6.2) arc(90:270:1.3) -- (0.4,4.95);
       \draw[red, ultra thick, pattern=north east lines wide, pattern color=red] (0.4,4.95)-- (-0.7,5.6) arc(147.5:220:1.3) -- (0.4,4.95);
       \node[] at (0.1, 6.5) {\color{gray}$\pmb{\mathcal{X}}$};
      \node[] at (0.1, 6.5) {\color{gray}$\pmb{\mathcal{X}}$};
      \draw[ultra thick, black, ->] (0.4,4.95) -- (1.4,4.95);
      \node[] at (1.7, 4.95) {$\pmb{\theta_\star}$};
     \node[] at (-0.8,3.8) {\color{red}$\pmb{\mcal{X}_-}$}; 
    \end{tikzpicture}
    }\hfill
  \subcaptionbox{Symmetrical arm-set (unit-ball).}{
   \begin{tikzpicture}[scale=0.8,transform shape]
     \begin{axis}[
        axis x line=center,
        axis y line=center,
        xtick=\empty,
        ytick=\empty,
        scaled ticks=false,
        clip=false,
        xmin=-4.5,
        xmax=4.5,
        ymin=0,
        ymax=1.1,
        xlabel=$x^\transp\theta$,
        ylabel=$\mu$,
        xlabel=$x^\transp\theta$,
        x label style={at={(axis description cs:1.3,0)},anchor=north},
       y label style={at={(axis cs:0.3,1.1)},anchor=north},
        title style={at={(0.5,-0.2)},anchor=north,yshift=-0.1}
   ]
     \node[circle,fill,scale=0.3, label=270:{$\color{black}\pmb{x_\star(\theta_\star)^\transp\theta_\star}$}] at (axis cs:3,0.0) {};
     \node[circle,fill,scale=0.3, label=270:{$\color{black}\pmb{-x_\star(\theta_\star)^\transp\theta_\star}$}] at (axis cs:-3,0.0) {};
     
     \draw[dotted, line width= 0.2mm] (axis cs:-3,0.01) -- (axis cs:-3,0.0);
     \draw[dotted, line width= 0.2mm] (axis cs:3,1) -- (axis cs:3,0.0);

     \path[name path=axis] (axis cs:-3,0) -- (axis cs:3,0);
     \path[name path=axis2] (axis cs:-3,0) -- (axis cs:-0.4,0);
    \addplot[domain=-3:3, black, line width=2pt, smooth, name path=f] {0.01+1/(e^(-1.5*x)+1)};
    \addplot[domain=-3:-0.4, black, line width=1pt, smooth, name path=f2] {0.01+1/(e^(-1.5*x)+1)};
    \addplot[gray!25, fill opacity=0.5] fill between[of=f and axis];
   
    \addplot[draw, red, pattern=flexible hatch,pattern color=red] fill between[of=f2 and axis2];
    
    \draw[ultra thick, gray, pattern=north east lines wide] (axis cs:-5,1) circle (1.2cm);
    \node[] at (axis cs:-5,1.3) {$\color{gray}\pmb{\mathcal{X}}$};
    \draw[->, ultra thick, black] (axis cs:-5,1) -- (axis cs:-3,1);
    \node[] at (axis cs: -2.5,1) {$\pmb{\theta_\star}$};
    \node[] at (axis cs: -5.5,0.7) {\color{red}$\pmb{\mcal{X}_-}$};
    \end{axis}
    
    \draw[red, ultra thick, pattern=north east lines wide, pattern color=red] (-0.35,5.2)-- (-0.6,6.36) arc(100:255:1.2) -- (-0.35,5.2);
    
    \end{tikzpicture}
}
\caption{Graphical illustration of $\mcal{X}_-$.}
\label{fig:defxmoins}
\end{figure*}
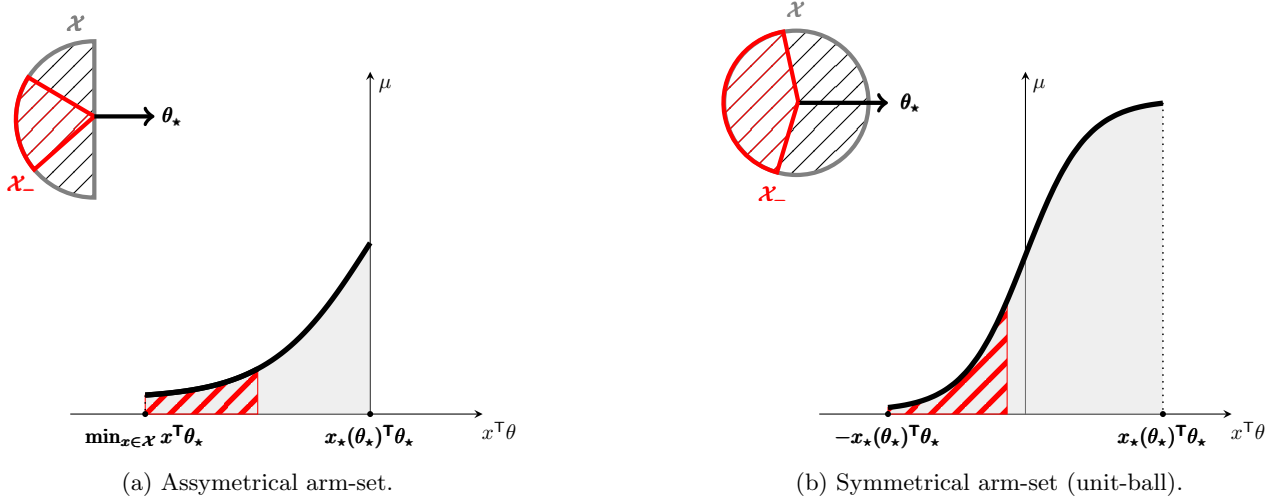

%% file: hli.tex
\section{HIGH LEVEL IDEAS}
\label{sec:sketch}

\subsection{Key Arguments behind Theorem 1}
We provide here the main ideas behind the proof of \cref{thm:generalregret}. We assume that the high probability event $\{\theta_\star\in\mcal{C}_t(\delta)\}$ holds. The optimistic nature of the pair $(x_t,\theta_t)$ along with a second-order Taylor expansion of the regret yields:

\ifaistats
\begin{equation*}
\begin{aligned}
	\regret_{\theta_\star}(T) \leq &\underbrace{\sum_{t=1}^T \dot{\mu}(x_t^\transp\theta_\star)x_t^\transp(\theta_t-\theta_\star)}_{\regretplus_{\theta_\star}(T)} \\+ &\underbrace{\sum_{t=1}^T\ddot{\mu}(z_t)\{\theta_\star^\transp(x_\star(\theta_\star)-x_t)\}^2}_{\regretmoins_{\theta_\star}(T)} \; .
\end{aligned}
\end{equation*}
\else
\begin{align}
\label{eq:ps1}
	\regret_{\theta_\star}(T) \leq\underbrace{\sum_{t=1}^T \dot{\mu}(x_t^\transp\theta_\star)x_t^\transp(\theta_t-\theta_\star)}_{R_1(T)} + \underbrace{\sum_{t=1}^T\ddot{\mu}(z_t)\{\theta_\star^\transp(x_\star(\theta_\star)-x_t)\}^2}_{R_2(T)} \; .
\end{align}
\fi
where $z_t\in[x_t^\transp\theta_\star,x_\star(\theta_\star)^\transp\theta_\star]$. 

We start by examining $\regretplus_{\theta_\star}(T)$. Leveraging the self-concordance property of the logistic function (cf \cref{app:sc}) and the structure of $\mcal{C}_t(\delta)$ one gets:
\begin{align*}
\ifaistats
	\regretplus_{\theta_\star}(T) &\lesssim_T \sqrt{d} \sum_{t=1}^T \dot{\mu}(x_t^\transp\theta_\star)\left\lVert x_t\right\rVert_{\mbold{H_t^{-1}(\theta_\star)}} \; ,\\
	&\lesssim_T d \sqrt{\sum_{t=1}^T \dot{\mu}(x_t^\transp\theta_\star)}\; .
\else
R_1(T) \lesssim_T \sqrt{d} \sum_{t=1}^T \dot{\mu}(x_t^\transp\theta_\star)\left\lVert x_t\right\rVert_{\mbold{H_t^{-1}(\theta_\star)}} 
	\lesssim_T d \sqrt{\sum_{t=1}^T \dot{\mu}(x_t^\transp\theta_\star)}\; .
\fi
\end{align*}
where we last used the Elliptical Potential Lemma (cf. \cref{app:aux}) and Cauchy-Schwarz inequality.\\
A brutal bound of the type $\dot{\mu}\leq 1/4$ yields $\regretplus_{\theta_\star}(T)\lesssim_T d\sqrt{T}$ and retrieves the first order term in \citep{faury2020improved}. This bound is however considerably \emph{loose}: an asymptotically optimal strategy often plays $x_\star(\theta_\star)$ (or relatively close actions). Therefore most of the time $\dot{\mu}(x_t^\transp\theta_\star) \approx \dot{\mu}(x_\star(\theta_\star)^\transp\theta_\star)=\kappa_\star^{-1}$. Formalizing this intuition (cf. \cref{subsec:thmgeneralregret}) yields:
\begin{align*}
\regretplus_{\theta_\star}(T) \lessapprox d\sqrt{\frac{T}{\kappa_\star}}\; .
\end{align*}

We now investigate $\regretmoins_{\theta_\star}(T)$. First, note that a crude upper-bound directly yields an explicit dependency in $\kappa_\mcal{X}$: from the boundedness of $\vert\ddot{\mu}\vert$ one obtains
\begin{align*}
	\regretmoins_{\theta_\star}(T) \lessapprox d\sum_{t=1}^T  \left\lVert x_t\right\rVert_{\mbold{H}_t^{-1}(\theta_\star)}^2 \lessapprox d^2 \kappa_\mcal{X} \; .
\end{align*}
where we used $\mbold{H_t}(\theta_\star)\succeq \kappa^{-1}_\mcal{X}\sum_{s=1}^{t-1}x_sx_s^\transp$ along with the Elliptical Potential Lemma. While it may be unimprovable in some cases, this bound is particularly pessimistic as it discards the good cases where $\ddot{\mu}(z_t)$ and $\mbold{H_t(\theta_\star)}$ compensate each other. \\
We first illustrate this fact with an extreme argument: if $x_t^\transp\theta_\star\geq0$ for all $t$ then $z_t\geq 0$ and $\ddot{\mu}(z_t)\leq 0$. In this case we obtain $\regretmoins_{\theta_\star}(T)\leq 0$. This suggests that in more general scenarios the arms $\mcal{X}$ should be classified depending on their position w.r.t $\theta_\star$. Along with the previous example, this idea hints towards decomposing $\regretmoins_{\theta_\star}(T)$ as follows:
\begin{align*}
	\regretmoins_{\theta_\star}(T)  &\leq \sum_{t=1}^T\ddot{\mu}(z_t)\{\theta_\star^\transp(x_\star(\theta_\star)-x_t)\}^2\mathds{1}\left\{x_t^\transp\theta_\star\leq 0\right\}\; ,\\
	&\lessapprox  \sum_{t=1}^T \mathds{1}\left\{x_t^\transp\theta_\star\leq 0\right\}\; .
\end{align*}
where we last used the self-concordance of $\mu$. The main point of this last inequality is that $\regretmoins_{\theta_\star}(T) $ is linked to the number of times the algorithm played \bad{} arms. As long as there are few such actions one can therefore expect a good algorithm to have a small associated $\regretmoins_{\theta_\star}(T)$ - this is the point of \cref{prop:lengthtransitory}. The illustrative discussion we are displaying here is formalized in \cref{thm:generalregret} by introducing a finer and more general definition for \bad{} arms $\mcal{X}_-$. 

\subsection{Key Arguments behind Theorem 2}

We discuss here the construction of our local lower-bound. Let $\theta_\star$ denote a fixed nominal instance and $\pi$ a policy which has low-regret when playing against $\theta_\star$. Our strategy is to find an alternative problem $\theta'$ which satisfies the two following \emph{conflicting} criteria: \textbf{(1)} $\pi$ has the same behavior against both $\theta_\star$ and $\theta'$ and \textbf{(2)} $\theta'$ is \emph{far} from $\theta_\star$ so that the optimal arms $x_\star(\theta_\star)$ and $x_\star(\theta')$ significantly \emph{differ}. 

When playing against $\theta_\star$, we can expect $\pi$ to produce a trajectory where most of the time $x_t\approx x_\star(\theta_\star)$. Indeed since:
\begin{align*}
	\regret_{\theta\star}^\pi(T) \propto \sum_{t=1}^T \ltwo{x_t-x_\star(\theta_\star)}^2\; ,
\end{align*}
a small regret against $\theta_\star$ implies an accurate tracking of $x_\star(\theta_\star)$. Notice that when $\mcal{X} = \mcal{B}_d(0,1)$ we have $x_\star(\theta_\star)$ is co-linear with $\theta_\star$.
As a consequence there are $d-1$ directions (orthogonal to $\theta_\star$) where $\theta_\star$ is poorly estimated. This suggest that parameters laying in $\mcal{H}_\perp^\star$ (the hyperplane supported by $\theta_\star$, cf. \cref{fig:tikzlb}) can easily be confused with $\theta_\star$ for the policy $\pi$. This notion of \emph{distinguishability} between parameters can be formalized through a discrepancy measures $d_T(\theta_\star,\theta')$ which quantifies how easy it is for $\pi$ to determine if the rewards it receives are generated by either $\theta_\star$ or $\theta'$. For any $\theta'\in\mcal{H}_\perp^\star$ it scales as follow:
\begin{align*}
	d_T(\theta_\star,\theta') \approx \sqrt{\frac{T}{\kappa_\star(\theta_\star)}}\ltwo{\theta_\star-\theta}^2
\end{align*}

\input{tikz/tikz_lb.tex}

This scaling is rather intuitive; the larger $T$, the more occasions for $\pi$ to separate $\theta_\star$ from $\theta'$. Further, the larger $\kappa_\star$, the smaller the conditional variance of the rewards and the longer it takes to correctly estimate an arm's mean reward and determine wether it was generated by $\theta_\star$ or $\theta'$. To satisfy $\textbf{(1)}$ we must choose $\theta'$ so that $d_T(\theta_\star,\theta')$ is small; the trade-off with \textbf{(2)} suggests picking $\theta'$ such that: 
\begin{align}
	\ltwo{\theta'-\theta_\star}^2 \approx \sqrt{\frac{\kappa_\star(\theta_\star)}{T}}
	\label{eq:lb_sketch}
\end{align}
Under such conditions, $\pi$ cannot separate $\theta_\star$ from $\theta'$ and must therefore \emph{act} similarly against both parameters (i.e most of the time we will have $x_t\approx x_\star(\theta_\star)$ against $\theta'$). Easy computations show that the regret of $\pi$ against $\theta'$ then writes:
\begin{align*}
\regret_{\theta'}^\pi(T) &\approx \frac{1}{\kappa_\star(\theta_\star)} \sum_{t=1}^T \ltwo{x_t-x_\star(\theta')}^2\\
 &\approx \frac{1}{\kappa_\star(\theta_\star)} \sum_{t=1}^T \ltwo{x_\star(\theta_\star)-x_\star(\theta')}^2\\
&\approx \frac{1}{\kappa_\star(\theta_\star)} T \ltwo{\theta_\star-\theta'}^2
\end{align*}
which gives the announced behavior after replacing $\ltwo{\theta_\star-\theta'}$ by the scaling suggested by the trade-off between \textbf{(1)} and \textbf{(2)} presented in \cref{eq:lb_sketch}.

%% file: tikz/tikz_lb.tex
\begin{figure}
\centering
 \begin{tikzpicture}[scale=0.8, transform shape]
     \begin{axis}[
        axis x line=none,
        axis y line=none,
        xtick=\empty,
        ytick=\empty,
        scaled ticks=false,
        clip=false,
        xmin=-2,
        xmax=2,
        ymin=-2,
        ymax=2,
        xlabel=,
        ylabel=
   ]      
     \draw[ultra thick, black, pattern=north east lines wide] (axis cs:0,0) circle (1.5cm);
     \draw[ultra thick, black] (axis cs:0,0) circle (1.55cm);
     
     \draw[ultra thick, red] (axis cs:0,0) -- (axis cs:1.2,0);
     \draw[ultra thick, red, dashed] (axis cs:1.22,0) -- (axis cs:1.75,0);
      \draw[red] (axis cs:1.75,0) -- (axis cs:3,0);
     \draw[gray, fill=gray!20, fill opacity=0.5, thick] (axis cs:1.2,-2.1) -- (axis cs:2.4,-1.6) --  (axis cs:2.4,2.3) -- (axis cs:1.2,1.8) -- cycle;
    
     \draw[color=blue, thick, fill=blue, opacity=0.2] plot [smooth cycle, tension=0.9] coordinates {(axis cs: 1.5,-1.3) (axis cs: 2.1,-1.3) (axis cs: 2.1,1.3) (axis cs: 1.4,1.5)} node at (axis cs:0,0) {};
     
      \node[circle,fill,scale=0.3, red, label=271:{{$\color{red}\pmb{\theta_\star}$}}] at  (axis cs: 1.75,0) {} ;
     \draw[] (axis cs:1.85,0) -- (axis cs:1.85,0.2) --  (axis cs:1.7,0.2);
     \draw [->,decorate,decoration={snake,amplitude=.4mm,segment length=2mm,post length=1mm}, blue, thick] (axis cs:2.6,1.7) -- (axis cs:2,1.3);
      \draw [ultra thick, blue, dashed] (axis cs:1.15,1.36) -- (axis cs:1.6,1.9);
      \draw [ultra thick, blue] (axis cs:0,0) -- (axis cs:1.15,1.36);
      \node[circle,fill,scale=0.3, blue, label=90:{{$\color{blue}\pmb{\theta'}$}}] at  (axis cs: 1.6,1.88) {} ;
      
   \node [] at (axis cs: 3.4, 1.8) {$\color{blue}\pmb{\{d_T(\theta_\star,\theta')\leq 1\}}$};  
   \node[] at (axis cs:0,1.4) {$\pmb{\mathcal{X}}$};
   \node[] at (axis cs:2.6, 2.8) {$\color{gray}\pmb{\mathcal{H}_\perp^\star}$};
   
    \end{axis}
    \end{tikzpicture}
    \caption{Illustration of the construction behind the local lower-bound.}
     \label{fig:tikzlb}
\end{figure}
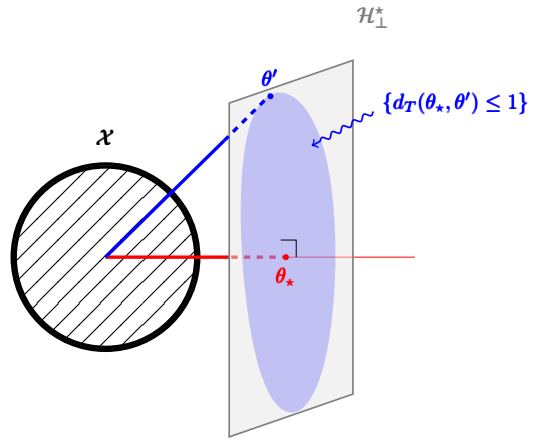

%% file: relaxation.tex
\section{TRACTABILITY THROUGH CONVEX RELAXATION}
\label{sec:tractable}

\input{tikz/tikz_conf_set.tex}

The optimization program presented in \cref{eq:ouralgo} and to be solved by \ouralgo{} is challenging. Indeed, the constraint $\theta\in\mcal{C}_t(\delta)$ is non-convex and therefore there exist no standard approach for provably approximately solving this program. 

\paragraph{A Convex Relaxation.} We circumvent this issue by designing a \emph{convex relaxation} for the set $\mcal{C}_t(\delta)$:
\begin{align*}
	\mcal{E}_t(\delta) \defeq \left\{\theta\in\Theta\, \middle\vert \, \mcal{L}_t(\theta)-\mcal{L}_t(\hat\theta_t) \leq \beta_t(\delta)^2 \right\}\; .
\end{align*}
where $\beta_t(\delta)\defeq\gamma_t(\delta) + \gamma_t^2(\delta)/\sqrt{\lambda_t}$. The convexity of the log-loss immediatly implies that $\mcal{E}_t(\delta)$ is convex (illustrated in \cref{fig:confset}). The following statement ensures that \textbf{(1.)} it does relax the confidence set $\mcal{C}_t(\delta)$ yet \textbf{(2.)} preserves core concentration guarantees.

\begin{restatable}{lemma}{lemmaouralgorelaxed}
\label{lemma:ouralgorelaxed}
\ifappendix
The following statements hold:
\begin{enumerate}
\item $\mcal{C}_t(\delta) \subseteq \mcal{E}_t(\delta)$ for all $t\geq 1$ and therefore $\mbb{P}\left(\forall t\geq 1,\, \theta_\star\in\mcal{E}_t(\delta) \right)\geq 1-\delta$.
\item With probability at least $1-\delta$, we have:
\begin{align*}
\forall\theta\in\mcal{E}_t(\delta), \; \left\lVert \theta-\theta_\star\right\rVert_{\mbold{H_t(\theta_\star)}} \leq (2+2S)\gamma_t(\delta) + 2\sqrt{1+S}\beta_t(\delta) \, .
\end{align*}
Therefore if $\lambda_t = d\log(t)$ with probability at least $1-\delta$: 
\begin{align*}
	\forall\theta\in\mcal{E}_t(\delta), \, \left\lVert \theta-\theta_\star\right\rVert_{\mbold{H_t(\theta_\star)}} = \bigo{\sqrt{d\log(t)}} \; .
\end{align*}
\end{enumerate} 
\else
The following statements hold:
\begin{enumerate}
\item $\mcal{C}_t(\delta) \subseteq \mcal{E}_t(\delta)$.
\item $\forall\theta\in\mcal{E}_t(\delta)$: $\left\lVert \theta-\theta_\star\right\rVert_{\mbold{H_t(\theta_\star)}} = \mcal{O}(\sqrt{d\log(t)})$ w.h.p.
\end{enumerate} 
\fi
\end{restatable}
The proof is deferred to \cref{subsec:convexrelaxation}. 

\paragraph{Relaxing the Optimistic Planning.} 
Building on $\mcal{E}_t(\delta)$ we obtain \ouralgorelaxed{} where the planning is performed as follows:
\begin{align}
	(x_t,\tilde\theta_t) \in \argmax_{x\in\mcal{X},\theta\in\mcal{E}_t(\delta)} x^\transp\theta\; .
	\label{eq:relax.opt.planning}
\end{align}
Note the similarities with the \textbf{OFUL} algorithm of \cite{abbasi2011improved}; the planning consists in the minimization of a \emph{bilinear} objective under \emph{convex} constraints.  While solving the program presented in \cref{eq:relax.opt.planning} remains challenging in general, a tractable procedure can be developed for finite arm-sets - summarized in~\cref{algo:ouralgorelaxed}.  The following proposition guarantees that it effectively guarantees optimism. 

\begin{restatable}{prop}{proptractablealgo}
\label{prop:tractablealgo}
Let $(\tilde{x}_t,\tilde\theta_t)$ be the pair returned by \cref{algo:ouralgorelaxed}. Then:
\begin{align*}
	(\tilde{x}_t,\tilde\theta_t) \in  \argmax_{x\in\mcal{X},\theta\in\mcal{E}_t(\delta)} x^\transp\theta\; .
\end{align*} 
\end{restatable}
The main complexity of \cref{algo:ouralgorelaxed} reduces to maximizing a linear objective under convex constraints. The maximizer can therefore be found efficiently by solving the dual problem.

\paragraph{Regret Guarantees.} We conclude this section with \cref{cor:tractableguarantees} proving that relaxing the original optimistic search does not impact the learning performances thus recovering the guarantees of \ouralgo{}. 

\begin{restatable}{cor}{cortractableguarantees}
\label{cor:tractableguarantees}
\cref{thm:generalregret}, \cref{prop:lengthtransitory} and \cref{thm:regretball} are also satisfied by \ouralgorelaxed. 
\end{restatable}
This claim directly follows from \cref{lemma:ouralgorelaxed}.

\ifaistats
\input{ouralgorelaxed.tex}
\else\fi

%% file: tikz/tikz_conf_set.tex
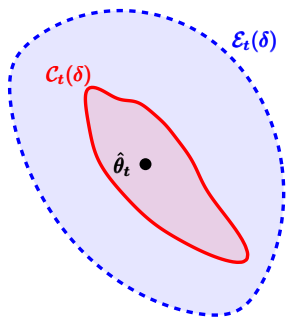
\begin{figure}[t]
\centering
\begin{tikzpicture}[scale=0.8, transform shape]
   \begin{axis}[
    axis x line=none,
    axis y line=none,
    xtick=\empty,
    ytick=\empty,
    scaled ticks=false,
    xmin=-20,
    xmax=20,
    ymin=-15,
    ymax=15,
    clip=false,]
    \addplot[smooth, ultra thick, blue, dashed, fill=blue, fill opacity=0.1] table [col sep=comma,x=x, y=y] {tikz/et.dat};
    \node[] at (axis cs: 13,13) {$\color{blue}\pmb{\mcal{E}_t(\delta)}$};
  \addplot[smooth, ultra thick, red, fill=red, fill opacity=0.1] table [col sep=comma,x=x, y=y] {tikz/ct.dat};
    \node[] at (axis cs: -5,10) {$\color{red}\pmb{\mcal{C}_t(\delta)}$};

   \node[label=180:{{$\pmb{\hat\theta_t}$}},circle,fill,inner sep=2pt] at (axis cs:2.35,2.24]) {};
   
  \end{axis}
\end{tikzpicture}
\caption{The confidence set $\mcal{C}_t(\delta)$ and its convex relaxation $\mcal{E}_t(\delta)$ obtained through a trajectory with: $T=1000$, $\mcal{X}=\mcal{B}_d(0,1)$ and $\kappa_\mcal{X}=22$.}
\label{fig:confset}
\end{figure}

%% file: ouralgorelaxed.tex
\begin{algorithm}[t]
   \caption{Planning for \ouralgorelaxed{}}
   \label{algo:ouralgorelaxed}
\begin{algorithmic}
   \STATE {\bfseries input:} finite arm-set $\mcal{X}$, set $\mcal{E}_t(\delta)$.
   \FOR{$x\in\mcal{X}$}
   \STATE Solve $\theta_x \leftarrow \argmax_{\theta\in\mcal{E}_t(\delta)} x^\transp\theta$.
   \ENDFOR
   \STATE Compute $\tilde{x} \leftarrow \argmax_{x\in\mcal{X}} x^\transp\theta_x$.
   \RETURN $(\tilde{x}, \theta_{\tilde{x}})$.
\end{algorithmic}
\end{algorithm}

%% file: conclusion.tex
\section{CONCLUSION}

In this paper we bring forward an improved characterization of the regret minimization problem in Logistic Bandit through the lense of \ouralgo, a parameter-based optimistic algorithm. Our analysis further describes the impact of non-linearity on the exploration-exploitation trade-off. For a large number of settings, we show that non-linearity \emph{eases} regret minimization in \textbf{LogB}. This is embodied by the $\bigo{\sqrt{T/\kappa_\star}}$ upper-bound of \ouralgo, which we show is optimal by proving a matching, local and problem-dependent lower-bound. Such rates are however conditioned on reaching a permanent regime. The regret associated with the transitory phase acts as a second-order term tied to problem-dependent quantities. 

\paragraph{Generalized Linear Bandits.} Part of the findings presented here can be easily extended to other generalized linear bandits (namely the $\bigo{\sqrt{T/\kappa_\star}}$ rate) however with potentially different conclusions. The findings related to the transitory regime are however specific to Logistic Bandits. In general, we believe that attempting to treat all generalized linear bandits in a model-agnostic approach is sub-optimal for a fine characterization of the non-linearity's effect. This should be done in a problem-dependent fashion, relative and specific to the considered model and the singularities behind its non-linear nature.

\paragraph{Efficient Algorithms.} An interesting avenue for future work resides in modifying the arguments presented here to develop order-optimal yet fully online algorithms for \textbf{LogB}. Jointly achieving efficiency and regret minimax-optimality is still an open question. Improving guarantees for online logistic regression (under a well-specification assumption) and marrying them with our analysis seems like a promising direction to complete this goal.

%% file: appendix/appendix.tex
\input{appendix/app_orga}

\input{appendix/app_notations}

\input{appendix/app_confidence_set}

\input{appendix/app_regret_bounds}

\input{appendix/app_lower_bounds} 

\input{appendix/app_tractability}

\input{appendix/app_self_concordance}

\input{appendix/app_aux}

\input{appendix/app_exp}

%% file: appendix/app_orga.tex
\aistatstitle{Instance-Wise Minimax-Optimal Algorithms for Logistic Bandits \\ Supplementary Material}

\section*{\uppercase{Organization of the Appendix}}
This appendix is organized as follows:
\begin{itemize}[noitemsep]
    \item In \cref{app:notations} we introduce useful notations, and introduce some central inequalities.
    \item In \cref{app:confidenceset} we prove that $\mcal{C}_t(\delta)$ and $\mcal{E}_t(\delta)$ are confidence sets for $\theta_\star$. 
    \item In \cref{app:regretbound} we prove the different regret  upper-bounds announced in the main manuscript. 
    \item In \cref{app:lowerbound} we prove the regret lower-bound.
     \item In \cref{app:tractability} we give some guarantees for the optimistic solving of $\ouralgorelaxed{}$.
     \item In \cref{app:sc} we prove some key self-concordance results.
    \item In \cref{app:aux} we introduce and prove some auxiliary results, needed for the analysis. 
    \item In \cref{app:exp} we display illustrative numerical experiments. 
\end{itemize}

\vfill 
{\hypersetup{linkcolor=black}
\parttoc}
\vfill

%% file: appendix/app_notations.tex
\newpage
\section{\uppercase{Notations and First Inequalities}}
\label{app:notations}
We collect here a list of symbols and definitions that will be used throughout this appendix. Recall the definition of the regularized logistic loss given a sequence of vectors $\{x_i\}_{i=1}^{t-1}$, rewards $\{r_i\}_{i=2}^{t}$ and a (predictable) regularization parameter $\lambda_t$:
\begin{align*}
    \mcal{L}_t(\theta) := -\sum_{s=1}^{t-1}\left[r_{s+1}\log\mu( x_s^\transp \theta) + (1-r_{s+1})\log(1-\mu(x_s^\transp\theta))\right] + \frac{\lambda_t}{2}\ltwo{\theta}^2\; .
\end{align*}
$\mcal{L}_t(\theta)$ being a strictly convex and coercive function, we can safely define $\hat\theta_t\!=\!\argmin_{\theta\in\mbb{R}^d}\mcal{L}_t(\theta)$. Define also for all $\theta\in\mbb{R}^d$:
\begin{align*}
    g_t(\theta) := \sum_{s=1}^{t-1}\mu( x_s^\transp\theta)x_s + \lambda_t\theta \; ,  \qquad \mbold{H_t}(\theta) := \sum_{s=1}^{t-1}\dot{\mu}(x_s^\transp\theta)x_sx_s^\transp + \lambda_t\mbold{I_d}\; .
\end{align*}
For all $x,\theta_1,\theta_2\in\mbb{R}^d$ let:
\begin{align*}
    \alpha(x,\theta_1,\theta_2) &\defeq \int_{v=0}^1 \dot{\mu}\left(x^\transp\theta_1 + vx^\transp(\theta_2-\theta_1)\right)dv \; ,\\
    \tilde\alpha(x,\theta_1,\theta_2) &\defeq \int_{v=0}^1 (1-v)\dot{\mu}\left(x^\transp\theta_1 + vx^\transp(\theta_2-\theta_1)\right)dv \; ,\\
    \mbold{G_t}(\theta_1,\theta_2) &\defeq \sum_{s=1}^{t-1}\alpha(x_s,\theta_1,\theta_2)x_sx_s^\transp + \lambda_t \mbold{I_d}  \; ,\\
    \mbold{\widetilde{G}_t}(\theta_1,\theta_2) &\defeq \sum_{s=1}^{t-1}\tilde{\alpha}(x_s,\theta_1,\theta_2)x_sx_s^\transp + \lambda_t \mbold{I_d} \; .\\
\end{align*}

Note that since $\mu$ is strictly increasing (and therefore $\dot\mu\geq 0$) we easily have $\alpha(x,\theta,\theta_2)\geq \tilde{\alpha}(x,\theta,\theta_2)$. It easily follows that $\mbold{G_t}(\theta_1,\theta_2) \succeq \mbold{\widetilde{G}_t}(\theta_1,\theta_2)$. Thanks to the mean-value theorem, we also have for all $\theta_1,\theta_2$:
\begin{align}
\label{eq:mvt}
    g_t(\theta_1)-g_t(\theta_2) = \mbold{G_t}(\theta_1,\theta_2)(\theta_1-\theta_2)\; .
\end{align}
Also, thanks to \cref{lemma:firstselfconcordance,lemma:secondselfconcordance} we have the following inequalities for any $\theta_1,\theta_2\in\Theta$:
\begin{align}
    \mbold{G_t}(\theta_1,\theta_2) &\succeq (1+2S)^{-1}\mbold{H_t}(\theta) \text{ for } \theta\in\{\theta_1,\theta_2\} \label{eq:lowerboundGt} \; ,\\
    \mbold{\widetilde{G}_t}(\theta_1,\theta_2) &\succeq (2+2S)^{-1}\mbold{H_t}(\theta_1) \; .\label{eq:lowerboundGttilde}
\end{align}

We will also use the notation:
\begin{align*}
	\mbold{V}_t \defeq \sum_{s=1}^{t-1} x_sx_s^\transp +\lambda_t\mbold{I_d}
\end{align*}
Thanks to the inequality $\dot{\mu}(x^\transp\theta_1)\geq \kappa^{-1}_{\mcal{X}}(\theta_1)$ for any $x\in\mcal{X}$ along with $\kappa_\mcal{X}(\theta_1)\!>\!1$ for any $\theta_1$, we have:
\begin{align}
	\mbold{H}_t(\theta_1) \succeq \kappa^{-1}_\mcal{X}(\theta_1) \mbold{V_t} 
	\label{eq:upperboundVt}
\end{align}

%% file: appendix/app_confidence_set.tex
\newpage
\section{CONFIDENCE SETS}
\label{app:confidenceset}

\subsection{Concentration Inequality}
Our results build on the concentration inequality of \cite[Theorem 1]{faury2020improved}. We present below a marginally modified version inspired from the proof of Theorem 1 in \citep{russac2019weighted}, which allows for time-varying (yet predictable) regularization (without ressorting to union bounds). In time, this will allow us to design near-optimal algorithms without the knowledge of the horizon $T$.

\begin{restatable}{thm}{thmconcentration}
\label{thm:concentration}
Let $\{\mcal{F}_t\}_{t=1}^{\infty}$ be a filtration. Let $\{x_t\}_{t=1}^{\infty}$ be a stochastic process in $\mcal{B}_2^d(1)$ such that $x_t$ is $\mcal{F}_t$-measurable. Let $\{\varepsilon_t\}_{t=2}^{\infty}$ be a real-valued martingale difference sequence such that $\varepsilon_t$ is $\mcal{F}_t$-measurable. Further, assume $\vert \varepsilon_t \vert\leq 1$ holds almost surely for all $t\geq 2$ and denote $\sigma_t
^2=\mathbb{E}\left[\varepsilon_t^2 \vert \mcal{F}_t\right]$. Let $\{\lambda_t\}_{t=1}^\infty$ be a predictable sequence of non-negative scalars. Define:
\begin{align*}
    S_t := \sum_{s=1}^{t-1} \varepsilon_s x_s\, , \qquad \mbold{H_t} = \sum_{s=1}^{t-1} \sigma_s^2 x_sx_s^\transp + \lambda_t\mbold{I_d}\, .
\end{align*}
Then for any $\delta\in(0,1]$:
\begin{align*}
    \mbb{P}\left(\exists t\in\mbb{N} \text{ s.t } \left\lVert S_t \right\rVert_{\mbold{H_t^{-1}}} \geq \frac{2}{\sqrt{\lambda_t}}\log\left(\frac{2^d\lambda_t^{\nicefrac{-d}{2}}\det\left(\mbold{H_t}\right)^{\nicefrac{1}{2}}}{\delta}\right)+\frac{\sqrt{\lambda_t}}{2}\right)\leq \delta \, .
\end{align*}
\end{restatable}

\begin{proof}
The proof essentially follows the proof of Theorem 1 in \cite{faury2020improved}, up to a minor modification to allow for a time-varying regularization. In the following, denote $\mbold{\bar{H}_t} \defeq \sum_{s=1}^{t-1}\sigma_s^2x_sx_s^\transp$
and for all $\xi\in\mcal{B}_d(0,1)$ let:
\begin{align*}
    M_0(\xi) = 1\quad \text{and} \quad M_t(\xi) := \exp\left(\xi^\transp S_t - \left\lVert \xi\right\rVert_{\mbold{\bar{H}_t}}^2\right)\, \, \forall t\geq 1\; .
\end{align*}
We know thanks to Lemma 5 of \cite{faury2020improved} that $M_t(\xi)$ is a super-martingale and hence checks $\mbb{E}\left[M_t(\xi)\right]\leq 1$ for all $\xi\in\mcal{B}_d(0,1)$. Further, let $g_t(\xi)$ be the density of the normal distribution of precision $2\mbold{H_t}$ truncated on the ball $\mcal{B}_d(0,1/2)$ and let:
\begin{align*}
    \bar{M}_t = \int M_t(\xi)g_t(\xi)d\xi\; .
\end{align*}
Note that $\bar{M}_t$ is not (in all generality) a super-martingale - this is where our analysis differs from \citep{faury2020improved}. This however doesn't hurt the final result as one can still apply an appropriate stopping time construction. Let $\tau$ be a stopping time with respect to $\{\mcal{F}_t\}_t$. One can easily check (see for instance the proof of Theorem 1 in \cite{abbasi2011improved}) that $M_\tau(\xi)$ is well-defined and $\mbb{E}\left[M_\tau(\xi)\right]\leq 1$ for all $\xi\in\mcal{B}_d(0,1/2)$, . Clearly we have:
\begin{align*}
    \mbb{E}\left[ \bar{M}_\tau\right] = \int \mbb{E}\left[M_\tau(\xi)\right]g_\tau(\xi)d\xi \leq 1\; .
\end{align*}
Following the proof of Theorem 1 in \cite{faury2020improved}, computing $\bar{M}_\tau$ eventually leads us to:
\begin{align*}
    \mbb{P}\left(\left\lVert S_\tau\right\rVert_{\mbold{H_\tau}} \leq \frac{\sqrt{\lambda_\tau}}{2} + \frac{2}{\sqrt{\lambda_\tau}}\log\left(\frac{2^d\det(\mbold{H_\tau})^{1/2}}{\delta\lambda_\tau^{d/2}}\right)\right) \geq 1-\delta\; .
\end{align*}
From there, directly following the stopping time construction in the proof of Theorem 1 in \cite{abbasi2011improved} yields the announced result.
\end{proof}


\subsection{Confidence Set}
Recall the confidence set definition:
\begin{align*}
    \mcal{C}_t(\delta) &= \left\{\theta\in\Theta \, \middle\vert \, \left\lVert g_t(\theta) - g_t(\hat\theta_t)\right\rVert_{\mbold{H_t^{-1}}(\theta) }  \leq \gamma_t(\delta)\right\} \; ,
\end{align*}
where:
 \begin{align}
 	\gamma_t(\delta) = \sqrt{\lambda_t}(S+\frac{1}{2}) + \frac{d}{\sqrt{\lambda_t}}\log\left(\frac{4}{\delta}\left(1+\frac{t}{16d\lambda_t}\right)\right)\; .
	\label{eq:defgamma}
\end{align}

\propconfidenceset*

\begin{proof}
We trivially have:
\begin{align*}
	\Big\{ \forall t\geq 1,\, \theta_\star\in\mcal{C}_t(\delta) \Big\} = E_\delta 
\end{align*}
    From the optimality conditions of $\hat\theta_t$ one easily gets that $g_t(\hat\theta_t) = \sum_{s=1}^{t-1}r_{s+1}x_s$.
Therefore:
    \begin{align*}
        \left\lVert g_t(\hat\theta_t)-g_t(\theta_\star)\right\rVert_{\mbold{H_t^{-1}}(\theta_\star)} &= \left\lVert \sum_{s=1}^{t-1}\left(r_{s+1}-\mu(x_s^\transp\theta_\star)\right)x_s - \lambda_t\theta_\star\right\rVert_{\mbold{H_t^{-1}}(\theta_\star)} &\\
        &\leq \sqrt{\lambda_t} S + \left\lVert \sum_{s=1}^{t-1}\varepsilon_{s+1}x_s\right\rVert_{\mbold{H_t^{-1}}(\theta_\star)}\, , &\left(\ltwo{\theta_\star}\leq S,\, \mbold{H_t}(\theta_\star) \succeq \lambda_t\mbold{I_d}\right)
    \end{align*}
    where we defined for all $s\geq 1$: $\varepsilon_{s+1} \defeq r_{s+1}-\mu(x_s^\transp\theta_\star)$.
Remember that conditionally on $\mcal{F}_s$ the rewards are such that $r_{s+1}\sim \text{Bernoulli}(\mu(x_s^\transp\theta_\star))$. Therefore:
\begin{equation*}\left\{
\begin{aligned}
	\mbb{E}\left[\epsilon_{s+1}\middle\vert \mcal{F}_s\right] &= 0 \; ,\\
	\mbb{V}\text{ar}\left[\epsilon_{s+1}\middle\vert \mcal{F}_s\right] &= \mu(x_s^\transp\theta_\star)(1-\mu(x_s^\transp\theta_\star)) = \dot\mu(x_s^\transp\theta_\star)\; .
\end{aligned}\right.
\end{equation*}
If we define $S_t\defeq \sum_{s=1}^{t-1}\varepsilon_{s+1}x_s$ and $\mbold{H_t} = \mbold{H_t}(\theta_\star)$ all conditions of \cref{thm:concentration} are met and we have: 
\begin{align*}
    1-\delta &\geq \mbb{P}\left(\forall t\geq 1, \, \left\lVert S_t\right\rVert_{\mbold{H_t^{-1}}(\theta_\star)} \leq \frac{2}{\sqrt{\lambda_t}}\log\left(\frac{2^d\lambda_t^{\nicefrac{-d}{2}}\det\left(\mbold{H_t}\right)^{\nicefrac{1}{2}}}{\delta}\right)+\frac{\sqrt{\lambda_t}}{2}\right)\\
    &\geq \mbb{P}\left( \forall t\geq 1, \, \left\lVert S_t\right\rVert_{\mbold{H_t^{-1}}(\theta_\star)} \leq \gamma_t(\delta)- \sqrt{\lambda_t}S\right) \\
    &= \mbb{P}\left( \forall t\geq 1, \,  \sqrt{\lambda_t}S+\left\lVert \sum_{t=1}^{s-1}\epsilon_{s+1}x_s\right\rVert_{\mbold{H_t^{-1}}(\theta_\star)} \leq \gamma_t(\delta)\right)  &(\text{def. of $S_t$})\\
    &= \mbb{P}\left( \forall t\geq 1, \,  \left\lVert g_t(\hat\theta_t)-g_t(\theta_\star)\right\rVert_{\mbold{H_t^{-1}}(\theta_\star)} \leq \gamma_t(\delta)\right) = \mbb{P}(E_\delta)
\end{align*}
where the second inequality results from simple upper-bounding and the use of \cref{lemma:determinant_trace_inequality}. 
\end{proof}

\subsection{Convex Relaxation}
\label{subsec:convexrelaxation}
Recall the definition:
\begin{align}
	\mcal{E}_t(\delta) = \left\{\theta\in\Theta \, \middle\vert \, \mcal{L}_t(\theta)-\mcal{L}_t(\hat\theta_t) \leq \beta_t(\delta)^2\right\} \quad \text{ where }\beta_t(\delta) = \gamma_t(\delta) + \gamma^2_t(\delta)/\sqrt{\lambda_t}\; .\label{eq:defbt}
\end{align}
We recall and prove \cref{lemma:ouralgorelaxed} (we provide here a more detailed version than in the main manuscript). 
\lemmaouralgorelaxed*

\begin{proof}
We start by proving that $\mcal{C}_t(\delta)\subseteq \mcal{E}_t(\delta)$. First, we claim \cref{lemma:secoundboundgt}, which proof is deferred to \cref{sec:secoundboundgt}.
\begin{restatable}{lemma}{lemmasecondboundgt}
\label{lemma:secoundboundgt}
Let $\delta\in(0,1]$. For all $\theta\in\mcal{C}_t(\delta)$:
    \begin{align*}
    \left \lVert g_t(\theta) - g_t(\hat\theta_t)\right\lVert_{\mbold{G_t^{-1}}(\theta,\hat\theta_t)} \leq \frac{\gamma_t^2(\delta)}{\sqrt{\lambda_t}} + \gamma_t(\delta)\; .
\end{align*}
\end{restatable}
Thanks to exact second-order Taylor expansion of the logistic loss, we have that for all $\theta\in\mbb{R}^d$:
\begin{align*}
    \mcal{L}_t(\theta) = \mcal{L}_t(\hat\theta_t) + \nabla\mcal{L}_t(\hat\theta_t)^\transp(\theta-\hat\theta_t)+ (\theta-\hat\theta_t)^\transp\left(\int_{v=0}^1 (1-v)\nabla^2\mcal{L}_t(\hat\theta_t + v(\theta-\hat\theta_t))dv\right)(\theta_\star-\hat\theta_t)\, .
\end{align*}
By definition of $\hat\theta_t$ we have that $\nabla\mcal{L}_t(\hat\theta_t) = 0$ and therefore:
\begin{align*}
    \mcal{L}_t(\theta) &= \mcal{L}_t(\hat\theta_t) + (\theta-\hat\theta_t)^\transp\left(\int_{v=0}^1 (1-v)\nabla^2\mcal{L}_t(\hat\theta_t + v(\theta-\hat\theta_t))dv\right)(\theta_\star-\hat\theta_t) &\\
    &= \mcal{L}_t(\hat\theta_t) + (\theta-\hat\theta_t)^\transp\left(\int_{v=0}^1 (1-v)\mbold{H_t}(\hat\theta_t + v(\theta-\hat\theta_t))dv\right)(\theta_\star-\hat\theta_t) & (\nabla^2\mcal{L}_t=\mbold{H_t})\\
    &= \mcal{L}_t(\hat\theta_t) + \left\lVert \theta-\hat\theta_t \right\rVert_{\mbold{\widetilde{G}_t}(\hat\theta_t,\theta)}^2 & (\text{def. of }\mbold{\widetilde{G}_t}(\hat\theta_t,\theta))\\
    &\leq \mcal{L}_t(\hat\theta_t) + \left\lVert \theta-\hat\theta_t \right\rVert_{\mbold{G_t}(\hat\theta_t,\theta)}^2 & (\mbold{\widetilde{G}_t} \leq \mbold{G_t})\\
    &= \mcal{L}_t(\hat\theta_t) + \left\lVert g_t(\theta)-g_t(\hat\theta_t) \right\rVert_{\mbold{G_t^{-1}}(\hat\theta_t,\theta)}^2 &(\text{Equation~\eqref{eq:mvt}})\, .\\
   &=  \mcal{L}_t(\hat\theta_t) + \left\lVert g_t(\theta)-g_t(\hat\theta_t) \right\rVert_{\mbold{G_t^{-1}}(\hat\theta_t,\theta)}^2 &(\mbold{G_t}(\hat\theta_t,\theta)=\mbold{G_t}(\theta,\hat\theta_t))\, .
\end{align*}
Therefore for any $\theta\in\mcal{C}_t(\delta)$:
\begin{align*}
	\mcal{L}_t(\theta) - \mcal{L}_t(\hat\theta_t) &\leq \left \lVert g_t(\theta) - g_t(\hat\theta_t)\right\lVert_{\mbold{G_t^{-1}}(\theta_\star,\hat\theta_t)}^2 &\\
	&\leq \left(\frac{\gamma_t^2(\delta)}{\sqrt{\lambda_t}} + \gamma_t(\delta)\right)^2 = \beta_t(\delta)^2&(\text{\cref{lemma:secoundboundgt}}) \; .\\
	\end{align*}
proving that $\theta\in\mcal{C}_t(\delta)\Rightarrow \theta\in\mcal{E}_t(\delta)$ and therefore $\mcal{C}_t(\delta)\subseteq \mcal{E}_t(\delta)$.

We now prove the second part of \cref{lemma:ouralgorelaxed}. \underline{We will assume that $E_\delta$ holds}, which happens with probability at least $1-\delta$ (cf. \cref{prop:confset}).  We rely on the following second-order Taylor expansion. For all $\theta\in\mcal{E}_t(\delta)$:
\begin{align*}
    \mcal{L}_t(\theta) &= \mcal{L}_t(\theta_\star) + (\theta-\theta_\star)^\transp\nabla\mcal{L}_t(\theta_\star) + (\theta-\theta_\star)^\transp\left(\int_{v=0}^1 (1-v)\nabla^2\mcal{L}_t(\theta_\star + v(\theta-\theta_\star))dv\right)(\theta-\theta_\star)\\
    &= \mcal{L}_t(\theta_\star) + (\theta-\theta_\star)^\transp\nabla\mcal{L}_t(\theta_\star) +\left\lVert \theta-\theta_\star\right\rVert_{\mathbf{\widetilde{G}_t}(\theta_\star,\theta)}^2
\end{align*}
Therefore:
\begin{align*}
    \mcal{L}_t(\theta) - \mcal{L}_t(\theta_\star) - (\theta-\theta_\star)^\transp \nabla\mcal{L}_t(\theta_\star)&= \left\lVert \theta-\theta_\star\right\rVert_{\mathbf{\widetilde{G}_t}(\theta_\star,\theta)}^2 &\\
    &\geq (2+2S)^{-1}\left\lVert \theta-\theta_\star\right\rVert_{\mathbf{H_t}(\theta_\star)}^2 &(\text{\cref{eq:lowerboundGttilde}})
\end{align*}
which can be rewritten as:
\begin{align*}
    \left\lVert \theta-\theta_\star\right\rVert_{\mathbf{H_t}(\theta_\star)}^2 &\leq (2+2S)\left\vert\mcal{L}_t(\theta) - \mcal{L}_t(\theta_\star)\right\vert  + (2+2S)\left\vert(\theta-\theta_\star)^\transp \nabla\mcal{L}_t(\theta_\star)\right\vert &\\
    &\leq 2(2+2S)\beta_t(\delta)^2 + (2+2S)\left\vert(\theta-\theta_\star)^\transp \nabla\mcal{L}_t(\theta_\star)\right\vert &(\theta,\theta_\star\in\mcal{E}_t(\delta))\\
    &\leq  2(2+2S)\beta_t(\delta)^2 + (2+2S)\left\lVert\theta-\theta_\star\right\rVert_{\mbold{H_t}(\theta_\star)}\left\lVert \nabla\mcal{L}_t(\theta_\star)\right\rVert_{\mbold{H_t^{-1}}(\theta_\star)} &(\text{Cauchy-Schwartz})\\
    &\leq 2(2+2S)\beta_t(\delta)^2 + (2+2S)\gamma_t(\delta) \left\lVert \theta-\theta_\star\right\rVert_{\mathbf{H_t}(\theta_\star)}
\end{align*}
where we last used:
\begin{align*}
    \left\lVert \nabla\mcal{L}_t(\theta_\star)\right\rVert_{\mbold{H_t^{-1}}(\theta_\star)} &= \left\lVert g_t(\theta_*) - \sum_{s=1}^{t-1}r_{s+1}x_s \right\rVert_{\mbold{H_t^{-1}}(\theta_\star)} \\
    &= \left\lVert g_t(\theta_\star)-g_t(\hat\theta_t)\right\rVert_{\mbold{H_t^{-1}}(\theta_\star)}\\
    &\leq \gamma_t(\delta) \; .&(E_\delta\text{ holds})
\end{align*}
To sum-up, we have the following polynomial inequality on  $\left\lVert \theta-\theta_\star\right\rVert_{\mathbf{H_t}(\theta_\star)}$:
\begin{align*}
	      \left\lVert \theta-\theta_\star\right\rVert_{\mathbf{H_t}(\theta_\star)}^2 \leq 2(2+2S)\beta_t(\delta)^2 + (2+2S)\gamma_t(\delta) \left\lVert \theta-\theta_\star\right\rVert_{\mathbf{H_t}(\theta_\star)}\; .
\end{align*}
Solving it (cf. \cref{prop:polynomialineq}) yields:
\begin{align*}
    \left\lVert \theta-\theta_\star\right\rVert_{\mathbf{H_t}(\theta_\star)} \leq (2+2S)\gamma_t(\delta) + 2\sqrt{1+S}\beta_t(\delta)\; .
\end{align*}
Finally, note that when $\lambda_t=d\log(t)$ we obtain the following scalings:
\begin{align*}
	\gamma_t(\delta) &= \mcal{O}(\sqrt{d\log(t)}) \; ,&(\text{\cref{eq:defgamma}})\\
	\beta_t(\delta) &= \gamma_t(\delta) + \gamma^2_t(\delta)/\sqrt{\lambda_t }=  \mcal{O}(\sqrt{d\log(t)}) \; .&(\text{\cref{eq:defbt}})
\end{align*}
and therefore we obtain that $\forall\theta\in\mcal{E}_t(\delta)$:
\begin{align*}
   \left\lVert \theta-\theta_\star\right\rVert_{\mathbf{H_t}(\theta_\star)} = \mcal{O}\left(\sqrt{d\log(t)}\right)\; .
\end{align*}
This holds as soon as $E_\delta$ does, which happens with probability at least $1-\delta$.
\end{proof}

\subsection{Proof of \cref{lemma:secoundboundgt}}
\label{sec:secoundboundgt}
\lemmasecondboundgt*
\begin{proof}
    Note that thanks to \cref{lemma:firstselfconcordance} we have:
    \begin{align*}
        \mbold{G_t}(\theta,\hat\theta_t)&= \sum_{s=1}^{t-1}\alpha(x_s,\theta,\hat\theta_t)x_sx_s^\transp + \lambda_t\mbold{I_d}&\\
        &\geq \sum_{s=1}^{t-1}\left(1+\vert x_s^\transp(\theta-\hat\theta_t)\vert\right)^{-1} \dot{\mu}(x_s^\transp\theta)x_sx_s^\transp + \lambda_t\mbold{I_d}&(\text{ \cref{lemma:firstselfconcordance}})\\
        &\geq \sum_{s=1}^{t-1}\!\!\left(1+\left\lVert x_s\right\rVert_{\mbold{G_t^{-1}}(\theta,\hat\theta_t)}\left\lVert \theta-\hat\theta_t\right\rVert_{\mbold{G_t}(\theta,\hat\theta_t)}\right)^{-1}\!\!\! \dot{\mu}(x_s^\transp\theta)x_sx_s^\transp + \lambda_t\mbold{I_d}& \text{(Cauchy-Schwartz)}\\
        &\geq \left(1+\lambda_t^{-1/2}\left\lVert \theta-\hat\theta_t\right\rVert_{\mbold{G_t}(\theta,\hat\theta_t)}\right)^{-1}\sum_{s=1}^{t-1}\dot{\mu}(x_s^\transp\theta)x_sx_s^\transp + \lambda_t\mbold{I_d}&(\mbold{G_t}(\theta,\hat\theta_t)\geq \lambda_t\mbold{I_d})\\
        & \geq \left(1+\lambda_t^{-1/2}\left\lVert \theta-\hat\theta_t\right\rVert_{\mbold{G_t}(\theta,\hat\theta_t)}\right)^{-1} \left(\sum_{s=1}^{t-1}\dot{\mu}(x_s^\transp\theta)x_sx_s^\transp + \lambda_t\mbold{I_d}\right) &\\
        &= \left(1+\lambda_t^{-1/2}\left\lVert \theta-\hat\theta_t\right\rVert_{\mbold{G_t}(\theta,\hat\theta_t)}\right)^{-1} \mbold{H_t}(\theta) &\\
        &= \left(1+\lambda_t^{-1/2}\left\lVert g_t(\theta)-g_t(\hat\theta_t)\right\rVert_{\mbold{G_t^{-1}}(\theta,\hat\theta_t)}\right)^{-1} \mbold{H_t}(\theta) &(\text{\cref{eq:mvt}})\\
    \end{align*}
    Using this inequality, we therefore obtain that:
    \begin{align*}
        \left\lVert g_t(\theta)-g_t(\hat\theta_t)\right\rVert_{\mbold{G_t^{-1}}(\theta,\hat\theta_t)}^2 &\leq \left(1+\lambda_t^{-1/2}\left\lVert g_t(\theta)-g_t(\hat\theta_t)\right\rVert_{\mbold{G_t^{-1}}(\theta,\hat\theta_t)}\right)\left\lVert g_t(\theta)-g_t(\hat\theta_t)\right\rVert_{\mbold{H_t^{-1}}(\theta)}^2&\\
        &\leq  \lambda^{-1/2}\gamma^2_t(\delta)\left\lVert g_t(\theta)-g_t(\hat\theta_t)\right\rVert_{\mbold{G_t^{-1}}(\theta,\hat\theta_t)} + \gamma^2_t(\delta) &(\theta\in\mcal{C}_t(\delta))
    \end{align*}
Solving this polynomial inequality in $\left\lVert g_t(\theta)-g_t(\hat\theta_t)\right\rVert_{\mbold{G_t^{-1}}(\theta,\hat\theta_t)}$ (cf. \cref{prop:polynomialineq}) yields :
\begin{align*}
    \left\lVert g_t(\theta)-g_t(\hat\theta_t)\right\rVert_{\mbold{G_t^{-1}}(\theta,\hat\theta_t)} &\leq \gamma_t(\delta)^2/\sqrt{\lambda_t} + \gamma_t(\delta)) 
\end{align*}
which proves the announced result. 
\end{proof}


%% file: appendix/app_regret_bounds.tex
\newpage

\section{\uppercase{Regret Upper-Bounds}}

\label{app:regretbound}

\subsection{Proof of \cref{thm:generalregret}}
\label{subsec:thmgeneralregret}
\thmgeneralregret*

\begin{proof}
\underline{In the following, we assume the good event $\{\forall t\geq 1, \theta_\star\in\mcal{C}_t(\delta)\}$  to hold}, which happens with probability at least $1-\delta$ according to \cref{prop:confset}. 

Recall the strategy followed by \ouralgo{}:
\begin{align*}
    (x_t,\theta_t) \in \argmax_{x\in\mcal{X},\theta\in\mcal{C}_t(\delta)} x^\transp\theta\, .
\end{align*}
and therefore under the good event we have $x_\star(\theta_\star)^\transp\theta_\star \leq x_t^\transp\theta_t$. We will need the following result, which proof is postponed to \cref{sec:propbounddevtheta}. 

\begin{restatable}{prop}{propbounddevtheta}
\label{prop:bounddevtheta}
If $\theta_*\in\mcal{C}_t(\delta)$ then for all $\theta\in\mcal{C}_t(\delta)$:
\begin{align*}
\left\Vert \theta-\theta_\star\right\rVert_{\mbold{H_t}(\theta_\star)} \leq 2(1+2S)\gamma_t(\delta)
\end{align*}
\end{restatable}

We start by performing a second Taylor expansion of the regret. 
\begin{align*}
	\regret_{\theta_\star}(T) &= \sum_{t=1}^T \mu\Big(x_\star(\theta_\star)^\transp\theta_\star\Big) - \mu\Big(x_t^\transp\theta_\star\Big)\\
	&=  \sum_{t=1}^T \dot\mu\left(x_t^\transp\theta_\star\right)(x_\star(\theta_\star)-x_t)^\transp\theta_\star + \sum_{t=1}^T  \left[\int_{v=0}^1 (1-v)\ddot{\mu}\left(x_t^\transp\theta_\star + v(x_\star(\theta_\star)-x_t)^\transp\theta_\star\right)dv\right]\left\{(x_\star(\theta_\star)-x_t)^\transp\theta_\star\right\}^2 \\
	&=  \underbrace{\sum_{t=1}^T \dot\mu\left(x_t^\transp\theta_\star\right)(x_\star(\theta_\star)-x_t)^\transp\theta_\star}_{R_1(T)} + \underbrace{\sum_{t=1}^T  \tilde{\vartheta}_t\left\{(x_\star(\theta_\star)-x_t)^\transp\theta_\star\right\}^2}_{R_2(T)}\; .
\end{align*}
where we defined:
\begin{align}
	\tilde{\vartheta}_t = \int_{v=0}^1 (1-v)\ddot{\mu}\left(x_t^\transp\theta_\star + v(x_\star(\theta_\star)-x_t)^\transp\theta_\star\right)dv \; .
\label{eq:defvartheta}
\end{align}
We start by examining $R_1(T)$. We have the following bound:
\begin{align*}
	R_1(T) &= \sum_{t=1}^T \dot\mu\left(x_t^\transp\theta_\star\right)(x_\star(\theta_\star)-x_t)^\transp\theta_\star &\\
	&\leq \sum_{t=1}^T  \dot\mu\left(x_t^\transp\theta_\star\right)x_t^\transp(\theta_t-\theta_\star) &(x_t^\transp\theta_t\geq x_\star(\theta_\star)^\transp\theta_\star\text{ since } E_\delta \text{ holds}) \\
	&= \sum_{t=1}^T  \dot\mu\left(x_t^\transp\theta_\star\right)\left\lVert x_t \right\rVert_{\mbold{H_t^{-1}}(\theta_\star)}\left\lVert \theta_t-\theta_\star\right\rVert_{\mbold{H_t}(\theta_\star)} &(\text{Cauchy-Schwarz})\\
	&\leq 2(1+2S)\sum_{t=1}^T  \gamma_t(\delta)\dot\mu\left(x_t^\transp\theta_\star\right)\left\lVert x_t \right\rVert_{\mbold{H_t^{-1}}(\theta_\star)}&(\text{\cref{prop:bounddevtheta}}, \, E_\delta\text{ holds}) \\
	&\leq 2(1+2S)\bar{\gamma}_T(\delta)\sum_{t=1}^T \dot\mu\left(x_t^\transp\theta_\star\right)\left\lVert x_t \right\rVert_{\mbold{H_t^{-1}}(\theta_\star)}
\end{align*}
where we used the notation $\bar{\gamma}_T(\delta) = \max_{t\in[T]} \gamma_t(\delta)$. 

In the following, we denote $\tilde{x}_t\defeq \sqrt{\dot\mu(x_t^\transp\theta_\star})x_t$ and $\mbold{\widetilde{V}_t} \defeq \sum_{s=1}^{t-1}\tilde{x}_s\tilde{x}_s^\transp+\lambda_t\mbold{I_d} = \mbold{H_t}(\theta_\star)$. 
We have:
\begin{align*}
R_1(T) &\leq 2(1+2S)\bar{\gamma}_T(\delta)\sum_{t=1}^T \dot\mu\left(x_t^\transp\theta_\star\right)\left\lVert x_t \right\rVert_{\mbold{H_t^{-1}}(\theta_\star)} &\\
&\leq 2(1+2S)\bar{\gamma}_T(\delta)\sqrt{\sum_{t=1}^T \dot\mu\left(x_t^\transp\theta_\star\right)}\sqrt{\sum_{t=1}^T \dot\mu\left(x_t^\transp\theta_\star\right)\left\lVert x_t \right\rVert_{\mbold{H_t^{-1}}(\theta_\star)}^2} &(\text{Cauchy-Schwarz})\\
&\leq 2(1+2S)\bar{\gamma}_T(\delta)\sqrt{\sum_{t=1}^T \dot\mu\left(x_t^\transp\theta_\star\right)}\sqrt{\sum_{t=1}^T\left\lVert \tilde{x}_t \right\rVert_{\mbold{\widetilde{V}_t^{-1}}}^2} &\\
&\leq  4(1+2S)\bar{\gamma}_T(\delta)\sqrt{d\log\left(\lambda_T + \frac{T}{16d}\right)} \sqrt{\sum_{t=1}^T \dot\mu\left(x_t^\transp\theta_\star\right)} &(\text{\cref{lemma:ellipticalpotentia}})\\
&\leq C_1 d\log(T) \sqrt{\sum_{t=1}^T \dot\mu\left(x_t^\transp\theta_\star\right)}
\end{align*}
where $C_1$ is a universal (more precisely, independent of $\kappa_\mcal{X}(\theta_\star)$, $d$ and $T$), and where we used that $\bar\gamma_T(\delta) = \mcal{O}(\sqrt{d\log(T)})$ since $\lambda_t = d\log(t)$.

Finally, note that by a first-order Taylor expansion of $\dot\mu$:
\begin{align*}
	\sum_{t=1}^T \dot\mu\left(x_t^\transp\theta_\star\right) &= \sum_{t=1}^T \dot{\mu}\left(x_\star(\theta_\star)^\transp\theta_\star\right)  + \sum_{t=1}^T \left[\int_{v=0}^1 \ddot{\mu}\left(x_\star(\theta_\star)^\transp\theta_\star + v(x_t-x_\star(\theta_\star))^\transp\theta_\star\right)dv\right](x_t-x_\star(\theta_\star))^\transp\theta_\star \\
	&= \frac{T}{\kappa_\star(\theta_\star)} + \sum_{t=1}^T \left[\int_{v=0}^1\ddot{\mu}\left(x_\star(\theta_\star)^\transp\theta_\star + v(x_t-x_\star(\theta_\star))^\transp\theta_\star\right)dv\right](x_t-x_\star(\theta_\star))^\transp\theta_\star &(\text{def. } \kappa_\star) \\
	&\leq \frac{T}{\kappa_\star(\theta_\star)} + \sum_{t=1}^T\left\vert \left[\int_{v=0}^1 \ddot{\mu}\left(x_\star(\theta_\star)^\transp\theta_\star + v(x_t-x_\star(\theta_\star))^\transp\theta_\star\right)dv\right](x_t-x_\star(\theta_\star))^\transp\theta_\star \right\vert \\
	&\leq \frac{T}{\kappa_\star(\theta_\star)} + \sum_{t=1}^T \left[\int_{v=0}^1 \left\vert\ddot{\mu}\left(x_\star(\theta_\star)^\transp\theta_\star + v(x_t-x_\star(\theta_\star))^\transp\theta_\star\right)dv\right\vert\right](x_\star(\theta_\star)-x_t)^\transp\theta_\star & (x_\star(\theta_\star)^\transp\theta_\star\geq x_t^\transp\theta_\star) \\
	&\leq \frac{T}{\kappa_\star(\theta_\star)} + \sum_{t=1}^T \left[\int_{v=0}^1 \dot{\mu}\left(x_\star(\theta_\star)^\transp\theta_\star + v(x_t-x_\star(\theta_\star))^\transp\theta_\star\right)dv\right](x_\star(\theta_\star)-x_t)^\transp\theta_\star & (\vert\ddot\mu\vert \leq \mu) \\
	&\leq \frac{T}{\kappa_\star(\theta_\star)} + \sum_{t=1}^T \alpha(\theta_\star,x_\star(\theta_\star),x_t)(x_\star(\theta_\star)-x_t)^\transp\theta_\star & (\text{def. } \alpha) \\
	&= \frac{T}{\kappa_\star(\theta_\star)} + \sum_{t=1}^T \mu(x_\star(\theta_\star)^\transp\theta_\star) - \mu(x_t^\transp\theta_\star) &(\text{mean-value theorem})\\
	&=  \frac{T}{\kappa_\star(\theta_\star)} + \regret_{\theta_\star}(T)
\end{align*}
Using that $\sqrt{a+b} \leq \sqrt{a} + \sqrt{b}$ for all $a,b\geq 0$ we obtain the following intermediate bound on $R_1(T)$:
\begin{align}
	R_1(T) \leq  C_1d\log(T)\left(\sqrt{\frac{T}{\kappa_\star(\theta_\star)}} + \sqrt{\regret_{\theta_\star}(T)}\right)
\label{eq:R1intermediate}
\end{align}
We now turn our attention to $R_2(T)$. We start with a crude-bound and retrieve \cite{faury2020improved} second-order term. Indeed from $\tilde{\vartheta}_t\leq 1$ we get that:
\begin{align*}
	R_2(T) &\leq \sum_{t=1}^T \left\{(x_\star(\theta_\star)-x_t)^\transp\theta_\star\right\}^2 &\\
	&\leq \sum_{t=1}^T \left\{x_t^\transp(\theta_t-\theta_\star)\right\}^2 &(x_t^\transp\theta_t \geq x_\star(\theta_\star)^\transp\theta_\star \text{ since } E_\delta \text{ holds}) \\
	&\leq \sum_{t=1}^T \left\lVert x_t \right\rVert_{\mbold{H_t^{-1}(\theta_\star)}}^2 \left\lVert \theta_t - \theta_\star\right\rVert_{\mbold{H_t(\theta_\star)}}^2 &(\text{Cauchy-Schwarz}) \\
	&\leq 4(1+2S)^2\bar{\gamma}_T(\delta)^2\sum_{t=1}^T \left\lVert x_t \right\rVert_{\mbold{H_t^{-1}(\theta_\star)}}^2 &(\text{\cref{prop:bounddevtheta}, }E_\delta\text{ holds}) \\
	&\leq 4(1+2S)^2\bar{\gamma}_T(\delta)^2 \kappa_\mcal{X}(\theta_\star) \sum_{t=1}^T \left\lVert x_t \right\rVert_{\mbold{V_t^{-1}}}^2 &(\text{\cref{eq:upperboundVt}})\\
	&\leq 16d(1+2S)^2\bar{\gamma}_T(\delta)^2\kappa_\mcal{X}(\theta_\star)\log\left(\lambda_T + \frac{T}{d}\right)& (\text{\cref{lemma:ellipticalpotentia}})
\end{align*}
Introducing $C_2$ another universal constant (independent of $d$, $T$ and $\kappa_\mcal{X}(\theta_\star)$); 
\begin{align}
	R_2(T) \leq C_2 d^2 \kappa_\mcal{X}(\theta_\star)\log^2(T)
\label{eq:R2first}
\end{align}
We now refine this bound to take into account \bad{} arms. The following always holds:
\begin{align}
R_2(T) = \sum_{t=1}^T \tilde{\vartheta}_t\left\{(x_\star(\theta_\star)-x_t)^\transp\theta_\star\right\}^2\mathds{1}\left(x_t\in\mcal{X}_-\right) + \sum_{t=1}^T \tilde{\vartheta}_t\left\{(x_\star(\theta_\star)-x_t)^\transp\theta_\star\right\}^2\mathds{1}\left(x_t\in\mcal{X}_+\right)\; ,
\label{eq:R2decomposition}
\end{align}
with $\mcal{X}_+=\mcal{X}\setminus \mcal{X}_-$. We start by bounding the most-left term in the above inequality. Note that by self-concordance ($\vert\ddot{\mu}\vert\leq \dot\mu$) of the logistic function we have $\tilde\vartheta_t\leq \alpha(\theta_\star,x_\star(\theta_\star),x_t)$ and therefore:
\begin{align}
 \sum_{t=1}^T \tilde{\vartheta}_t\left\{(x_\star(\theta_\star)-x_t)^\transp\theta_\star\right\}^2\mathds{1}\left(x_t\in\mcal{X}_-\right)  &\leq  \sum_{t=1}^T \alpha(\theta_\star,x_\star(\theta_\star),x_t)\left\{(x_\star(\theta_\star)-x_t)^\transp\theta_\star\right\}^2\mathds{1}\left(x_t\in\mcal{X}_-\right)\notag\\
 &\leq S\sum_{t=1}^T\alpha(\theta_\star,x_\star(\theta_\star),x_t)\left\{(x_\star(\theta_\star)-x_t)^\transp\theta_\star\right\}\mathds{1}\left(x_t\in\mcal{X}_-\right) \notag\\
 &= S\sum_{t=1}^T\left[\mu(x_\star(\theta_\star)^\transp\theta_\star)-\mu(x_t^\transp\theta_\star)\right]\mathds{1}\left(x_t\in\mcal{X}_-\right)\notag\\
 &\leq S\mu(x_\star(\theta_\star)^\transp\theta_\star)\sum_{t=1}^T\mathds{1}\left(x_t\in\mcal{X}_-\right)\label{eq:R2second1}
\end{align}
where we used $\ltwo{\theta_\star}\leq S$ and $\ltwo{x}\leq 1$ (for any $x\in\mcal{X}$) in the second-inequality and the mean-value theorem for the equality which follows. 
We now turn to bounding the most-right term in \cref{eq:R2decomposition}. \underline{We start with the case $x_\star(\theta_\star)^\transp\theta_\star\geq 0$}. We therefore look at the following definition for the \bad{} arms:
\begin{align*}
\mcal{X}_- = \left\{x\in\mcal{X}\, \middle\vert x^\transp\theta_* \leq -1 \right\}
\end{align*}
Fix $t$ and assume that $x_t\in\mcal{X}_+$. Note that when $x_t^\transp\theta_\star\geq 0$ we inherit $\tilde\vartheta_t\leq 0$ from the fact that $\ddot{\mu}(z)\leq 0$ for all $z\geq 0$. Using this fact ($\ddot{\mu}\leq 0$ on $\mbb{R}^+$) we can show that when $x_t^\transp\theta_\star\leq 0$:
\begin{align*}
	\tilde\vartheta_t &\leq \int_{v=0}^1 (1-v) \ddot{\mu}\left((1-v)x_t^\transp\theta_\star\right)dv &\\
	&\leq  \int_{v=0}^1 \dot\mu\left((1-v)x_t^\transp\theta_\star\right)dv &(\ddot{\mu}\leq \vert\ddot \mu\vert \leq \dot\mu)\\
	&\leq \dot\mu(x_t^\transp\theta_\star) \int_{v=0}^1 \exp\left(v\vert x_t^\transp\theta_\star\vert\right) &(\text{\cref{lemma:fthirdselfconcordance}})\\
	&\leq e^1 \dot\mu(x_t^\transp\theta_\star) &(-1\leq x_t^\transp\theta_\star\leq 0)
\end{align*}
where in the last inequality we used $x_t^\transp\theta_\star\geq -1$ since $x_t\in\mcal{X}_+$. Packing this results together we showed that:
\begin{align*}
	\tilde\vartheta_t\mathds{1}(x_t\in\mcal{X}_+) &\leq e^1 \dot{\mu}(x_t^\transp\theta_*)\mathds{1}(x_t\in\mcal{X}_+,x_t^\transp\theta_\star\leq 0) + 0\cdot\mathds{1}(x_t\in\mcal{X}_+,x_t^\transp\theta_\star\geq 0)\\
	&\leq  e^1 \dot{\mu}(x_t^\transp\theta_*)\mathds{1}(x_t\in\mcal{X}_+)\\
	&\leq  e^1 \dot{\mu}(x_t^\transp\theta_*)
\end{align*}
Therefore we obtain:
\begin{align*}
	\sum_{t=1}^T \tilde{\vartheta}_t\left\{(x_\star(\theta_\star)-x_t)^\transp\theta_\star\right\}^2\mathds{1}\left(x_t\in\mcal{X}_+\right) &\leq e^1 \sum_{t=1}^T \dot{\mu}(x_t^\transp\theta_\star)\left\{(x_\star(\theta_\star)-x_t)^\transp\theta_\star\right\}^2\\
	&\leq e^1  \sum_{t=1}^T \dot{\mu}(x_t^\transp\theta_\star)\left\{x_t^\transp(\theta_t-\theta_\star)\right\}^2 &(\text{optimism})\\ 
	&\leq 4e^1 (1+2S)^2 \bar{\gamma}^2_T(\delta)\sum_{t=1}^T \dot{\mu}(x_t^\transp\theta_\star)\left\lVert x_t\right\rVert_{\mbold{H_t^{-1}(\theta_\star)}}^2 &\\
\end{align*}
Using \cref{lemma:ellipticalpotentia} with $\tilde{x}_t=\sqrt{\dot\mu(x_t^\transp\theta_\star})x_t$ and $\mbold{\widetilde{V}_t} \defeq \sum_{s=1}^{t-1}\tilde{x}_s\tilde{x}_s^\transp+\lambda_t\mbold{I_d} = \mbold{H_t}(\theta_\star)$ finally yields:
\begin{align}
\sum_{t=1}^T \tilde{\vartheta}_t\left\{(x_\star(\theta_\star)-x_t)^\transp\theta_\star\right\}^2\mathds{1}\left(x_t\in\mcal{X}_+\right)  \leq 16d e^1(1+2S)^2 \bar{\gamma}^2_T \log\left(\lambda_T + \frac{T}{16d}\right)
\label{eq:R2second2}
\end{align}
 \underline{We now consider the case $x_\star(\theta_\star)^\transp\theta_\star\leq 0$}. The definition of $\mcal{X}_-$ becomes:
 \begin{align*}
 	\mcal{X}_- = \left\{ x\, \middle\vert \, \dot\mu(x^\transp\theta_*) \leq (2\kappa_\star(\theta_\star))^{-1}\right\} = \left\{ x\, \middle\vert \, \dot\mu(x^\transp\theta_*) \leq \dot\mu(x_\star(\theta_\star)^\transp\theta_\star)/2\right\}\; .
 \end{align*}
 Fix $t$ and assume that $x_t\in\mcal{X}_+$. Thanks to $\vert\ddot\mu\vert\leq \dot\mu$:
 \begin{align*}
 	\tilde\vartheta_t &\leq \alpha(\theta_\star,x_\star(\theta_\star),x_t) &\\
	&\leq \dot\mu(x_\star(\theta_\star)^\transp\theta_*) &(x_t^\transp\theta_\star\leq x_\star(\theta_\star)^\transp\theta_\star\leq 0 \text{ and } \dot\mu \text{ increasing on }\mbb{R}^-)\\
	&\leq 2\dot\mu(x_t^\transp\theta_\star) &(x\in\mcal{X}_+)
 \end{align*}
Therefore we obtain:
\begin{align*}
	\sum_{t=1}^T \tilde{\vartheta}_t\left\{(x_\star(\theta_\star)-x_t)^\transp\theta_\star\right\}^2\mathds{1}\left(x_t\in\mcal{X}_+\right) &\leq 2 \sum_{t=1}^T \dot{\mu}(x_t^\transp\theta_\star)\left\{(x_\star(\theta_\star)-x_t)^\transp\theta_\star\right\}^2\\
	&\leq 2 \sum_{t=1}^T \dot{\mu}(x_t^\transp\theta_\star)\left\{x_t^\transp(\theta_\star-\theta_t)\right\}^2 &(\text{optimism})\\
	&\leq 2\sum_{t=1}^T \dot{\mu}(x_t^\transp\theta_\star)\left\lVert x_t\right\rVert_{\mbold{H_t^{-1}(\theta_\star)}}^2 \left\lVert \theta_t-\theta_\star\right\rVert_{\mbold{H_t(\theta_\star)}}^2  &(\text{Cauchy-Schwarz})\\
	&\leq   8(1+2S)^2 \bar{\gamma}^2_T(\delta)\sum_{t=1}^T \dot{\mu}(x_t^\transp\theta_\star)\left\lVert x_t\right\rVert_{\mbold{H_t^{-1}(\theta_\star)}}^2  &(\text{\cref{prop:bounddevtheta}})
\end{align*}
Using \cref{lemma:ellipticalpotentia} again yields:
\begin{align}
\sum_{t=1}^T \tilde{\vartheta}_t\left\{(x_\star(\theta_\star)-x_t)^\transp\theta_\star\right\}^2\mathds{1}\left(x_t\in\mcal{X}_+\right)  \leq 32d (1+2S)^2 \bar{\gamma}^2_T \log\left(\lambda_T + \frac{T}{16d}\right)
\label{eq:R2second3}
\end{align}
Assembling \cref{eq:R2decomposition}-\eqref{eq:R2second1}-\eqref{eq:R2second2}-\eqref{eq:R2second3} we obtain that:
\begin{align*}
	R_2(T) \leq C_3d^2\log^2(T) + C_4\mu(x_\star(\theta_\star)^\transp\theta_\star)\sum_{t=1}^T\mathds{1}(x_t\in\mcal{X}_-)
\end{align*}
where $C_3$ and $C_4$ constants independent of $d$, $T$ and $\kappa_\mcal{X}$. Merging this result with \cref{eq:R2first} finally yields:
\begin{align}
R_2(T) \leq \Big[C_2 d^2 \kappa_\mcal{X}(\theta_\star)\log^2(T)\Big]\wedge \Big[C_3d^2\log^2(T) + C_4\mu(x_\star(\theta_\star)^\transp\theta_\star)\sum_{t=1}^T\mathds{1}(x_t\in\mcal{X}_-)\Big]
\label{eq:R2final}
\end{align}
We are now ready to finish the proof of \cref{thm:generalregret}. From the decomposition $\regret_{\theta_\star}(T) = R_1(T) + R_2(T)$ and \cref{eq:R1intermediate} we have:
\begin{align*}
	\regret_{\theta_\star}(T) \leq C_1 d\log(T)\sqrt{\frac{T}{\kappa_\star(\theta_\star)}} + C_1d\log(T)\sqrt{\regret_{\theta_\star}(T)} + R_2(T)
\end{align*}
This is a second-order polynomial inequation in $\sqrt{\regret_{\theta_\star}(T)}$. Solving it (cf. \cref{prop:polynomialineq}) yields:
\begin{align*}
	\sqrt{\regret_{\theta_\star}(T)} &\leq C_1 d\log(T) + \sqrt{C_1d\log(T)\sqrt{\frac{T}{\kappa_\star(\theta_\star)} }+ R_2(T)}
\end{align*}
Using $(a+b)\leq 2(a^2+b^2)$ we obtain:
\begin{align*}
	\regret_{\theta_\star}(T) &\leq 2C_1^2 d^2\log^2(T) + 2C_1d\log(T)\sqrt{\frac{T}{\kappa_\star(\theta_\star)} }+ 2R_2(T)
\end{align*}
We obtain the announced inequality after plugging \cref{eq:R2final} in this last inequality. Indeed, \underline{ignoring universal constants} we obtain:
\begin{align*}
\regret_{\theta_\star}(T) \leq d\log(T)\sqrt{\frac{T}{\kappa_\star(\theta_\star)}} + d^2\log^2(T) + \Big[d^2 \kappa_\mcal{X}(\theta_\star)\log^2(T)\Big]\wedge \Big[d^2\log^2(T) + \mu(x_\star(\theta_\star)^\transp\theta_\star)\sum_{t=1}^T\mathds{1}(x_t\in\mcal{X}_-)\Big]
\end{align*}
Slightly re-arranging:
\begin{align*}
\regret_{\theta_\star}(T) \leq \underbrace{d\log(T)\sqrt{\frac{T}{\kappa_\star(\theta_\star)}}}_{\regretplus_{\theta_\star}(T)} + \underbrace{\Big[d^2 (\kappa_\mcal{X}(\theta_\star)+1)\log^2(T)\Big]\wedge \Big[2d^2\log^2(T) + \mu(x_\star(\theta_\star)^\transp\theta_\star)\sum_{t=1}^T\mathds{1}(x_t\in\mcal{X}_-)\Big]}_{\regretmoins_{\theta_\star}(T)}
\end{align*}
which finishes the proof. 
\end{proof}

\subsection{Proof of \cref{prop:bounddevtheta}}
\label{sec:propbounddevtheta}
\propbounddevtheta*
\begin{proof} Let $\theta\in\mcal{C}_t(\delta)$.
\begin{align*}
    \left\lVert \theta-\theta_\star \right\rVert_{\mbold{H_t}(\theta_\star)}  &\leq \sqrt{1+2S}  \left\lVert \theta-\theta_\star \right\rVert_{\mbold{G_t}(\theta_\star,\theta)} & (\theta_\star,\theta\in\Theta, \text{ Equation~\eqref{eq:lowerboundGt}}) \\
    &= \sqrt{1+2S}  \left\lVert g_t(\theta)-g_t(\theta_\star) \right\rVert_{\mbold{G_t^{-1}}(\theta_\star,\theta)} &(\text{Equation~\eqref{eq:mvt}}) \\
    &\leq  \sqrt{1+2S} \left(\left\lVert g_t(\theta)-g_t(\hat\theta_t) \right\rVert_{\mbold{G_t^{-1}}(\theta_\star,\theta)}+\left\lVert g_t(\theta_\star)-g_t(\hat\theta_t) \right\rVert_{\mbold{G_t^{-1}}(\theta_\star,\theta)}\right) \\
    &\leq (1+2S)\left(\left\lVert g_t(\theta)-g_t(\hat\theta_t) \right\rVert_{\mbold{H_t^{-1}}(\theta)}+\left\lVert g_t(\theta_\star)-g_t(\hat\theta_t) \right\rVert_{\mbold{H_t^{-1}}(\theta_\star)}\right) &(\theta_\star,\theta\in\Theta, \text{ Equation~\eqref{eq:lowerboundGt}})\\
    &\leq 2(1+2S)\gamma_t(\delta) &(\theta,\theta_\star\in\mcal{C}_t(\delta))
\end{align*}
which proves the announced result. 
\end{proof}

\subsection{Proof of \cref{prop:lengthtransitory}}
\label{sub:proplengthtransitory}
\proplengthtransitory* 

\subsubsection{Proof of \cref{eq:rmoinsK}}

\begin{proof}
We assume the event $E_\delta=\{\forall t\geq 1, \theta_\star\in\mcal{C}_t(\delta)\}$ holds - this happens with high probability (cf. \cref{prop:confset}). To bound $\regretmoins_{\theta_\star}(T)$ we will start from the bound given in the detailed version of \cref{thm:generalregret} in \cref{subsec:thmgeneralregret}, that is with $C_1$ and $C_2$ being universal constants:

\begin{align}
\regretmoins_{\theta_\star}(T) \leq C_1d^2\log^2(T) + C_2\mu(x_\star(\theta_\star)^\transp\theta_\star) \sum_{t=1}^T \mathds{1}\left(x_t\in\mcal{X}_-\right)
\label{eq:boundrmoins}
\end{align}

Assume that there is a finite number of \bad{} arms, \emph{i.e} $\vert\mcal{X}_-\vert = K<\infty$. We will separate three cases 1. $x_\star(\theta_\star)^\transp\theta_\star\geq 0$ and 2. $x_\star(\theta_\star)^\transp\theta_\star\leq -1$. and 3. $x_\star(\theta_\star)^\transp\theta_\star\in[-1,0]$. Note that 2. and 3. are sub-cases of the more general  $x_\star(\theta_\star)^\transp\theta_\star\leq 0$. We separate them here to simplify the analysis. 
\begin{enumerate}[leftmargin=0.cm,itemindent=.5cm,labelwidth=\itemindent,labelsep=0cm,align=left, itemsep=-5pt]
\item[\underline{Case 1}]. \underline{$x_\star(\theta_\star)^\transp\theta_\star\geq 0$}. In this setting we have:
\begin{align*}
	\mcal{X}_- = \left\{x\in\mcal{X} \middle\vert x^\transp\theta_\star\leq -1 \right\}
\end{align*}
This implies that \bad{} arms have a large (constant) gap. Indeed for any $x\in\mcal{X}_-$:
\begin{align*}
	\mu(x_\star(\theta_\star)^\transp\theta_\star) - \mu(x^\transp\theta_\star) &\geq \mu(x_\star(\theta_\star)^\transp\theta_\star) - \mu(-1)\\
	&\geq 1/2-\mu(-1)
\end{align*}
which yields that:
\begin{align}
\mu(x_\star(\theta_\star)^\transp\theta_\star) - \mu(x^\transp\theta_\star) \geq 1/5
\label{eq:largegap}
\end{align}
We can use this result to show that \ouralgo{} plays \bad{} arms only logarithmically often. Indeed, for any $x\in\mcal{X}_-$  let $\tau_x$ be the last time-step when $x$ is played, and $N_x$ the number of time $x$ was played over the whole horizon. Formally:
\begin{align*}
	\tau_x = \max_{t} \left\{t\in[T]\, \middle\vert \,x_t=x\right\} \quad \text{ and } \quad N_x = \sum_{t=1}^T \mathds{1}(x_t=x)= \sum_{t=1}^{\tau_x} \mathds{1}(x_t=x)\; .
\end{align*}
Fix  $x\in\mcal{X}_-$ and let $\tau=\tau_x$ (\emph{i.e} $x_\tau=x$). Thanks to \cref{eq:largegap} and the mean-value theorem:
\begin{align}
	1/5 &\leq \mu\left(x_\star(\theta_\star)^\transp\theta_\star\right) - \mu\left(x_\tau^\transp\theta_\star\right) \notag\\
	&\leq \mu\left(x_\tau^\transp\theta_\tau\right) - \mu\left(x_\tau^\transp\theta_\star\right) &(\text{optimism}, E_\delta\text{ holds})\notag\\
	      &\leq \alpha(x_\tau,\theta_\tau,\theta_\star)x_\tau^\transp(\theta_\tau-\theta_\star)\notag&(\text{mean-value theorem})\\
	      &= \alpha(x_\tau,\theta_\tau,\theta_\star)x_\tau^\transp\mbold{G_\tau^{-1}}(\theta_\tau,\theta_\star)\left(g_\tau(\theta_\tau)-g_\tau(\theta_\star)\right) &(\text{\cref{eq:mvt}})\notag\\
	      &\leq  \alpha(x_\tau,\theta_\tau,\theta_\star)\left\lVert x_\tau\right\rVert_{\mbold{G_\tau^{-1}}(\theta_\tau,\theta_\star)}\left\lVert g_\tau(\theta_\tau)-g_\tau(\theta_\star)\right\rVert_{\mbold{G_\tau^{-1}}(\theta_\tau,\theta_\star)} &(\text{Cauchy-Schwarz})\notag\\
	      &\leq 2\sqrt{1+2S}\gamma_\tau(\delta) \alpha(x_\tau,\theta_\tau,\theta_\star)\left\lVert x_\tau\right\rVert_{\mbold{G_\tau^{-1}}(\theta_\tau,\theta_\star)}&\label{eq:gapbound}
\end{align}
where we last used  $\left\lVert g_t(\theta_t)-g_t(\theta_\star)\right\rVert_{\mbold{G_t^{-1}}(\theta_t,\theta_\star)}\leq 2\sqrt{1+2S}\gamma_t(\delta)$ (cf. proof of \cref{prop:bounddevtheta}). Note also that $\mbold{G_{\tau}}(\theta_\tau,\theta_\star) \succeq N_x \alpha(x,\theta_{\tau},\theta_\star)xx^\transp+ \lambda_\tau\mbold{I_d}$. 
It is therefore easy to show (for instance, using the Sherman-Morison formula) that $\left\lVert x_\tau\right\rVert_{\mbold{G_\tau^{-1}}(\theta_\tau,\theta_\star)}^2 \leq (\alpha(x_\tau,\theta_\tau,\theta_*)N_x)^{-1}$. We therefore finally obtain by injecting this into \cref{eq:gapbound}:
\begin{align*}
	N_x &\leq 100(1+2S)\gamma_\tau(\delta)^2\alpha(x_\tau,\theta_\tau,\theta_\star)\\
	&\leq 25(1+2S)\gamma_\tau(\delta)^2 &(\alpha \leq \sup \dot\mu \leq 1/4)
\end{align*}
Remember that this results holds for \emph{any} $x\in\mcal{X}_-$. Henceforth from \cref{eq:boundrmoins}:
\begin{align*}
	\regretmoins_{\theta_\star}(T) &\leq C_1d^2\log^2(T) + C_2\mu(x_\star(\theta_\star)^\transp\theta_\star) \sum_{t=1}^T \mathds{1}\left(x_t\in\mcal{X}_-\right) \\
	&\leq C_1d^2\log^2(T) +  C_2\sum_{t=1}^T \mathds{1}\left(x_t\in\mcal{X}_-\right) &(\mu\leq 1)\\
	&= C_1d^2\log^2(T) + C_2\sum_{t=1}^T \sum_{x\in\mcal{X}_-} \mathds{1}(x_t=x)\\
	&=C_1d^2\log^2(T) +  C_2\sum_{x\in\mcal{X}_-} N_x\\
	&\leq C_1d^2\log^2(T) +  25C_2(1+2S)\sum_{x\in\mcal{X}_-} \gamma_{\tau_x}(\delta)^2\\
	&\leq C_1d^2\log^2(T) +  25C_2(1+2S)K\max_{t\in[T]}\gamma_{t}(\delta)^2 &(\vert \mcal{X}_-\vert =K)
\end{align*}
Using the fact that $\max_{t\in[T]}\gamma_t(\delta) \lesssim_T \sqrt{d\log(T)}$ we obtain the announced result:
\begin{align*}
	\regretmoins_{\theta_\star}(T) \lesssim_T d^2 + dK
\end{align*}
\item[\underline{Case 2}]. \underline{$x_\star(\theta_\star)^\transp\theta_\star< -1$}. This necessarily implies $x^\transp\theta_\star\leq -1$  and $\mu(x^\transp\theta_\star)\leq \mu(x_\star(\theta_\star)^\transp\theta_\star)\leq\mu(-1)\leq1/2$ for any $x\in\mcal{X}$. 
We start by characterizing the gap of detrimental arms which are now defined by:
\begin{align*}
	\mcal{X}_- =  \left\{ x\in\mcal{X}\, \middle\vert \, \dot\mu(x^\transp\theta_\star) \leq \dot\mu(x_\star(\theta_\star)^\transp\theta_\star)/2\right\}
\end{align*}
From $\dot\mu=\mu(1-\mu)$ we get that for any $x\in\mcal{X}_-$:
\begin{align*}
	\mu(x^\transp\theta_\star) &\leq \frac{\mu(x_\star(\theta_\star)^\transp\theta_\star)}{2}\frac{1-\mu(x_\star(\theta_\star)^\transp\theta_\star)}{1-\mu(x^\transp\theta_\star)}\\
	&\leq \mu(x_\star(\theta_\star)^\transp\theta_\star)/2 &(\mu(x^\transp\theta_\star)\leq \mu(x_\star(\theta_\star)^\transp\theta_\star) )
\end{align*}
and therefore for any $\mcal{X}_-$:
\begin{align}
	\mu(x_\star(\theta_\star)^\transp\theta_\star) -\mu(x^\transp\theta_\star) \geq 	\mu(x_\star(\theta_\star)^\transp\theta_\star) /2
	\label{eq:neggapbound}
\end{align}
Note the difference with case 1. since here the gap is no longer lower-bounded by a constant (\emph{i.e} it is problem-dependent). Fix  $x\in\mcal{X}_-$ and let $\tau=\tau_x$ (\emph{i.e} $x_\tau=x$).
Using the mean-value theorem we obtain:
\begin{align}
	\mu(x_\star(\theta_\star)^\transp\theta_\star) /2 &\leq \mu\left(x_\star(\theta_\star)^\transp\theta_\star\right) - \mu\left(x_\tau^\transp\theta_\star\right) \notag\\
	      &\leq \alpha(\theta_\star, x_\star(\theta_\star),x_\tau)\theta_\star^\transp(x_\star(\theta_\star)-x_\tau)\notag\\
	      &\leq  \alpha(\theta_\star, x_\star(\theta_\star),x_\tau)x_\tau^\transp(\theta_\tau-\theta_\star) &(\text{optimism}) \notag\\
	      &\leq  \alpha(\theta_\star, x_\star(\theta_\star),x_\tau)x_\tau^\transp\mbold{G_\tau^{-1}}(\theta_\tau,\theta_\star)(g_\tau(\theta_\tau)-g_\tau(\theta_\star)) &(\text{\cref{eq:mvt}}) \notag\\
	      &\leq   \alpha(\theta_\star, x_\star(\theta_\star),x_\tau)\left\lVert x_\tau\right\rVert_{\mbold{G_\tau^{-1}}(\theta_\tau,\theta_\star)}\left\lVert g_\tau(\theta_\tau)-g_\tau(\theta_\star)\right\rVert_{\mbold{G_\tau^{-1}}(\theta_\tau,\theta_\star)} &(\text{Cauchy-Schwarz})\notag\\
	      &\leq 2\sqrt{1+2S}\gamma_\tau(\delta) \alpha(\theta_\star, x_\star(\theta_\star),x_\tau)\left\lVert x_\tau\right\rVert_{\mbold{G_\tau^{-1}}(\theta_\tau,\theta_\star)}& \notag \\
	      &\leq  2\sqrt{1+2S}\gamma_\tau(\delta) \dot{\mu}(x_\star(\theta_\star)^\transp\theta_\star)\left\lVert x_\tau\right\rVert_{\mbold{G_\tau^{-1}}(\theta_\tau,\theta_\star)}\label{eq:boundgap2}&
\end{align}
where we used $\lVert g_\tau(\theta_t)-g_\tau(\theta_\star)\rVert_{\mbold{G_\tau^{-1}}(\theta_\tau,\theta_\star)}\leq 2\sqrt{1+2S}\gamma_\tau(\delta)$ (cf. proof of \cref{prop:bounddevtheta}) and the fact that $\dot\mu$ is increasing on $[x_\tau^\transp\theta_\star,x_\star(\theta_\star)^\transp\theta_\star]$ which yields $\alpha(\theta_\star, x_\star(\theta_\star),x_\tau)\leq \dot\mu(x_\star(\theta_\star)^\transp\theta_\star)$. 
We now need to separate two cases:
\begin{enumerate}[leftmargin=1cm,itemindent=.5cm,labelwidth=\itemindent,labelsep=0.1cm,align=left, itemsep=-5pt]
	\item[2.1.]  \underline{$x^\transp\theta_\tau\leq 0$}. Thanks to optimism (\emph{i.e} $x^\transp\theta_\tau\geq x_\star(\theta_\star)^\transp\theta_\star$) and the monotonicity (increasing) of $\dot\mu$ in $\mbb{R}^-$ we obtain that $\dot{\mu}(x_\star(\theta_\star)^\transp\theta_\star)\leq \dot\mu(x^\transp\theta_\tau)$. Further:
	\begin{align}
		\left\lVert x_\tau\right\rVert_{\mbold{G_\tau^{-1}}(\theta_\tau,\theta_\star)} &\leq \sqrt{1+2S}\left\lVert x_\tau\right\rVert_{\mbold{H_\tau^{-1}}(\theta_\tau)}  &(\text{\cref{eq:lowerboundGt}})\notag\\
		&\leq \sqrt{1+2S}(N_x \dot\mu(x^\transp\theta_\tau))^{-1/2} &(\text{Sherman-Morison})\notag\\
		&\leq \sqrt{1+2S}(N_x \dot\mu(x_\star(\theta_\star)^\transp\theta_\star))^{-1/2} \label{eq:boundgap21}&
	\end{align}
	\item[2.2.]  \underline{$x^\transp\theta_\tau\geq 0$}.
	\begin{align}
		\left\lVert x_\tau\right\rVert_{\mbold{G_\tau^{-1}}(\theta_\tau,\theta_\star)} &\leq (N_x \alpha(x,\theta_\tau,\theta_\star))^{-1/2} &(\text{Sherman-Morison})\notag\\
		&\leq N_x^{-1/2} \left(\frac{x^\transp\theta_\tau-x^\transp\theta_\star}{\mu(x^\transp\theta_\tau) - \mu(x^\transp\theta_*)}\right)^{1/2} &\text{(mean-value theorem})\notag\\
		&\leq N_x^{-1/2}\sqrt{2S}\left(\mu(x^\transp\theta_\tau) - \mu(x^\transp\theta_*)\right)^{-1/2} & (\ltwo{x}\leq 1, \theta_\tau,\theta_\star\in\Theta)\notag\\
		&\leq N_x^{-1/2}\sqrt{2S}\left(1/2 - \mu(x^\transp\theta_*)\right)^{-1/2} & (x^\transp\theta_\tau\geq 0 \Rightarrow \mu(x^\transp\theta_\tau)\geq 1/2)\notag\\
		&\leq N_x^{-1/2}\sqrt{2S}\left(1/2 - \mu(-1)\right)^{-1/2} & (x^\transp\theta_\star\leq 0 \Rightarrow \mu(x^\transp\theta_\star)\leq \mu(-1))\notag\\
		&\leq  5  N_x^{-1/2}\sqrt{2S} \notag\\
		&\leq 5(N_x\dot{\mu}(x_\star(\theta_\star)^\transp\theta_\star))^{-1/2} \sqrt{2S} &(0\leq \dot\mu\leq 1)\label{eq:boundgap22}
	\end{align}
\end{enumerate} 
Therefore combining \cref{eq:boundgap21,eq:boundgap22} we obtain that whichever we are in case 2.1 or 2.2, for any $x\in\mcal{X}_-$:
\begin{align*}
	\left\lVert x_\tau\right\rVert_{\mbold{G_\tau^{-1}}(\theta_\tau,\theta_\star)} \leq C_3\left(N_x\dot{\mu}(x_\star(\theta_\star)^\transp\theta_\star)\right)^{-1/2}\gamma_\tau(\delta)
\end{align*}
where $C_3$ is a constant hiding universal terms and $S$ dependencies. Plugging this result in \cref{eq:boundgap2} and introducing a similar constant $C_4$ we obtain that:
\begin{align*}
	\mu(x_\star(\theta_\star)^\transp\theta_\star) /2 \leq C_4N_x^{-1/2}\left(\dot{\mu}(x_\star(\theta_\star)^\transp\theta_\star)\right)^{1/2}\gamma_\tau(\delta)
\end{align*}
Therefore for any $x\in\mcal{X}_-$:
\begin{align}
	N_x &\leq 4C_4^2\frac{\dot{\mu}(x_\star(\theta_\star)^\transp\theta_\star)}{\mu(x_\star(\theta_\star)^\transp\theta_\star)^2}\gamma_\tau(\delta)^2\notag\\
	&\leq \frac{4C_4^2}{\mu(x_\star(\theta_\star)^\transp\theta_\star)}\gamma_\tau(\delta)^2 &(\dot\mu\leq \mu) \label{eq:boundiamtiredoflabelingstuff}
\end{align}
Henceforth from \cref{eq:boundrmoins}:
\begin{align*}
	\regretmoins_{\theta_\star}(T) &\leq C_1d^2\log^2(T) + C_2\mu(x_\star(\theta_\star)^\transp\theta_\star) \sum_{t=1}^T \mathds{1}\left(x_t\in\mcal{X}_-\right) \\
	&= C_1d^2\log^2(T) + C_2\mu(x_\star(\theta_\star)^\transp\theta_\star) \sum_{t=1}^T \sum_{x\in\mcal{X}_-} \mathds{1}(x_t=x)\\
	&=C_1d^2\log^2(T) +  C_2\mu(x_\star(\theta_\star)^\transp\theta_\star) \sum_{x\in\mcal{X}_-} N_x\\
	&\leq C_1d^2\log^2(T) +  4C_2C_4\sum_{x\in\mcal{X}_-} \gamma_{\tau_x}(\delta)^2 &(\text{\cref{eq:boundiamtiredoflabelingstuff}})\\
	&\leq C_1d^2\log^2(T) +  25C_2K\max_{t\in[T]}\gamma_{t}(\delta)^2 &(\vert \mcal{X}_-\vert =K)
\end{align*}
Using the fact that $\max_{t\in[T]}\gamma_t(\delta) \lesssim_T \sqrt{d\log(T)}$ we obtain the announced result:
\begin{align*}
	\regretmoins_{\theta_\star}(T) \lesssim_T d^2 + dK
\end{align*}

\item[\underline{Case 3}]. \underline{$x_\star(\theta_\star)^\transp\theta_\star\in[-1,0]$}. Recall the definition of $\mcal{X}_-$ in this case:
\begin{align*}
	\mcal{X}_- = \left\{ x\in\mcal{X}\, \middle\vert \, \dot\mu(x^\transp\theta_\star) \leq \dot\mu(x_\star(\theta_\star)^\transp\theta_\star)/2\right\}
\end{align*}
We can directly re-use the characterization of the sub-optimality gap for detrimental arms of \cref{eq:neggapbound}. This yields that for any $x\in\mcal{X}_-$:
\begin{align*}
	\mu(x_\star(\theta_\star)^\transp\theta_\star) - \mu(x^\transp\theta_\star) &\geq \dot{\mu}(x_\star(\theta_\star)^\transp\theta_\star)/2 \\
	&\geq \dot{\mu}(-1)/2 \geq 9/200
\end{align*}
We are therefore in the same configuration as in case 1 (the sub-optimality gap of \bad{} arms is lower-bounded by a non-problem dependent constant). Following the same reasoning yields to the announced claim. This finishes the proof. 
\end{enumerate}

\end{proof}

\subsubsection{Proof of \cref{eq:rmoinsball}}

\begin{proof}
As in the proof of~\cref{eq:rmoinsK}, we work under the assumption that the event $E_\delta=\{\forall t\geq 1, \theta_\star\in\mcal{C}_t(\delta)\}$ holds, which happens with high probability (cf. \cref{prop:confset}). We focus here on the case where $\mathcal{X} = \mathcal{B}_d(0,1)$, which implies that any parameter $\theta$ is co-linear with its associated optimal arm. More precisely: $x_\star(\theta) = \theta / \| \theta\|$ for any $\theta\in\Theta$. Further, this guarantees that $x_\star(\theta)^\transp \theta \geq 0$ for all $\theta \in \Theta$. In particular, $x_\star(\theta_\star)^\transp \theta_\star \geq 0$ and we have the following definition for the detrimental arms:
\begin{equation*}
	\mcal{X}_- = \left\{x\in\mcal{X} \middle\vert x^\transp\theta_\star\leq -1 \right\}.
\end{equation*}

The objective of the proof is to bound the number of time \bad{} arms are played by \ouralgo{} within $T$ rounds. We collect this in the following set:
\begin{equation}
    \mathcal{T} := \{ t \leq T \text{ s.t } x_t \in \mathcal{X}_{-} \},
    \label{eq:def.Tau.detrimental.arms}\; .
\end{equation}
To do so, we start by decomposing the set $\mathcal{T}$ in distinct subsets, each one being of small cardinality. Formally, we construct $\{\mathcal{T}_i\}_{i \geq 1}$ through the following backward induction.

\begin{enumerate}
    \item \textbf{Initialization.} $\mathcal{T}_0 = \emptyset$, $i=0$.
    \item \textbf{Backward induction.} While $\bigcup_{j\geq 1} \mathcal{T}_j \neq \mathcal{T}$, we increment $i$ by $1$, and define
    \begin{equation}
        \begin{aligned}
        \tau_{i} &= \max \left\{ t \in \mathcal{T}, t \notin \bigcup_{j< i} \mathcal{T}_j\right\}, \\
        \mathcal{T}_i &= \left\{ t \leq \tau_i,t \notin \bigcup_{j< i} \mathcal{T}_j\geq 0,\, x_t^\transp \theta_{\tau_i},\, x_t\in\mcal{X}_- \right\}.
        \end{aligned}
        \label{eq:backward.construct.taui}
    \end{equation}
    \end{enumerate}

Such construction immediately implies that $\{\mathcal{T}_i\}_{i\geq 1}$ is a partition of $\mathcal{T}$.

\begin{prop}\label{prop:partition.tau}
Let $\mathcal{T}$ and $\{\mathcal{T}_i\}_{i\geq 1}$ be defined as in~\cref{eq:def.Tau.detrimental.arms} and~\cref{eq:backward.construct.taui}, and let $N$ be the number of subsets $\{\mathcal{T}_i\}$. Then:
\begin{equation*}
    \bigcup_{i=1}^N \mathcal{T}_i = \mathcal{T}; \quad\quad \mathcal{T}_i \cap \mathcal{T}_j = \emptyset, \forall i\neq j; \quad\quad N \leq (d+1).
\end{equation*}
\end{prop}
\begin{proof}[Proof of~\cref{prop:partition.tau}]
    The fact that $\bigcup_{i=1}^N \mathcal{T}_i$ is a partition of $\mathcal{T}$ directly follows from its construction. Thus, we only have to prove that $N \leq (d+1)$. By construction, of the time steps $\tau_i$ for $i=1,\dots,N$, we have that 
    \begin{equation*}
        \forall j > i, \quad x_{\tau_i}^\transp \theta_{\tau_j} < 0,
    \end{equation*}
    and since $\theta_{\tau_i}$ is co-linear with $x_{\tau_i}$, we obtain 
    \begin{equation*}
        \forall j,i \in [N], \quad x_{\tau_i}^\transp
        x_{\tau_j} < 0.
    \end{equation*}
    We conclude by using Lemma.~19 in~\citet{dong2019performance}, which states that it can only exists at least $d+1$ such arms, and hence such time steps. As a result, $N \leq (d+1)$.
\end{proof}
    
From the definition of $\bigcup_{i=1}^N \mathcal{T}_i$ and~\cref{prop:partition.tau}, we have that 
\begin{equation*}
    | \mathcal{T}| = \sum_{i=1}^N |\mathcal{T}_i| \leq (d+1) \max_{i=1,\dots,N} |\mathcal{T}_i|.
\end{equation*}
As a result, we only have to bound $|\mathcal{T}_i|$ for any $i \in [N]$ to conclude the proof.\\

First, notice that $\tau_i$ is the last time step in $\mathcal{T}_i$ and that for all $t \in \mathcal{T}_i$, $x_t^\transp \theta_* \leq -1$ (from the definition of $\mathcal{T}$) while $x_t^\transp \theta_{\tau_i} \geq 0$ (from the construction of the partition). Hence, for all $t\in\mathcal{T}_i$:

\begin{align}
        \mu(0) - \mu(-1) & \leq \mu(x_t^\transp \theta_{\tau_i}) - \mu(x_t^\transp \theta_\star) \notag\\
        &= \alpha(x_t,\theta_{\tau_i},\theta_\star) x_t^\transp (\theta_{\tau_i} - \theta_\star)  &(\text{mean-value theorem}) \notag\\
        &\leq \alpha(x_{t},\theta_{\tau_i},\theta_\star)\left\lVert x_t\right\rVert_{\mbold{G}_{\tau_i}^{-1}(\theta_{\tau_i},\theta_\star)}\left\lVert \theta_{\tau_i}-\theta_\star\right\rVert_{\mbold{G}_{\tau_i}(\theta_{\tau_i},\theta_\star)} &(\text{Cauchy-Schwarz}) \notag\\
        &\leq 2\sqrt{1+2S}\gamma_{\tau_i}(\delta) \alpha(x_t,\theta_{\tau_i},\theta_\star)\left\lVert x_t\right\rVert_{\mbold{G}_{\tau_i}^{-1}(\theta_{\tau_i},\theta_\star)} &(E_\delta \text{ holds, \cref{prop:bounddevtheta}}) \notag\\
        &\leq \frac{\sqrt{1 + 2 S}}{2} \gamma_{\tau_i}(\delta) \left\lVert x_t\right\rVert_{\mbold{G_{\tau_i}^{-1}}(\theta_{\tau_i},\theta_\star)}\; . &(\alpha\leq \sup \dot\mu \leq 1/4)\label{eq:proofpro2bla}
\end{align}

Further, for all $t\in\mathcal{T}_i$, $x_t^\transp \theta_\star \leq -1$ and $x_t^\transp \theta_{\tau_i}\geq 0$ leads to 
\begin{equation*}
\begin{aligned}
\alpha(x_t,\theta_{\tau_i},\theta_\star) &= \frac{\mu(x_t^\transp \theta_{\tau_i}) - \mu(x_t^\transp \theta_\star) } { x_t^\transp (\theta_{\tau_i} - \theta_\star)} \\
& \geq \frac{\mu(0) - \mu(-1)}{2S}\; . &(\ltwo{x}\leq 1, \, \theta_{\tau_i},\theta_\star\in\Theta)
\end{aligned}
\end{equation*}
As a result, let $\mbold{\bar{V}_{\tau_i}} := \sum_{s \in \mathcal{T}_i}  x_s x_s^\transp + \lambda_{\tau_i}\mbold{I_d}$, one obtains,
\begin{equation*}
    \mbold{G_{\tau_i}}(\theta_{\tau_i},\theta_\star)  \succcurlyeq \sum_{s \in \mathcal{T}_i} \alpha(x_s,\theta_{\tau_i},\theta_\star) x_s x_s^\transp + \lambda_{\tau_i}\mbold{I_d} \succeq \frac{\mu(0) - \mu(-1)}{2 S} \mbold{\bar{V}_{\tau_i}},
\end{equation*}
which combined with \cref{eq:proofpro2bla} leads to:
\begin{equation}
\label{eq:proof.prop2.ball.temp1}
(\mu(0) - \mu(-1))^{3/2} \leq \sqrt{S/2} \sqrt{1+2S} \gamma_{\tau_i}(\delta) \|x_t\|_{\mbold{\bar{V}^{-1}_{\tau_i}}}.
\end{equation}
Taking the square and summing over $t \in \mathcal{T}_i$ yields:
\begin{equation*}
\begin{aligned}
(\mu(0) - \mu(-1))^3 |\mathcal{T}_i | &\leq (S/2) (1+2S)\gamma_{\tau_i}(\delta)^2 \sum_{t\in\mathcal{T}_i} \|x_t\|^2_{\mbold{\bar{V}^{-1}_{\tau_i}}}\\
&\leq (S/2) (1+2S)\gamma_{\tau_i}(\delta)^2 \text{Tr}\left(\mbold{\bar{V}^{-1}_{\tau_i}}\sum_{t\in\mcal{T}_i} x_tx_t^\transp\right)\\
&\leq (S/2) (1+2S)\gamma_{\tau_i}(\delta)^2 d   
\end{aligned}
\end{equation*}
and therefore $|\mathcal{T}_i|\leq C_5 d \gamma^2_{\tau_i}(\delta)$. Since $\max_{t\in[T]}\gamma_t(\delta) \lesssim_T \sqrt{d\log(T)}$ we obtain
\begin{equation*}
| \mathcal{T}| = \sum_{i=1}^N |\mathcal{T}_i| \leq C_5 (d+1) d \max_{i=1,\dots,N} \gamma^2_{\tau_i}(\delta) \leq C_6 d^3 \log(T).
\end{equation*}
which we plug in~\cref{eq:boundrmoins} to obtain the desired result, 
\begin{equation*}
\regretmoins_{\theta_\star}(T) \leq C_1d^2\log^2(T) + C_6 d^3 \log(T).
\end{equation*}
Here $C_5$ and $C_6$ are universal constants hiding dependencies in $\text{poly}(S)$.

\end{proof}
\subsection{Proof of \cref{thm:regretball}}
\thmregretball*

\begin{proof}
	The result is easily obtained by merging \cref{thm:generalregret} with \cref{eq:rmoinsball} in \cref{prop:lengthtransitory}.
\end{proof}

%% file: appendix/app_lower_bounds.tex
\newpage

\section{\uppercase{Regret Lower-Bound}}
\label{app:lowerbound}

We give below a statement of \cref{thm:lowerboundlocal} which is more detailed than its version in the main text. In particular we emphasize the fact that $\epsilon_T$ is \emph{small} enough that $\theta_\star$ and all the alternative packing $\left\{\ltwo{\theta-\theta_\star}\leq \epsilon_T\right\}$ have roughly the same problem-dependent constants (cf \textbf{2.} in \cref{thm:lowerboundlocal}).

\thmlowerboundlocal*
%

\subsection{Proof of \cref{thm:lowerboundlocal}}
The strategy for proving this result is the following: for any policy $\pi$,\footnote{The policy is arbitrary, we only ask that at round $t$ its actions are $\mcal{F}_t$-adapted.} we will assume that for a well-chosen set $\Xi$ we have:
\begin{align*}
    \forall\theta\in\Xi, \quad \regret^\pi_\theta(T) = \mcal{O}\left(d\sqrt{\frac{T}{\kappa_\star(\theta)}}\right)\; .
\end{align*}
We shall arrive to a contradiction of the form:
\begin{align*}
    \exists\theta\in\Xi \quad \text{s.t} \quad \regret^\pi_\theta(T) = \Omega\left(d\sqrt{\frac{T}{\kappa_\star(\theta)}}\right)\, .
\end{align*}

\begin{proof}
 
In the following, we fix the policy $\pi$. We follow \cite{lattimore2020bandit} and will note $(\Omega_t,\mcal{F}_t,\mbb{P}_{\pi\theta})$
the \emph{canonical} bandit probability space at round $t$ under the parameter $\theta$. We refer the interested reader to \cite[Section 4.7]{lattimore2020bandit} for a thorough definition of this probability space. To simplify notations, we will denote $\mbb{P}_\theta=\mbb{P}_{\pi\theta}$ the 
probability measure of the random sequence $\{x_1,r_2,..,x_T, r_{T+1}\}$, obtained by having $\pi$ interact with the environment parameter $\theta$. Recall that we work in a logistic bandit setting, meaning that at any round $t$:
\begin{align*}
    \mathbb{P}_\theta(r_t \vert x_t) = \text{Bernoulli}(\mu(x_t^\transp\theta))
\end{align*}
where $\mu(z)=(1+\exp(-z))^{-1}$ is the logistic function. Note that when $\mcal{X}=\mcal{S}_d(0,1)$ we have $\kappa_\star(\theta)=\kappa_\mcal{X}(\theta)$ for any $\theta$. We therefore use the notation $\kappa(\theta)$ for short. We will need the following result, of which we defer the proof to \cref{subsec:proofregretlowerboundzero}.
\begin{restatable}{prop}{propregretlowerboundzero}
\label{prop:regretlowerboundzero}
    For all $\theta\in\mbb{R}^d$ the following holds:
    \begin{align}
    	 \regret_\theta(T)&\geq  \frac{\ltwo{\theta}}{\kappa(\theta)}\sum_{i=1}^{d} \mbb{E}_\theta\left[\sum_{t=1}^T\left[ x(\theta)-x_t\right]_i^2\right]
	 \label{eq:regretlowerboundzerozero}
    \end{align}
    Further if $\ltwo{\theta}\geq 1$:
    \begin{align}
        \regret_\theta(T) &\geq \frac{1}{6}\mbb{E}_\theta\left[\sum_{t=1}^{T} \dot{\mu}(x_t^\transp\theta)\left\lVert x(\theta)-x_t \right\rVert^2\right] \label{eq:regretlowerboundzero}
    \end{align}
\end{restatable}

In this proof we will assume that $\ltwo{\theta_\star}\geq 1$ (which implies that $\kappa(\theta_\star)\geq 5$).\footnote{This assumption can  be avoided, and we make it here to simplify computations and avoid clutter. Note that $\kappa(\theta_\star)\geq 5$  is precisely the region of interest for this lower-bound, \emph{i.e} large values of $\kappa$.}
 Let $\{e_i\}_{i=1}^d$ the canonical basis of $\mathbb{R}^d$ and without loss of generality assume that $\theta_\star = \ltwo{\theta_*} e_1$. With such notations, we now introduce the set of \emph{unidentifiable} parameters:
\begin{align*}
   \Xi \defeq \left\{ \theta_* + \epsilon\sum_{2=1}^d v_ie_i\, , \quad   v\in\{-1,1\}^d\right\}
\end{align*}
where $\epsilon$ is a (small) positive scalar to be tuned later. For now, we will only make the following assumption on $\epsilon$:
\begin{align}
    \varepsilon \leq \ltwo{\theta_\star}/\sqrt{d-1}
    \label{eq:hypeps}
\end{align}
Intuitively, $\Xi$ is a set of slightly perturbed versions of $\theta_\star$. The goal is to set $\epsilon$ small enough so the parameters are indiscernible for a policy interacting with each of them, however large enough so the policy can't perform well on all problems. Note that all the elements $\theta$ of $\Xi$ have the same norm, and henceforth the same $\kappa(\theta)=:\kappa_\epsilon$. As anticipated earlier, we are going to make the hypothesis that for all $\theta\in\Xi$, the regret is dominated by $d\sqrt{T/\kappa_\epsilon}$. Note that if this assumption does not hold, then by definition there exists $\theta\in\Xi$ such that $\regret_\theta(T)=\Omega(d\sqrt{T/\kappa_\epsilon})$ and the proof is over.

\begin{hyp*}
There exists a universal constant $C$ such that:
\begin{align}
    \forall \theta\in\Xi,\quad  R_\theta(T) \leq Cd\sqrt{\frac{T}{\kappa_\epsilon}}\tag{\textbf{\color{purple}H1}}\label{hyp:regret}
\end{align}
Without loss of generality, we will take $C=1$\footnote{This assumption is made to avoid clutter and is not necessary. Keeping $C$ only impacts our lower-bound by a universal constant, independent of the problem.}.
\end{hyp*}

Starting from \cref{eq:regretlowerboundzero} we are going to provide a first lower-bound of the regret for any $\theta\in\Xi$. To do so, introduce for any direction $i\in[d,2]$ the event:
\begin{align*}
    A_i(\theta) \defeq \left\{ \left[x_\star(\theta)-x_\star(\theta_\star)\right]_i\cdot \left[\frac{1}{T}\sum_{t=1}^T x_t - x_\star(\theta_\star)\right]_i \geq 0\right\}
\end{align*}
We have the following lower-bound, which proof is deferred to \cref{subsec:prooflowerboundregretdecomposition}. 
\begin{restatable}{lemma}{lowerboundregretdecomposition}
\label{lemma:lowerboundregretdecomposition}
For any $\theta\in\Xi$ we have:
\begin{align*}
    \regret_\theta(T) \geq\frac{T \epsilon^2}{2\kappa_\epsilon\ltwo{\theta_\star}} \sum_{i=2}^d\mbb{P}_\theta(A_i(\theta)) 
\end{align*}
\end{restatable}
The goal is now to find one $\theta\in\Xi$ such that the above lower-bound is large. This can be done thanks to a \emph{averaging hammer}, as in \cite[Section 24.1]{lattimore2020bandit}. We will need a \emph{flipping} operator $\flip{i}{\cdot}$ which for any $\theta\in\Xi$ changes the sign of the $i$\textsuperscript{th} coordinate of $\theta$. Formally, let:
\begin{align}
    \left[\flip{i}{\theta}\right]_i = - [\theta]_i\quad \text{ and }  \quad \left[\flip{i}{\theta}\right]_j = [\theta]_j \, \, \text{ for all } \, j\neq i
\end{align}
In the following Lemma, we show that the average value of $\sum_{i=2}^d\mbb{P}_\theta(A_i(\theta))$ over $\Xi$ is linked to the average relative entropy (denoted $D_{\text{KL}}$) between \emph{flipped} versions of $\theta$. 

\begin{restatable}[Averaging Hammer]{lemma}{lemmaaveraginghammer}
\label{lemma:averaginghammer}
    The following holds:
    \begin{align*}
        \frac{1}{\vert \Xi\vert }\sum_{\theta\in\Xi}\sum_{i=2}^d \mbb{P}_{\theta}(A_i(\theta)) \geq  \frac{d}{4} - \frac{\sqrt{d}}{2}\sqrt{\frac{1}{\vert \Xi\vert} \sum_{\theta\in\Xi}\sum_{i=2}^d D_{\text{KL}}\left(\mbb{P}_\theta,\mbb{P}_{\flip{i}{\theta}}\right)}
    \end{align*}
\end{restatable}
The proof is deferred to \cref{subsec:proofaveraginghammer}. 
We now have to characterize this average relative entropy. This is done in the following Lemma, which proof is presented in \cref{subsec:proofaveragerelativeentropy}.

\begin{restatable}[Average Relative Entropy ]{lemma}{lemmaaveragerelativeentropy} Under Hypothesis~\eqref{hyp:regret} we have:
\label{lemma:averagerelativeentropy}
\begin{align*}
    \frac{1}{\vert \Xi\vert} \sum_{\theta\in\Xi}\sum_{i=2}^d D_{\text{KL}}\left(\mbb{P}_\theta,\mbb{P}_{\flip{i}{\theta}}\right) \leq  \frac{2}{\kappa_\epsilon}dT\epsilon^4\exp(4\epsilon) + 4d\epsilon^2\exp(4\epsilon)(6+\frac{d}{2}\epsilon^2)\sqrt{\frac{T}{\kappa_\epsilon}}
\end{align*}
\end{restatable}

Combining \cref{lemma:averaginghammer,lemma:averagerelativeentropy} we therefore obtain that:
\begin{align*}
     \frac{1}{\vert \Xi\vert }\sum_{\theta\in\Xi}\sum_{i=2}^d \mbb{P}_{\theta}(A_i(\theta)) \geq \frac{d}{4}\left[1-2\left(2\epsilon^4\frac{T}{\kappa_\epsilon}+ 24\epsilon^2\sqrt{\frac{T}{\kappa_\epsilon}} + 2d\epsilon^4\sqrt{\frac{T}{\kappa_\epsilon}}\right)^{1/2}\exp(2\epsilon)\right]
\end{align*}
Because this results holds for an average over $\Xi$, it must still be true for at least one $\tilde\theta\in\Xi$. In other words, there exists $\tilde\theta\in\Xi$ such that:
\begin{align*}
    \sum_{i=2}^d \mbb{P}_{\tilde\theta}(A_i(\tilde\theta)) \geq \frac{d}{4}\left[1-2\left(2\epsilon^4\frac{T}{\kappa_\epsilon}+ 24\epsilon^2\sqrt{\frac{T}{\kappa_\epsilon}} + 2d\epsilon^4\sqrt{\frac{T}{\kappa_\epsilon}}\right)^{1/2}\exp(2\epsilon)\right]
\end{align*}

Thanks to \cref{lemma:lowerboundregretdecomposition} we therefore have that it exists $\tilde\theta\in\Xi$ such that:
\begin{align*}
    \regret_{\tilde\theta}(T) \geq dT\frac{\epsilon^2}{8\ltwo{\theta_\star}\kappa_\epsilon}\left[1-2\left(2\epsilon^4\frac{T}{\kappa_\epsilon}+ 24\epsilon^2\sqrt{\frac{T}{\kappa_\epsilon}} + 2d\epsilon^4\sqrt{\frac{T}{\kappa_\epsilon}}\right)^{1/2}    \exp(2\epsilon)\right]
\end{align*}
We only have left to tune $\epsilon$ to prove our result. Taking $\epsilon^2=\frac{1}{32}\sqrt{\frac{\kappa_\epsilon}{T}}$ yields, after some computations that:
\begin{align*}
    \regret_{\tilde\theta}(T) \geq \frac{1}{256\ltwo{\theta_\star}}d\sqrt{\frac{T}{\kappa_\epsilon}}\left(1-2\left(\frac{24576}{32^4}+\frac{2}{32^4}d\sqrt{\frac{\kappa_\epsilon}{T}}\right)^{1/2}\exp\left(\frac{2}{\sqrt{32}}\sqrt{\frac{\kappa_\epsilon}{T}}\right)\right)
\end{align*}
When $T\geq d^2\kappa$ (and therefore $T\geq \kappa$) we obtain:
\begin{align*}
    \regret_{\tilde\theta}(T) &\geq \frac{1}{256\ltwo{\theta_\star}}d\sqrt{\frac{T}{\kappa_\epsilon}}\left(1-\left(\frac{98312}{32^4}\right)^{1/2}\exp\left(\frac{1}{\sqrt{8}}\right)\right)\\
    &\geq \frac{1}{512\ltwo{\theta_\star}}d\sqrt{\frac{T}{\kappa_\epsilon}}
\end{align*}
To sum-up, we have shown that when Hypothesis~\eqref{hyp:regret} holds, there exists $\tilde\theta\in\Xi$ such that $\regret_{\tilde\theta}(T) = \Omega(d\sqrt{\frac{T}{\kappa_\epsilon}})$. Note that if Hypothesis~\eqref{hyp:regret} did not hold, then by definition such a parameter would also exist. 
This proves part \textbf{1.} of the claim; indeed by setting $\epsilon_T^2 = \frac{1}{32}\sqrt{\kappa_\epsilon/T}$ we have shown that for \emph{any} policy $\pi$ if $T\geq d^2\kappa_\epsilon$:
\begin{align*}
	\max_{\ltwo{\theta-\theta_\star}^2\leq d\epsilon^2_T} \regret_{\theta}^\pi(T) = \Omega\left(d\sqrt{\frac{T}{\kappa_\epsilon}}\right)
\end{align*}
and therefore since $\kappa_\epsilon\geq \kappa_\star(\theta_\star)$ there exists $\tilde\epsilon_T$ small enough ($\tilde\epsilon_T=\sqrt{d}\epsilon_T$) such that:
\begin{align*}
	\minmaxregret_{\theta_\star,T}(\tilde\epsilon_T) = \Omega\left(d\sqrt{\frac{T}{\kappa_\star(\theta_\star)}}\right)
\end{align*}

 This formulation is somehow a degradation of the result we obtained, because we showed that under $\tilde\theta$ (the hard nearby instance) the regret is $\Omega(d\sqrt{T/\kappa_\epsilon})$ and therefore \emph{directly involves} the problem-dependent constant $\kappa(\tilde\theta)=\kappa_\epsilon$. This degradation is however mild: our bound is \emph{local} and $\epsilon_T$ is \emph{small}. As a result, $\theta_\star$ and any nearby alternative $\theta\in\Xi$ fundamentally have the same problem-dependent constants. We now turn this intuition rigorous and prove part \textbf{2.} of the Theorem. By \cref{lemma:fthirdselfconcordance} for any $\theta$:
 \begin{align*}
 	\dot\mu\left(x_\star(\theta)^\transp\theta\right)\exp\left(-\left\vert x_\star(\theta_\star)^\transp\theta_\star-x_\star(\theta)^\transp\theta\right\vert\right) \leq \dot\mu\left(x_\star(\theta_\star)^\transp\theta_\star\right) \leq  \dot\mu\left(x_\star(\theta)^\transp\theta\right)\exp\left(\left\vert x_\star(\theta_\star)^\transp\theta_\star-x_\star(\theta)^\transp\theta\right\vert\right)
\end{align*}
which yields that if $\ltwo{\theta-\theta_\star}^2\leq d\epsilon^2_T$:
\begin{align*}
  \dot\mu\left(x_\star(\theta)^\transp\theta\right)\exp\left(-\sqrt{d}\epsilon_T\right) \leq \dot\mu\left(x_\star(\theta_\star)^\transp\theta_\star\right) \leq  \dot\mu\left(x_\star(\theta)^\transp\theta\right)\exp\left(\sqrt{d}\epsilon_T\right)
 \end{align*}
We obtain the desired result by noting that $d\epsilon^2_T = (1/32)d\sqrt{\kappa(\theta)/T}\leq 1/32$ when $T\geq d^2\kappa(\theta)$:
\end{proof}

\subsection{A Global Lower-Bound}
As announced in the main text, this local-minimax bound easily implies a global one. We state it here for the sake of completeness. 

\begin{cor}[Global Lower-Bound]
Let $\mcal{X}=\mcal{S}_d(0,1)$. For any policy $\pi$ and for any tuple $(T,d,\kappa)$ such that $T\geq d^2\kappa$, there exists a problem $\theta$ such that $\kappa_\star(\theta)=\kappa$ and:
\begin{align*}
	\regret_{\theta}^\pi(T) = \Omega\left(d\sqrt{\frac{T}{\kappa}}\right)
\end{align*}
\end{cor}
\begin{proof}
	This result is a direct consequence of \cref{thm:lowerboundlocal}. The proof only requires to select a nominal instance $\theta_\star$ which $\ell_2$-norm is large enough so that for any $\theta\in\Xi$ we have $\kappa_\star(\theta)=\kappa$. 
\end{proof}

\subsection{Proof of Proposition~\ref{prop:regretlowerboundzero}}
\label{subsec:proofregretlowerboundzero}
\propregretlowerboundzero*
\begin{proof}
We start by proving the second result. By definition of the regret:
\begin{align*}
    \regret_\theta(T) &= \mbb{E}_\theta\left[\sum_{t=1}^T \mu(x_\star(\theta)^\transp\theta) - \mu(x_t^\transp\theta)\right] &\\
    &= \mbb{E}_\theta\left[\sum_{t=1}^T \alpha\left(\theta,x_\star(\theta), x_t\right) \left(x_\star(\theta)^\transp\theta - x_t^\transp\theta\right)\right] &(\text{mean-value theorem})\\
    &\geq \mbb{E}_\theta\left[\sum_{t=1}^T \frac{\dot\mu(x_t^\transp \theta)}{1+\vert\theta^\transp(x_\star(\theta)-x_t)\vert}\left(x_\star(\theta)^\transp\theta - x_t^\transp\theta\right)\right] &(\text{\cref{lemma:firstselfconcordance}})\\
    &\geq \frac{1}{1+2\ltwo{\theta}} \mbb{E}_\theta\left[\sum_{t=1}^T \dot\mu(x_t^\transp \theta)\left(x_\star(\theta)^\transp\theta - x_t^\transp\theta\right)\right] &(\ltwo{x}\leq 1 \, \forall x\in\mcal{X})\\
    &\geq \frac{\ltwo{\theta}}{1+2\ltwo{\theta}} \mbb{E}_\theta\left[\sum_{t=1}^T \dot\mu(x_t^\transp \theta)\left(1 - x_t^\transp\frac{\theta}{\ltwo{\theta}}\right)\right] &\\
    &\geq \frac{\ltwo{\theta}}{2+4\ltwo{\theta}} \mbb{E}_\theta\left[\sum_{t=1}^T \dot\mu(x_t^\transp \theta)\ltwo{x_\star(\theta)-x_t}^2\right]
\end{align*}
where in the last line we used that for all $x,y\in\mcal{S}_d(0,1)$ we have $1-x^\transp y = \frac{1}{2}\ltwo{x-y}^2$. Using the fact that $\ltwo{\theta}\geq 1$ yields the second result.

 A similar bound can be written by using $\alpha(\theta,x_\star(\theta),x_t)\geq \dot\mu(x_\star(\theta)^\transp\theta)$. Namely, we obtain:
\begin{align*}
	   \regret_\theta(T) &\geq \mbb{E}_\theta\left[\sum_{t=1}^T \dot\mu(x_\star(\theta_\star)^\transp \theta)\left(x_\star(\theta)^\transp\theta - x_t^\transp\theta\right)\right] \\
	   &\geq \frac{\ltwo{\theta}}{\kappa_\star(\theta)}\mbb{E}_\theta\left[\sum_{t=1}^T\ltwo{x_\star(\theta)-x_t}^2\right] \\
	   &\geq \frac{\ltwo{\theta}}{\kappa_\star(\theta)}\mbb{E}_\theta\left[\sum_{t=1}^T\ltwo{x_\star(\theta)-x_t}^2\right] &(\ltwo{\theta_\star}\geq 1)\\
	   &\geq \frac{\ltwo{\theta}}{\kappa_\star(\theta)}\mbb{E}_\theta\left[\sum_{t=1}^T \sum_{i=1}^d \left[x_\star(\theta)-x_t\right]^2_i\right]
\end{align*}
Using the linearity of the expectation delivers the first claim. 
\end{proof}

\subsection{Proof of Lemma~\ref{lemma:lowerboundregretdecomposition}}
\label{subsec:prooflowerboundregretdecomposition}
\lowerboundregretdecomposition*

\begin{proof}

From Proposition~\ref{prop:regretlowerboundzero} we have that:
\begin{align*}
    \regret_\theta(T) &\geq \frac{\ltwo{\theta}}{\kappa(\theta)}\sum_{i=1}^d\mbb{E}_\theta\left[\sum_{t=1}^T \left[x_\star(\theta)-x_t\right]_i^2\right]\\
    &\geq \frac{\ltwo{\theta}}{\kappa(\theta)}\sum_{i=1}^d\mbb{E}_\theta\left[\sum_{t=1}^T \left[x_\star(\theta)-x_t\right]_i^2 \mathds{1}\left\{A_i(\theta)\right\}\right] \\
    &= \frac{\ltwo{\theta}}{\kappa(\theta)}\sum_{i=1}^d\mbb{E}_\theta\left[\sum_{t=1}^T \left[x_\star(\theta)-x_\star(\theta_\star)+x_\star(\theta_\star)-x_t\right]_i^2 \mathds{1}\left\{A_i(\theta)\right\}\right] \\
    &=\frac{\ltwo{\theta}}{\kappa(\theta)}\sum_{i=1}^d\left[x_\star(\theta)-x_\star(\theta_\star)\right]_i^2 \mbb{E}_\theta\left[\mathds{1}\left\{A_i(\theta)\right\}\right]  \\
    &\quad + \frac{\ltwo{\theta}}{\kappa(\theta)}\sum_{i=1}^d \mbb{E}_\theta\left[\sum_{t=1}^T\left[ x_\star(\theta_\star)-x_t\right]_i^2\mathds{1}\left\{A_i(\theta)\right\}\right]  \\
    &\quad +\frac{2T\ltwo{\theta}}{\kappa(\theta)}\sum_{i=1}^d \mbb{E}_\theta\left[\mathds{1}\left\{A_i(\theta)\right\} \left[x_\star(\theta_\star)-\frac{1}{T}\sum_{t=1}^T x_t\right]_i \left[x_\star(\theta)-x_\star(\theta_\star)\right]_i \right] \\
    &\geq \frac{\ltwo{\theta}}{\kappa(\theta)}T\sum_{i=1}^d\left[x_\star(\theta)-x_\star(\theta_\star)\right]_i^2 \mbb{E}_\theta\left[\mathds{1}\left\{A_i(\theta)\right\}\right]
\end{align*}
where in the last line we lower-bounded the last two terms by 0 (this was done for the second term thanks to the definition of $A_i(\theta)$). Some easy computations yield the result:
\begin{align*}
    \regret_\theta(T) &\geq T \frac{\ltwo{\theta}}{\kappa_\epsilon}\frac{\epsilon^2}{\ltwo{\theta_\star}^2+(d-1)\epsilon^2}\sum_{i=2}^d \mbb{E}_\theta\left[\mathds{1}\left\{A_i(\theta)\right\}\right] \\
    &\geq T \frac{\ltwo{\theta}}{\kappa_\epsilon} \frac{\epsilon^2}{2\ltwo{\theta_\star}^2}\sum_{i=2}^d \mbb{E}_\theta\left[\mathds{1}\left\{A_i(\theta)\right\}\right] &(\text{Equation~\eqref{eq:hypeps}})\\
    &= \frac{T \epsilon^2}{2\kappa_\epsilon\ltwo{\theta_\star}} \sum_{i=2}^d\mbb{P}_\theta(A_i(\theta)) 
\end{align*} 
\end{proof}

\subsection{Proof of Lemma~\ref{lemma:averaginghammer}}
\label{subsec:proofaveraginghammer}
\lemmaaveraginghammer*
\begin{proof} Let us fix $\theta\in\Theta$ and $i\in[2,d]$. Note that:
\begin{align}
    \mbb{P}_{\flip{i}{\theta}}(A_i(\flip{i}{\theta}) &\geq \mbb{P}_\theta(A_i(\flip{i}{\theta})) - D_{\text{TV}}\left(\mbb{P}_\theta,\mbb{P}_{\flip{i}{\theta}}\right) \notag\\
    &\geq \mbb{P}_\theta(A_i(\flip{i}{\theta})) - \sqrt{\frac{1}{2}D_{\text{KL}}\left(\mbb{P}_\theta,\mbb{P}_{\flip{i}{\theta}}\right)} &(\text{Pinsker inequality}) \notag\\
    &\geq \mbb{P}_\theta(A_i^C(\theta)) - \sqrt{\frac{1}{2}D_{\text{KL}}\left(\mbb{P}_\theta,\mbb{P}_{\flip{i}{\theta}}\right)} \label{eq:linkpkl}
\end{align}
where $D_{\text{KL}}$ denotes the relative entropy, and where we used the fact that:
\begin{align*}
    A_i(\flip{i}{\theta}) &= \left\{ \left[x_\star(\flip{i}{\theta})-x_\star(\theta_\star)\right]_i\cdot \left[\frac{1}{T}\sum_{t=1}^T x_t - x_\star(\theta_\star)\right]_i \geq 0\right\} &\text{(definition)} \\
    &= \left\{ \left[x_\star(\flip{i}{\theta})\right]_i\cdot \left[\frac{1}{T}\sum_{t=1}^T x_t \right]_i \geq 0\right\} & (x_\star(\theta_\star)_i=0) \\
    &= \left\{ -\left[x_\star(\theta)\right]_i\cdot \left[\frac{1}{T}\sum_{t=1}^T x_t \right]_i \geq 0\right\} &( [\flip{i}{\theta}]_i = - [\theta]_i)\\
    &= A_i(\theta)^C
\end{align*}
In the following, we denote $\Xi_i^+:=\{\theta\in \Xi \text{ such that } \text{sign}([\theta]_i)>0\}$ and $\Xi_i^-:=\{\theta\in \Xi \text{ such that } \text{sign}([\theta]_i)<0\}$. Then by averaging over $\Xi$:
\begin{align*}
    \frac{1}{\vert \Xi\vert}\sum_{\theta\in\Xi}\sum_{i=2}^d \mbb{P}_\theta(A_i(\theta)) &= \frac{1}{\vert \Xi\vert} \sum_{i=2}^d\sum_{\theta\in\Xi} \mbb{P}_\theta(A_i(\theta))\\
    &= \frac{1}{\vert \Xi\vert} \sum_{i=2}^d \sum_{\theta\in\Xi_i^+}\left(\mbb{P}_\theta(A_i(\theta)) + \mbb{P}_{\flip{i}{\theta}}(A_i({\flip{i}{\theta}})\right)\\
    &\geq \frac{1}{\vert \Xi\vert}\sum_{i=2}^d \sum_{\theta\in\Xi_i^+} \mbb{P}_\theta(A_i(\theta)) + \mbb{P}_\theta(A_i^C(\theta)) - \sqrt{\frac{1}{2}D_\text{KL}\left(\mbb{P}_\theta,\mbb{P}_{\flip{i}{\theta}}\right)}  &(\text{Equation~\eqref{eq:linkpkl}})\\
    &\geq \frac{1}{\vert \Xi\vert}\sum_{i=2}^d \sum_{\theta\in\Xi_i^+} 1 - \sqrt{\frac{1}{2}D_\text{KL}\left(\mbb{P}_\theta,\mbb{P}_{\flip{i}{\theta}}\right)}  &
\end{align*}
Repeating the same operation but referencing to $\Xi_i^-$ we easily get that:
\begin{align*}
     \frac{2}{\vert \Xi\vert}\sum_{\theta\in\Xi}\sum_{i=2}^d \mbb{P}_\theta(A_i(\theta)) &\geq \frac{1}{\vert \Xi\vert}\sum_{i=2}^d \sum_{\theta\in\Xi_i^+\cup\Xi_i^-} 1 - \sqrt{\frac{1}{2}D_\text{KL}\left(\mbb{P}_\theta,\mbb{P}_{\flip{i}{\theta}}\right)} \\
     &=  \frac{1}{\vert \Xi\vert}\sum_{i=2}^d \sum_{\theta\in\Xi} 1 - \sqrt{\frac{1}{2}D_\text{KL}\left(\mbb{P}_\theta,\mbb{P}_{\flip{i}{\theta}}\right)} \\
     &= (d-1) - \sum_{i=2}^d\frac{1}{\vert \Xi\vert} \sum_{\theta\in\Xi} \sqrt{\frac{1}{2}D_{\text{KL}}\left(\mbb{P}_\theta,\mbb{P}_{\flip{i}{\theta}}\right)} \\
     &\geq \frac{d}{2} - \sum_{i=2}^d\frac{1}{\vert \Xi\vert} \sum_{\theta\in\Xi} \sqrt{\frac{1}{2}D_{\text{KL}}\left(\mbb{P}_\theta,\mbb{P}_{\flip{i}{\theta}}\right)} &(d\geq 1) \\
     &\geq \frac{d}{2} - \frac{1}{\sqrt{2}}\sum_{i=2}^d\sqrt{\frac{1}{\vert \Xi\vert} \sum_{\theta\in\Xi} D_{\text{KL}}\left(\mbb{P}_\theta,\mbb{P}_{\flip{i}{\theta}}\right)} & (\text{Jensen inequality})\\
     &\geq  \frac{d}{2} - \sqrt{\frac{d-1}{2}}\sqrt{\sum_{i=2}^d\frac{1}{\vert \Xi\vert} \sum_{\theta\in\Xi} D_{\text{KL}}\left(\mbb{P}_\theta,\mbb{P}_{\flip{i}{\theta}}\right)} & (\text{Cauchy-Schwartz})\\
     &\geq \frac{d}{2} - \sqrt{d}\sqrt{\sum_{i=2}^d\frac{1}{\vert \Xi\vert} \sum_{\theta\in\Xi} D_{\text{KL}}\left(\mbb{P}_\theta,\mbb{P}_{\flip{i}{\theta}}\right)} & 
\end{align*}
which proves the announced result.

\subsection{Proof of Lemma~\ref{lemma:averagerelativeentropy}}
\label{subsec:proofaveragerelativeentropy}

\lemmaaveragerelativeentropy*

We will use the following result to control the relative entropy between two different parameters. It is a consequence of the relative entropy decomposition presented in \cite{lattimore2020bandit} along with the fact that the relative entropy is dominated by the chi-square divergence. The proof is deferred to Section~\ref{subsec:lemmakldecomposition}. 

\begin{restatable}[Relative Entropy Decomposition]{lemma}{lemmakldecomposition}
\label{lemma:kldecomposition}
For any $\theta, \theta'$ we have that:
\begin{align*}
    D_{\textnormal{KL}}\left(\mbb{P}_\theta, \mbb{P}_{\theta'}\right) \leq  \mbb{E}_\theta\left[\sum_{t=1}^T \frac{\left(\mu(x_t^\transp\theta) - \mu(x_t^\transp\theta'\right)^2}{\dot\mu(x_t^\transp\theta')}\right]
\end{align*}
\end{restatable}

Applying this result between $\mbb{P}_\theta$ and $\mbb{P}_{\flip{i}{\theta}}$ yields:
\begin{align*}
    D_{\textnormal{KL}}(\mbb{P}_\theta,\mbb{P}_{\flip{i}{\theta}}) &\leq \mbb{E}_\theta\left[\sum_{t=1}^T \frac{\left(\mu(x_t^\transp\theta) - \mu(x_t^\transp\flip{i}{\theta}\right)^2}{\dot\mu(x_t^\transp\flip{i}{\theta})}\right] &\\
    &\leq  \mbb{E}_\theta\left[\sum_{t=1}^T \frac{\alpha^2(x_t,\theta,\flip{i}{\theta})}{\dot\mu(x_t^\transp\flip{i}{\theta})}\left\{x_t^\transp(\theta-\flip{i}{\theta})\right\}^2\right] &(\text{mean-value theorem})\\
\end{align*}
We are now going to link $\alpha^2(x_t,\theta,\flip{i}{\theta})$ to $\dot\mu(x_t^\transp\flip{i}{\theta})$ and $\dot\mu(x_t^\transp\theta)$ thanks to the self-concordance. Indeed, it is easy to show (see the proof of Lemma~\ref{lemma:firstselfconcordance}) that for all $z_1,z_2$ we have $\dot{\mu}(z_1) \leq \dot\mu(z_2)\exp(\vert z_1-z_2\vert)$. We therefore have the following inequalities:
\begin{align*}
   \alpha(x_t,\theta,\flip{i}{\theta}) &\leq \dot\mu(x_t^\transp\theta) \exp\left(\left\vert x_t^\transp(\theta-\flip{i}{\theta}\right\vert\right) \quad \text{and} \\ \alpha^2(x_t,\theta,\flip{i}{\theta}) &\leq \dot\mu(x_t^\transp\flip{i}{\theta})\exp\left(\left\vert x_t^\transp(\theta-\flip{i}{\theta}\right\vert\right)
\end{align*}
Plugging this in the relative entropy decomposition we obtain:
\begin{align*}
    D_{\textnormal{KL}}(\mbb{P}_\theta,\mbb{P}_{\flip{i}{\theta}}) &\leq \mbb{E}_\theta\left[\sum_{t=1}^T \frac{\alpha^2(x_t,\theta,\flip{i}{\theta})}{\dot\mu(x_t^\transp\flip{i}{\theta})}\left\{x_t^\transp(\theta-\flip{i}{\theta})\right\}^2\right] \\
    &\leq \mbb{E}_\theta\left[\sum_{t=1}^T \dot\mu(x_t^\transp\theta)\left\{x_t^\transp(\theta-\flip{i}{\theta})\right\}^2\right]\exp\left(2\left\vert x_t^\transp(\theta-\flip{i}{\theta} \right\vert\right)\\
    &\leq \exp(4\epsilon) \mbb{E}_\theta\left[\sum_{t=1}^T \dot\mu(x_t^\transp\theta)\left\{x_t^\transp(\theta-\flip{i}{\theta})\right\}^2\right]\\
    &\leq 2\epsilon^2\exp(4\epsilon)\mbb{E}_\theta\left[\sum_{t=1}^T \dot\mu(x_t^\transp\theta)[x_t]_i^2\right] \\
    &= 2\epsilon^2\exp(4\epsilon)\mbb{E}_\theta\left[\sum_{t=1}^T \dot\mu(x_t^\transp\theta)\left[x_t-x_\star(\theta)+x_\star(\theta)\right]_i^2\right]\\
    &\leq 4\epsilon^2\exp(4\epsilon)\mbb{E}_\theta\left[\sum_{t=1}^T \dot\mu(x_t^\transp\theta)\left[x_t-x_\star(\theta)\right]_i^2+\sum_{t=1}^T \dot\mu(x_t^\transp\theta)\left[x_\star(\theta)\right]_i^2\right]
\end{align*}
where we last used the fact that $(a+b)^2\leq 2(a^2+b^2)$. Therefore by summing over $d$:
\begin{align*}
    \sum_{d=2}^d D_{\textnormal{KL}}(\mbb{P}_\theta,\mbb{P}_{\flip{i}{\theta}}) &\leq 4\epsilon^2\exp(4\epsilon) \mbb{E}_\theta\left[\sum_{t=1}^T \sum_{i=2}^d \dot\mu(x_t^\transp\theta)\left[x_t-x_\star(\theta)\right]_i^2+\sum_{t=1}^T\sum_{i=2}^d \dot\mu(x_t^\transp\theta)\left[x_\star(\theta)\right]_i^2\right] \\
    &\leq 4\epsilon^2\exp(4\epsilon) \mbb{E}_\theta\left[\sum_{t=1}^T \sum_{i=1}^d \dot\mu(x_t^\transp\theta)\left[x_t-x_\star(\theta)\right]_i^2+\sum_{t=1}^T\sum_{i=1}^d \dot\mu(x_t^\transp\theta)\left[x_\star(\theta)\right]_i^2\right] \\
    &\leq 4\epsilon^2\exp(4\epsilon) \mbb{E}_\theta\left[\sum_{t=1}^T  \dot\mu(x_t^\transp\theta)\ltwo{ x_t-x_\star(\theta)}^2+d\frac{\epsilon^2}{\ltwo{\theta_\star}^2 + (d-1)\epsilon^2}\sum_{t=1}^T \dot\mu(x_t^\transp\theta)\right] \\
    &\leq 4\epsilon^2\exp(4\epsilon) \mbb{E}_\theta\left[\sum_{t=1}^T  \dot\mu(x_t^\transp\theta)\ltwo{ x_t-x_\star(\theta)}^2+\frac{d}{2}\epsilon^2\sum_{t=1}^T \dot\mu(x_t^\transp\theta)\right]
\end{align*}
where we used Equation~\eqref{eq:hypeps} and the fact that $\ltwo{\theta_\star}\geq 1$. Using Proposition~\ref{prop:regretlowerboundzero} (more precisely Equation~\eqref{eq:regretlowerboundzero}) we obtain:
\begin{align}
    \sum_{d=2}^d D_{\textnormal{KL}}(\mbb{P}_\theta,\mbb{P}_{\flip{i}{\theta}}) &\leq 4\epsilon^2\exp(4\epsilon) \left(6\regret_\theta(T)+\frac{d}{2}\epsilon^2\mbb{E}_\theta\left[\sum_{t=1}^T \dot\mu(x_t^\transp\theta)\right]\right)
    \label{eq:klbound1}
\end{align}
We finish the proof by resorting to a Taylor expansion of $\dot\mu(x_t^\transp\theta)$. Formally:
\begin{align*}
    \sum_{t=1}^T \dot{\mu}(x_t^\transp\theta) \leq  \sum_{t=1}^T \left[\dot\mu(x_\star(\theta)^\transp\theta) + \left\vert\int_{v=0}^1 \ddot{\mu}(x_\star(\theta)^\transp\theta + v\theta^\transp(x_t-x_\star(\theta)))dv\right\vert\left\vert \theta^\transp(x_\star(\theta)-x_t)\right\vert\right]
\end{align*}
Using the fact that $\vert \ddot\mu\vert\leq \mu$ and $x_\star(\theta)^\transp\theta \geq x_t^\transp\theta$ we obtain that:
\begin{align*}
    \mbb{E}_\theta\left[\sum_{t=1}^T \dot{\mu}(x_t^\transp\theta)\right] &\leq  \mbb{E}_\theta\left[\sum_{t=1}^T\left[ \dot\mu(x_\star(\theta)^\transp\theta) + \alpha(\theta,x_\star(\theta),x_t) \theta^\transp(x_\star(\theta)-x_t)\right]\right]\\
    &= \frac{T}{\kappa_\epsilon} + \mbb{E}_\theta\left[\sum_{t=1}^T\alpha(\theta,x_\star(\theta),x_t) \theta^\transp(x_\star(\theta)-x_t)\right]\\
    &= \frac{T}{\kappa} + \regret_\theta(T)
\end{align*}
where we used the mean value theorem in the last line (see for instance the beginning of the proof of Proposition~\ref{prop:regretlowerboundzero}. Plugging this result in Equation~\eqref{eq:klbound1} we obtain:

\begin{align*}
    \sum_{d=2}^d D_{\textnormal{KL}}(\mbb{P}_\theta,\mbb{P}_{\flip{i}{\theta}}) &\leq 4\epsilon^2\exp(4\epsilon) \left(6\regret_\theta(T)+\frac{d}{2}\epsilon^2\left(\frac{T}{\kappa_\epsilon} + \regret_\theta(T)\right)\right)
\end{align*}
Averaging over $\Xi$ and since by Hypothesis~\eqref{hyp:regret} we know that $\regret_\theta(T)\leq d\sqrt{T/\kappa_\epsilon}$ we obtain the announced result.
\end{proof}

\subsection{Proof of \cref{lemma:kldecomposition}}
\label{subsec:lemmakldecomposition}

\lemmakldecomposition*

\begin{proof}
Denote $P_{x}^\theta = \mbb{P}_\theta(r\vert x)$. Thanks to \cite[Section 24.1]{lattimore2020bandit} we have:
\begin{align*}
    D_\text{KL}\left(\mbb{P}_\theta, \mbb{P}_{\theta'}\right) &= \mbb{E}_\theta\left[\sum_{t=1}^T D_{\text{KL}}\left(P_{x_t}^\theta,P_{x_t}^{\theta'}\right)\right] \\
    &= \mbb{E}_\theta\left[\sum_{t=1}^T D_{\text{KL}}\left(\text{Bernoulli}(x_t^\transp\theta),\text{Bernoulli}(x_t^\transp\theta')\right)\right]\\
    &\leq \mbb{E}_\theta\left[\sum_{t=1}^T D_{\chi^2}\left(\text{Bernoulli}(x_t^\transp\theta),\text{Bernoulli}(x_t^\transp\theta')\right)\right]\\
\end{align*}
where we used $D_{\text{KL}}\leq D_{\chi^2}$ \cite[Chapter 2]{tsybakov2008introduction}. Using the expression of the $\chi^2$-divergence for Bernoulli random variables finishes the proof. 
\end{proof}

%% file: appendix/app_tractability.tex
\newpage
\section{\uppercase{Tractability of \ouralgorelaxed}}
\label{app:tractability}

\subsection{Proof of \cref{prop:tractablealgo}}
\proptractablealgo*
\begin{proof}
Recall that we assume the arm-set $\mcal{X}$ to be finite. For any $x\in\mcal{X}$ denote:
\begin{align}
	\theta_x \in \argmax_{\theta\in\mcal{E}_t(\delta)} x^\transp\theta
	\label{eq:tractabletheta}
\end{align}
 which is well-defined, as the maximizer of a concave function under a convex constraint. We can now write:
\begin{align}
	\tilde x_t \in \argmax_{x\in\mcal{X}} x^\transp\theta_x
	\label{eq:tractablex}
\end{align}
Since we have $\tilde\theta_t=\theta_{\tilde{x}_t}$ we can prove that the planning of \ouralgorelaxed{} is indeed optimistic:
\begin{align*}
	\tilde{x}_t^\transp\tilde\theta_t &= \tilde{x}^\transp\theta_{\tilde{x}_t} &\\
	 &\geq x^\transp\theta_x  &(\text{\cref{eq:tractablex}})\\
	&\geq x^\transp\theta  &(\text{\cref{eq:tractabletheta}})
\end{align*}
which holds for any $x\in\mcal{X}$ and $\theta\in\mcal{E}_t(\delta)$. This finishes the proof.
\end{proof}

\subsection{Proof of \cref{cor:tractableguarantees}}
\cortractableguarantees*

\begin{proof}
The proof is fairly simple, as this result directly follows from \cref{lemma:ouralgorelaxed}. To see this, note that we only need \underline{two} ingredients to repeat the proof of \cref{thm:generalregret}:
\begin{enumerate}
\item[] \textbf{(1)} We rely on optimism to enforce $x_t^\transp\tilde\theta_t\geq x_\star(\theta_\star)^\transp\theta_\star$. This fact this holds (with high probability) as thanks to \cref{lemma:ouralgorelaxed} we have $\mcal{C}_t(\delta)\subseteq \mcal{E}_t(\delta)$  and therefore $\theta_\star\in\mcal{E}_t(\delta)$ for all $t\geq 1$ with probability at least $1-\delta$.
\item[] \textbf{(2)} We bound the deviation $\lVert \theta-\theta_\star \rVert_{\mbold{H_t(\theta_\star)}}$ by $\bigo{\sqrt{d\log(t)}}$ terms for any $\theta\in\mcal{C}_t(\delta)$ (cf. \cref{prop:bounddevtheta}). The same property holds for any $\theta\in\mcal{E}_t(\delta)$ thanks to \cref{lemma:ouralgorelaxed}.
\end{enumerate}

As a result of \textbf{(1)} and \textbf{(2)} proving that \ouralgorelaxed{} satisfies \cref{thm:generalregret} follows rigorously the same line of proof. The same arguments hold for proving that \ouralgorelaxed{} satisfies \cref{prop:lengthtransitory,thm:regretball}. 

\end{proof}

%% file: appendix/app_self_concordance.tex
\newpage
\section{\uppercase{Self-Concordance Results}}
\label{app:sc}
In this section we state some useful generalized self-concordance results. The first technical result is from \cite[Lemma 9]{faury2020improved}. We provide a proof for the sake of completeness. 

\begin{restatable}{lemma}{lemmafirstselfconcordance}
\label{lemma:firstselfconcordance}
Let $f$ be a strictly increasing function such that $\vert\ddot{f}\vert \leq \dot{f}$, and let $\mcal{Z}$ be any bounded interval of $\mathbb{R}$. Then, for all $z_1,z_2\in\mcal{Z}$:
\begin{align*}
    \int_{v=0}^1 \dot{f}\left(z_1+v(z_2-z_1)\right)dv \geq \frac{\dot{f}(z)}{1+\vert z_1-z_2\vert} \quad \text{ for } z\in\{z_1,z_2\}.
\end{align*}
\end{restatable}
\begin{proof}
    The function $f$ being strictly increasing, we have that $\dot{f}(z)>0$ for any $z\in\mcal{Z}$. Therefore:
    \begin{align}
         & &-1&\leq\frac{\ddot{f}(z)}{\dot{f}(z)} \leq 1 & &\notag\\
         &\Rightarrow &-\vert z_1-z_0\vert&\leq \int_{z_1 \wedge z_0}^{z_1 \vee z_0}\frac{\ddot{f}(z)}{\dot{f}(z)}dz \leq \vert z_1-z_0\vert& &\qquad \text{($z_0\in\mcal{Z}$)}\notag\\
         &\Leftrightarrow &-\vert z_1-z_0\vert &\leq \log\left(\dot{f}(z_1 \vee z_0)/\dot{f}(z_1 \wedge z_0) \right)\leq \vert z_1-z_0\vert & &\notag\\
         &\Leftrightarrow & \dot{f}(z_1 \wedge z_0)\exp\left(-\vert z_1-z_0\vert\right)&\leq \dot{f}(z_1 \vee z_0) \leq \dot{f}(z_1 \wedge z_0)\exp\left(\vert z_1-z_0\vert\right)\,. & &
         \label{eq:maxminsc}
    \end{align}
    Assume for now that $z_2\geq z_1$, let $v\geq0$ and set $z_0=z_1+v(z_2-z_1)$, which is such that $z_0\geq z_1$. Using this definition with the l.h.s inequality of \cref{eq:maxminsc} we easily get:
    \begin{align*}
         \dot{f}\left(z_1+v(z_2-z_1)\right)  &\geq \dot{f}\left(z_1\right)\exp\left(-v\vert z_2-z_1\vert\right)\\
         \Rightarrow \quad \int_{v=0}^{1} \dot{f}\left(z_1+v(z_2-z_1)\right)dv &\geq \dot{f}\left(z_1\right)\frac{1-\exp\left(-\vert z_1-z_2\vert\right)}{\vert z_1-z_2\vert}\\
         &\geq \dot{f}\left(z_1\right) (1+\vert z_1-z_2\vert)^{-1}\,.
    \end{align*}
where the last inequality is easily obtained by using $\exp(x)\geq 1+x$ for all $x\in\mathbb{R}$. The same inequality can be proven when $z_2\leq z_1$ by using the r.h.s inequality of \cref{eq:maxminsc} instead. We have therefore proven the announced result, but only for $z=z_1$. The proof is concluded by realizing than $z_1$ and $z_2$ play a symmetric role in the problem (for instance, perform the change of variable $u\leftarrow (1-v)$ in the integral that we wish to lower-bound).
\end{proof}

We now state a second result, which proof closely follows the one of \cref{lemma:firstselfconcordance}. 

\begin{restatable}{lemma}{lemmasecondselfconcordance}
\label{lemma:secondselfconcordance}
Let $f$ be a strictly increasing function such that $\vert\ddot{f}\vert \leq \dot{f}$, and let $\mcal{Z}$ be any bounded interval of $\mathbb{R}$. Then, for all $z_1,z_2\in\mcal{Z}$:
\begin{align*}
    \int_{v=0}^1 (1-v)\dot{f}\left(z_1+v(z_2-z_1)\right)dv \geq \frac{\dot{f}(z_1)}{2+\vert z_1-z_2\vert}\, .
\end{align*}
\end{restatable}
\begin{proof}
From \cref{eq:maxminsc} it can easily be extracted that for all $v\geq 0$:
\begin{align*}
    \dot{f}(z_1+v(z_2-z_1)) \geq \dot{f}(z_1)\exp\left(-v\vert z_1-z_2\vert\right)\, .
\end{align*}
Integrating between $v\in[0,1]$ and subsequently integrating by part, we obtain:
\begin{align*}
     \int_{v=0}^1 (1-v)\dot{f}\left(z_1+v(z_2-z_1)\right)dv &\geq \dot{f}(z_1)\left(\frac{1}{\vert z_1-z_2\vert}+\frac{\exp\left(-\vert z_1-z_2\vert\right)-1}{\vert z_1-z_2\vert^2 }\right)\\
     &= \dot{f}(z_1)g(\vert z_1-z_2\vert).
\end{align*}
where we defined:
\begin{align*}
    g(z) := \frac{1}{x}\left(1+\frac{\exp(-x)-1}{x}\right)\, .
\end{align*}
Finally, we use \cref{lemma:expinequality} which guarantees that $g(z)\geq (2+z)^{-1}$ for all $z\geq 0$ to prove the claimed result.
\end{proof}
We will need one last technical result obtained from the self-concordance property. Its proof can be extracted from \cref{eq:maxminsc} in the proof of \cref{lemma:firstselfconcordance}.

\begin{restatable}{lemma}{lemmathirdselfconcordance}
\label{lemma:fthirdselfconcordance}
Let $f$ be a strictly increasing function such that $\vert\ddot{f}\vert \leq \dot{f}$, and let $\mcal{Z}$ be any bounded interval of $\mathbb{R}$. Then, for all $z_1,z_2\in\mcal{Z}$:
\begin{align*}
    \dot{f}(z_2) \exp\left(-\vert z_2-z_1\vert\right)\leq \dot{f}(z_1) \leq \dot{f}(z_2)\exp\left(\vert z_2-z_1\vert\right)
\end{align*}
\end{restatable}

%% file: appendix/app_aux.tex
\newpage
\section{\uppercase{Auxiliary Results}}
\label{app:aux}

\begin{restatable}{lemma}{lemmaexpinequality}
\label{lemma:expinequality}
For all $x\geq 0$, the following inequality holds:
\begin{align*}
    \frac{1}{x}\left(1+\frac{\exp(-x)-1}{x}\right) \geq \frac{1}{2+x} \; .
\end{align*}
\end{restatable}
\begin{proof}
It is easy to show that the claimed inequality holds if and only if $\exp(-x)\geq (2-x)(2+x)^{-1}$. Let $h(x)=(2+x)\exp(-x)- (2-x)$. Easy computations yields that for all $x$ we have $h'(x) = -\exp(-x)(1+x)+1$. Using the fact that $\exp(-x)\leq (1+x)^{-1}$ for all $x\geq 0$ (derived from $e^{x}\geq 1+x$) we get that:
\begin{align*}
    h'(x) \geq -\frac{1+x}{1+x} +1 = 0\, .
\end{align*}
The increasing nature of $h$ on $\mbb{R}^+$, along with the fact that $h(0)=0$ is enough to show that $\exp(-x)\geq (2-x)(2+x)^{-1}$ for all $x\geq 0$. As laid out in the first lines of the proof, this suffices to prove our claim.
\end{proof}

\begin{restatable}[Polynomial Inequality]{prop}{proppolynomialineq}
\label{prop:polynomialineq}
Let $b,c\in\mbb{R}^+$, and $x\in\mbb{R}$. The following implication holds:
\begin{align*}
	x^2 \leq bx +c \pmb{\Longrightarrow} x \leq b+\sqrt{c}
\end{align*}

\end{restatable}
\begin{proof}
	Let $f:x\to x^2 - bx -c$. Then $f$ is a strongly-convex function which roots are:
	\begin{align*}
		\lambda_{1,2} = \frac{1}{2}(b \pm \sqrt{b^2+4c}) 
	\end{align*}
If $x^2\leq -b-c$ then by convexity of $f$ we obtain:
\begin{align*}
	x& \leq \max(\lambda_1,\lambda_2)\\
	&\leq  \frac{1}{2}(b + \sqrt{b^2+4c})  &\\
	&\leq b + \sqrt{c} &(\sqrt{x+y}\leq \sqrt{x}+\sqrt{y}, \; \forall x,y\geq 0)
\end{align*}
\end{proof}

The following theorem is extracted from \citep[Lemma 10]{abbasi2011improved}. 
\begin{lemma}[Determinant-Trace inequality]
     Let $\{x_s\}_{s=1}^\infty$ a sequence in $\mbb{R}^d$ such that $\ltwo{x_s}\leq X$ for all $s\in\mbb{N}$, and  let $\lambda$ be a non-negative scalar. For $t\geq 1$ define $\mbold{V}_t \defeq \sum_{s=1}^{t-1} x_sx_s^\transp+\lambda\mbold{I}_d$. The following inequality holds:
     \begin{align*}
         \det(\mbold{V}_{t+1}) \leq \left(\lambda+(t-1)X^2/d\right)^d
     \end{align*}
\label{lemma:determinant_trace_inequality}
\end{lemma}

We need a slight-variation of the Elliptical Potential Lemma \citep[Lemma 11]{abbasi2011improved} adjusted to handle (increasing) time-varying regulations. 

\begin{lemma}[Elliptical potential]
    Let $\{x_s\}_{s=1}^\infty$ a sequence in $\mbb{R}^d$ such that $\ltwo{x_s}\leq X$ for all $s\in\mbb{N}$. Further let $\{\lambda_s\}_{s=0}^\infty$ be an \emph{increasing} sequence in $\mbb{R}^+$ s.t $\lambda_1=1$. For $t\geq 1$ define $\mbold{V}_t \defeq \sum_{s=1}^{t-1} x_sx_s^\transp+\lambda_t\mbold{I}_d$. Then:
    $$
        \sum_{t=1}^{T} \left\lVert x_t\right\rVert_{\mbold{V}_t^{-1}}^2 \leq 2d(1+X^2)\log\left(\lambda_T + \frac{TX^2}{d}\right)
   $$
\label{lemma:ellipticalpotentia}
\end{lemma}
\begin{proof}
By definition of $\mbold{V}_t$:
\begin{align*}
    \left\vert \mbold{V}_{t+1}\right\vert &= \left\vert \sum_{s=1}^{t-1}x_sx_s^\transp + x_tx_t^\transp + \lambda_{t}\mbold{I_d}\right\vert\\
     &\geq \left\vert \sum_{s=1}^{t-1}x_sx_s^\transp + x_tx_t^\transp + \lambda_{t-1}\mbold{I_d}\right\vert &(\lambda_t\geq \lambda_{t-1})\\
     &= \left\vert \mbold{V}_t + x_tx_t^\transp\right\vert\\
    &\geq \left\vert \mbold{V}_t\right\vert \left\vert \mbold{I}_d + \mbold{V}_t^{-1/2}x_tx_t^T\mbold{V}_t^{-1/2}\right\vert\\
    &= \left\vert \mbold{V}_t\right\vert\left(1+\left\lVert x_t\right\rVert_{\mbold{V}_t^{-1}}^2\right)\\
\end{align*}
and therefore by taking the log on both side of the equation and summing from $t=1$ to $T$:
\begin{align*}
     \sum_{t=1}^T \log\left(1+\left\lVert x_t\right\rVert_{\mbold{V}_t^{-1}}^2\right) &\leq \sum_{t=1}^T \log\left\vert \mbold{V}_{t+1}\right\vert - \log\left\vert \mbold{V}_{t}\right\vert &\\
     &= \log\left(\frac{\det(\mbold{V}_{T+1})}{\det(\lambda_1\mbold{I}_d)}\right) &\text{(telescopic sum)} \\
     &= \log\left(\det(\mbold{V}_{T+1})\right)  & (\lambda_1=1)\\
     &\leq d\log\left(\lambda_T + \frac{TX^2}{d}\right) &(\text{\cref{lemma:determinant_trace_inequality}})
\end{align*}
Remember that for all $x\in[0,1]$ we have the inequality $\log(1+x)\geq x/2$. Also note that $\left\lVert x_t\right\rVert_{\mbold{V}_t^{-1}}^2\leq X^2/\lambda$. Therefore:
\begin{align*}
   d\log\left(\lambda_T + \frac{TX^2}{d}\right)  &\geq   \sum_{t=1}^T \log\left(1+\left\lVert x_t\right\rVert_{\mbold{V}_t^{-1}}^2\right)&\\
     &\geq \sum_{t=1}^T \log\left(1+\frac{1}{\max(1,X^2/\lambda_t)}\left\lVert x_t\right\rVert_{\mbold{V}_t^{-1}}^2\right)&\\
     &\geq \frac{1}{2\max(1,X^2/\lambda_1)}\sum_{t=1}^T \left\lVert x_t\right\rVert_{\mbold{V}_t^{-1}}^2\\
     &\geq \frac{1}{2(1+X^2)}\sum_{t=1}^T \left\lVert x_t\right\rVert_{\mbold{V}_t^{-1}}^2
\end{align*}
which yields the announced result. 
\end{proof}

%% file: appendix/app_exp.tex
\newpage

\section{NUMERICAL EXPERIMENTS}
\label{app:exp}

We present here a few illustrative experiments. We compare the three following algorithms:  \textbf{GLM-UCB} \citep{filippi2010parametric}, \textbf{LogUCB1} \citep{faury2020improved} and \textbf{OFULog} (this work). We didn't implement \textbf{LogUCB2} \citep{faury2020improved}: it is intractable as it relies on non-convex minimization that cannot be bypassed. These results, presented in \cref{fig:exps}, corroborate our theoretical analysis: \textbf{(1)} our algorithm displays a clear advantage over previous approaches ( \cref{fig:kappa50,fig:kappa400}) \textbf{(2)} a higher level of non-linearity (i.e higher values of $\kappa$) is actually beneficial (\cref{fig:allkappa}) for \textbf{OFULog}. Remember that this cannot be the case for other approaches as by design, the performances of \textbf{GLM-UCB} and \textbf{LogUCB1} can only degrade when $\kappa$ increases. The arm-set is composed of 40 arms, drawn uniformly at random on the $2$-dimensional ball at the beginning of each run. For each experiment, we average the regret curves over 50 independent runs and report standard-deviation in shaded colors. 

\begin{figure}[th]
    \begin{subfigure}{0.3\textwidth}
        \includegraphics[width=\linewidth]{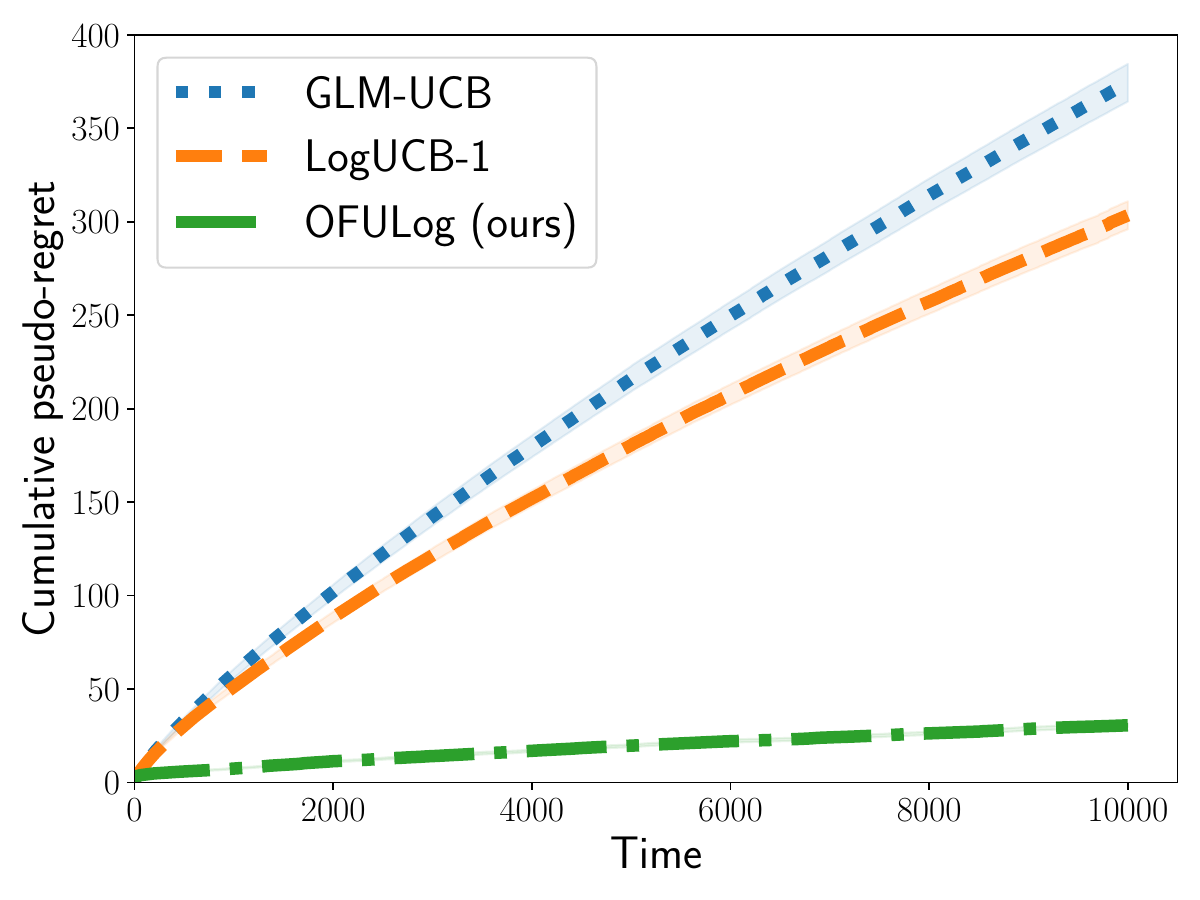}
        \caption{Regret curves for $\kappa\!=\!50$ in $d\!=\!2$ with 40 arms.}
        \label{fig:kappa50}
    \end{subfigure}
    \hfill
    \begin{subfigure}{0.3\textwidth}
        \includegraphics[width=\linewidth]{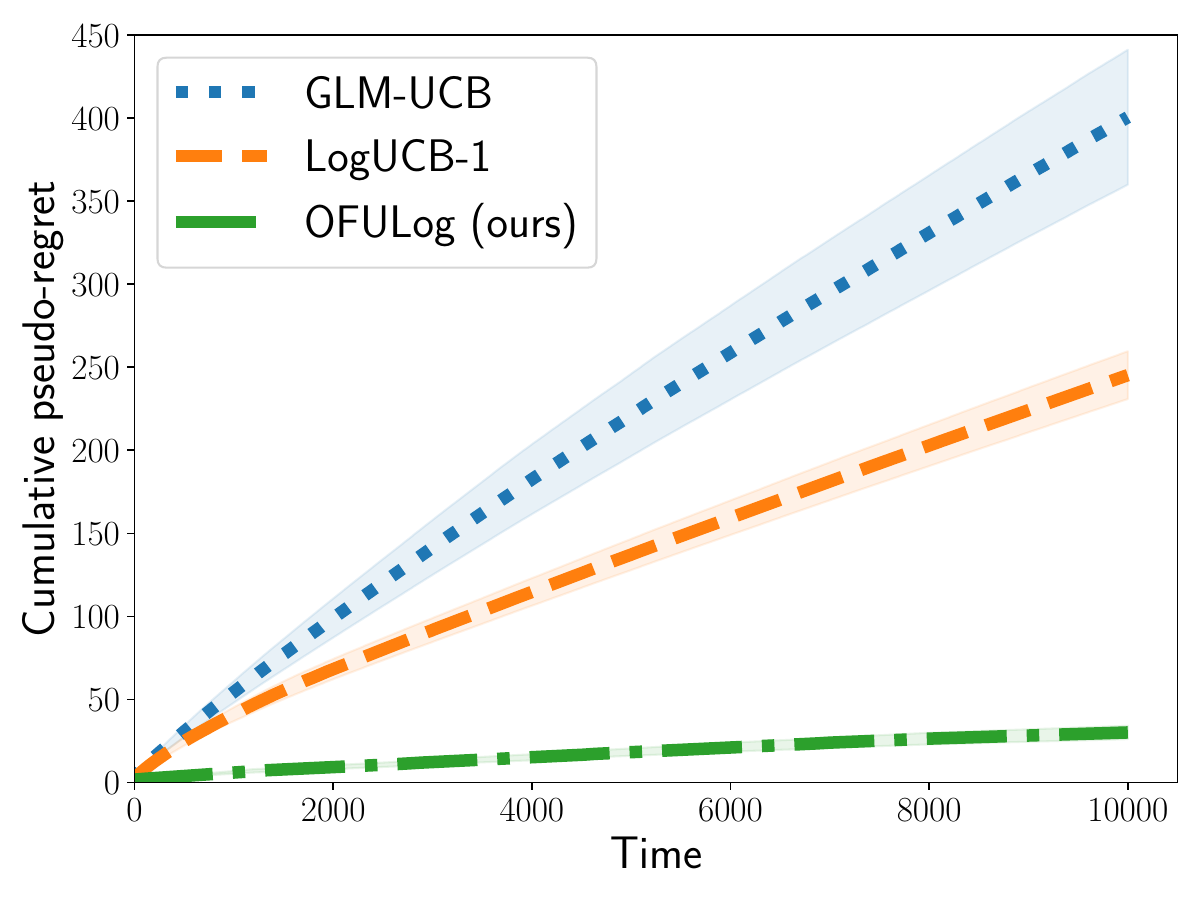}
        \caption{Regret curves for $\kappa\!=\!400$ in $d\!=\!2$ with 40 arms.}
        \label{fig:kappa400}
    \end{subfigure}
    \hfill
    \begin{subfigure}{0.3\textwidth}
        \includegraphics[width=\linewidth]{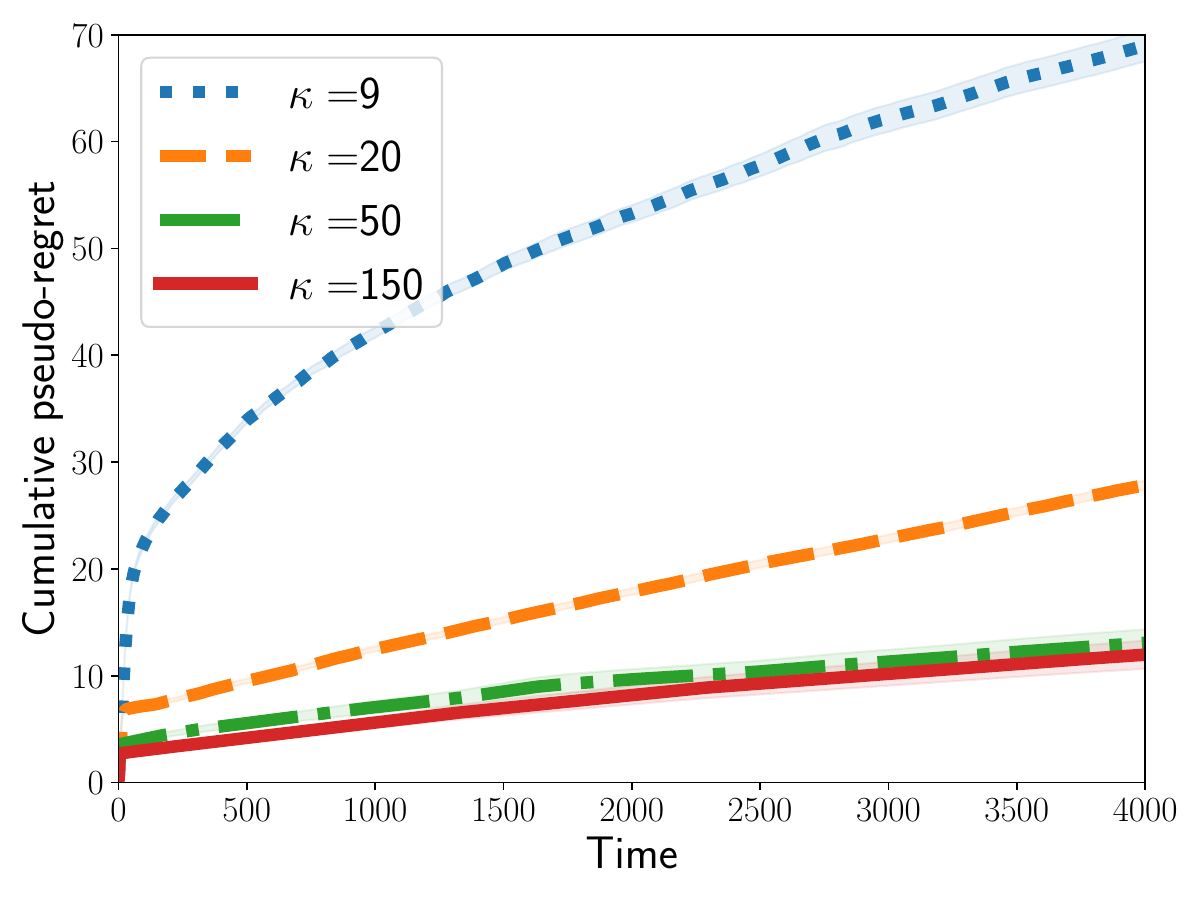}
        \caption{Regret curves of \textbf{OFULog} in $d=2$ with 40 arms for different $\kappa$.}
        \label{fig:allkappa}
    \end{subfigure}
    \caption{Illustrative numerical experiments. Shaded areas represent 1-standard deviation of the cumulative regret, aggregated over 50 independent experiments.}
    \label{fig:exps}
\end{figure}